\documentclass[letterpaper]{article}
\usepackage[margin=1in,dvips]{geometry}
\usepackage{graphicx,psfrag,amsmath,amsthm,amssymb}
\usepackage{natbib}
\usepackage{url,color,booktabs}
\usepackage{hyperref}

\def \ifempty#1{\def\temp{#1} \ifx\temp\empty }

\newcommand{\A}{\ensuremath{\mathbf{A}}}

\newcommand{\W}{\ensuremath{\mathbf{W}}}

\newcommand{\Z}{\ensuremath{\mathbf{Z}}}
\renewcommand{\aa}{\ensuremath{\mathbf{a}}}
\renewcommand{\b}{\ensuremath{\mathbf{b}}}
\renewcommand{\c}{\ensuremath{\mathbf{c}}}

\newcommand{\w}{\ensuremath{\mathbf{w}}}
\newcommand{\x}{\ensuremath{\mathbf{x}}}
\newcommand{\y}{\ensuremath{\mathbf{y}}}
\newcommand{\z}{\ensuremath{\mathbf{z}}}
\newcommand{\0}{\ensuremath{\mathbf{0}}}

\newcommand{\blambda}{\ensuremath{\boldsymbol{\lambda}}}

\newcommand{\btheta}{\ensuremath{\boldsymbol{\theta}}}

\newcommand{\bDelta}{\ensuremath{\boldsymbol{\Delta}}}

\newcommand{\bPi}{\ensuremath{\boldsymbol{\Pi}}}

\newcommand{\bTheta}{\ensuremath{\boldsymbol{\Theta}}}

\newcommand{\bbR}{\ensuremath{\mathbb{R}}}

\newcommand{\calC}{\ensuremath{\mathcal{C}}}

\newcommand{\calF}{\ensuremath{\mathcal{F}}}

\newcommand{\calL}{\ensuremath{\mathcal{L}}}

\newcommand{\calO}{\ensuremath{\mathcal{O}}}

\newcommand{\calS}{\ensuremath{\mathcal{S}}}

\newcommand{\ceil}[1]{\lceil#1\rceil}
\newcommand{\floor}[1]{\lfloor#1\rfloor}
\newcommand{\abs}[2][]{%
  \ifempty{#1} {\left\lvert#2\right\rvert} \else {#1\lvert#2#1\rvert} \fi}
\newcommand{\norm}[2][]{%
  \ifempty{#1} {\left\lVert#2\right\rVert} \else {#1\lVert#2#1\rVert} \fi}

\newcommand{\caja}[4][1]{{%
    \renewcommand{\arraystretch}{#1}%
    \begin{tabular}[#2]{@{}#3@{}}%
      #4%
    \end{tabular}%
    }}

{%
\begin{list}{#1}{
\vspace{-\topsep}
\vspace{-\partopsep}
\setlength{\itemindent}{0cm}
\setlength{\rightmargin}{0cm}
\setlength{\listparindent}{0cm}
\settowidth{\labelwidth}{#1}
\setlength{\leftmargin}{\labelwidth}
\addtolength{\leftmargin}{\labelsep}
\setlength{\itemsep}{0cm}
}%
}%
{%
\end{list}
\vspace{-\topsep}
\vspace{-\partopsep}
}

{\begin{enumerate}%
}%
{\end{enumerate}}

\DeclareMathOperator*{\argmin}{arg\,min}
\DeclareMathOperator*{\argmax}{arg\,max}

\newcommand{\sgnop}{\operatorname{sgn}}
\newcommand{\sgn}[1]{\ensuremath{\sgnop\left(#1\right)}}

\theoremstyle{plain}% default
\newtheorem{thm}{Theorem}[section]

\newtheorem*{lemma*}{Lemma}

\newtheorem*{prop*}{Proposition}

\theoremstyle{definition}

\newtheorem*{defn*}{Definition}

\newtheorem*{exmp*}{Example}

\newtheorem*{conj*}{Conjecture}

\theoremstyle{remark}

\newtheorem*{rmk*}{Remark}

\graphicspath{{grf/},{grf-miguel/}}

\AtBeginDvi{}

\bibpunct[, ]{(}{)}{;}{a}{,}{,} % Should be the natbib default but it isn't!

\title{Model compression as constrained optimization, \\ with application to neural nets. \\ Part II: quantization.}
\author{
  Miguel \'A.\ Carreira-Perpi\~n\'an\hspace{5ex} Yerlan Idelbayev \\
  Electrical Engineering and Computer Science, University of California, Merced \\
  {\url{http://eecs.ucmerced.edu}}
}
\date{July 13, 2017}

\begin{document}

\maketitle

\begin{abstract}
  
  We consider the problem of deep neural net compression by quantization: given a large, reference net, we want to quantize its real-valued weights using a codebook with $K$ entries so that the training loss of the quantized net is minimal. The codebook can be optimally learned jointly with the net, or fixed, as for binarization or ternarization approaches. Previous work has quantized the weights of the reference net, or incorporated rounding operations in the backpropagation algorithm, but this has no guarantee of converging to a loss-optimal, quantized net. We describe a new approach based on the recently proposed framework of \emph{model compression as constrained optimization} \citep{Carreir17a}. This results in a simple iterative ``learning-compression'' algorithm, which alternates a step that learns a net of continuous weights with a step that quantizes (or binarizes/ternarizes) the weights, and is guaranteed to converge to local optimum of the loss for quantized nets. We develop algorithms for an adaptive codebook or a (partially) fixed codebook. The latter includes binarization, ternarization, powers-of-two and other important particular cases. We show experimentally that we can achieve much higher compression rates than previous quantization work (even using just 1 bit per weight) with negligible loss degradation.

\end{abstract}

\section{Introduction}
\label{s:intro}

The widespread application of deep neural nets in recent years has seen an explosive growth in the size of the training sets, the number of parameters of the nets, and the amount of computing power needed to train them. At present, very deep neural nets with upwards of many million weights are common in applications in computer vision and speech. Many of these applications are particularly useful in small devices, such as mobile phones, cameras or other sensors, which have limited computation, memory and communication bandwidth, and short battery life. It then becomes desirable to compress a neural net so that its memory storage is smaller and/or its runtime is faster and consumes less energy.

Neural net compression was a problem of interest already in the early days of neural nets, driven for example by the desire to implement neural nets in VLSI circuits. However, the current wave of deep learning work has resulted in a flurry of papers by many academic and particularly industrial labs proposing various ways to compress deep nets, some new and some not so new (see related work). Various standard forms of compression have been used in one way or another, such as low-rank decomposition, quantization, binarization, pruning and others. In this paper we focus on quantization, where the ordinarily unconstrained, real-valued weights of the neural net are forced to take values within a codebook with a finite number of entries. This codebook can be adaptive, so that its entries are learned together with the quantized weights, or (partially) fixed, which includes specific approaches such as binarization, ternarization or powers-of-two approaches.

Among compression approaches, quantization is of great interest because even crudely quantizing the weights of a trained net (for example, reducing the precision from double to single) produces considerable compression with little degradation of the loss of the task at hand (say, classification). However, this ignores the fact that the quantization is not independent of the loss, and indeed achieving a really low number of bits per weight (even just 1 bit, i.e., binary weights) would incur a large loss and make the quantized net unsuitable for practical use. Previous work has applied a quantization algorithm to a previously trained, reference net, or incorporated ad-hoc modifications to the basic backpropagation algorithm during training of the net. However, none of these approaches are guaranteed to produce upon convergence (if convergence occurs at all) a net that has quantized weights and has optimal loss among all possible quantized nets.

In this paper, our primary objectives are: 1) to provide a mathematically principled statement of the quantization problem that involves the loss of the resulting net, and 2) to provide an algorithm that can solve that problem up to local optima in an efficient and convenient way. Our starting point is a recently proposed formulation of the general problem of model compression as a constrained optimization problem \citep{Carreir17a}. We develop this for the case where the constraints represent the optimal weights as coming from a codebook. This results in a ``learning-compression'' (LC) algorithm that alternates SGD optimization of the loss over real-valued weights but with a quadratic regularization term, and quantization of the current real-valued weights. The quantization step takes a form that follows necessarily from the problem definition without ad-hoc decisions: $k$-means for adaptive codebooks, and an optimal assignment for fixed codebooks such as binarization, ternarization or powers-of-two (with possibly an optimal global scale). We then show experimentally that we can compress deep nets considerably more than previous quantization algorithms---often, all the way to the maximum possible compression, a single bit per weight, without significant error degradation.

\section{Related work on quantization of neural nets}
\label{s:related}

Much work exists on compressing neural nets, using quantization, low-rank decomposition, pruning and other techniques, see \citet{Carreir17a} and references therein. Here we focus exclusively on work based on quantization. Quantization of neural net weights was recognized as an important problem early in the neural net literature, often with the goal of efficient hardware implementation, and has received much attention recently. The main approaches are of two types. The first one consists of using low-precision, fixed-point or other weight representations through some form of rounding, even single-bit (binary) values. This can be seen as quantization using a fixed codebook (i.e., with predetermined values). The second approach learns the codebook itself as a form of soft or hard adaptive quantization. There is also work on using low-precision arithmetic directly during training (see \citealp{Gupta_15a} and references therein) but we focus here on work whose goal is to quantize a neural net of real-valued, non-quantized weights.

\subsection{Quantization with a fixed codebook}
\label{s:related:quant-fixed}

Work in the 1980s and 1990s explored binarization, ternarization and general powers-of-two quantization \citep{Fiesler_90a,Marches_93a,TangKwan93a}. These same quantization forms have been revisited in recent years \citep{HwangSung14a,Courbar_15a,Rasteg_16a,Hubara_16b,Li_16b,Zhou_16b,Zhu_17a}, with impressive results on large neural nets trained on GPUs, but not much innovation algorithmically. The basic idea in all these papers is essentially the same: to modify backpropagation so that it encourages binarization, ternarization or some other form of quantization of the neural net weights. The modification involves evaluating the gradient of the loss $L(\w)$ at the quantized weights (using a specific quantization or ``rounding'' operator that maps a continuous weight to a quantized one) but applying the update (gradient or SGD step) to the continuous (non-quantized) weights. Specific details vary, such as the quantization operator or the type of codebook. The latter has recently seen a plethora of minor variations: $\{-1,0,+1\}$ \citep{HwangSung14a}, $\{-1,+1\}$ \citep{Courbar_15a}, $\{-a,+a\}$ \citep{Rasteg_16a,Zhou_16b}, $\{-a,0,+a\}$ \citep{Li_16b} or $\{-a,0,+b\}$ \citep{Zhu_17a}.

One important problem with these approaches is that their modification of backpropagation is ad-hoc, without guarantees of converging to a net with quantized weights and low loss, or of converging at all. Consider binarization to $\{-1,+1\}$ for simplicity. The gradient is computed at a binarized weight vector $\w \in \{-1,+1\}^P$, of which there are a finite number ($2^P$, corresponding to the hypercube corners), and none of these will in general have gradient zero. Hence training will never stop, and the iterates will oscillate indefinitely. Practically, this is stopped after a certain number of iterations, at which time the weight distribution is far from binarized (see fig.~2 in \citealp{Courbar_15a}), so a drastic binarization must still be done. Given these problems, it is surprising that these techniques do seem to be somewhat effective empirically in quantizing the weights and still achieve little loss degradation, as reported in the papers above. Exactly how effective they are, on what type of nets and why this is so is an open research question.

In our LC algorithm, the optimization essentially happens in the continuous weight space by minimizing a well-defined objective (the penalized function in the L step), but this is regularly corrected by a quantization operator (C step), so that the algorithm gradually converges to a truly quantized weight vector while achieving a low loss (up to local optima). The form of both L and C steps, in particular of the quantization operator (our compression function $\bPi(\w)$), follows in a principled, optimal way from the constrained form of the problem~\eqref{e:compression-problem}. That is, given a desired form of quantization (e.g.\ binarization), the form of the C step is determined, and the overall algorithm is guaranteed to converge to a valid (binary) solution.

Also, we emphasize that there is little practical reason to use certain fixed codebooks, such as $\{-1,+1\}$ or $\{-a,+a\}$, instead of an adaptive codebook such as $\{c_1,c_2\}$ with $c_1,c_2 \in \bbR$. The latter is obviously less restrictive, so it will incur a lower loss. And its hardware implementation is about as efficient: to compute a scalar product of an activation vector with a quantized weight vector, all we require is to sum activation values for each centroid and to do two floating-point multiplications (with $c_1$ and $c_2$). Indeed, our experiments in section~\ref{s:expts:LeNet} show that using an adaptive codebook with $K=2$ clearly beats using $\{-1,+1\}$.

\subsection{Quantization with an adaptive codebook}
\label{s:related:quant-adaptive}

Quantization with an adaptive codebook is, obviously, more powerful than with a fixed codebook, even though it has to store the codebook itself. Quantization using an adaptive codebook has also been explored in the neural nets literature, using approaches based on soft quantization \citep{NowlanHinton92a,Ullric_17a} or hard quantization \citep{Fiesler_90a,Marches_93a,TangKwan93a,Gong_15a,Han_15a}, and we discuss this briefly.

Given a set of real-valued elements (scalars or vectors), in adaptive quantization we represent (``quantize'') each element by exactly one entry in a codebook. The codebook and the assignment of values to codebook entries should minimize a certain distortion measure, such as the squared error. Learning the codebook and assignment is done by an algorithm, possibly approximate (such as $k$-means for the squared error). Quantization is related to clustering and often one can use the same algorithm for both (e.g.\ $k$-means), but the goal is different: quantization seeks to minimize the distortion rather than to model the data as clusters. For example, a set of values uniformly distributed in $[-1,1]$ shows no clusters but may be subject to quantization for compression purposes. In our case of neural net compression, we have an additional peculiarity that complicates the optimization: the quantization and the weight values themselves should be jointly learned to minimize the loss of the net on the task.

Two types of clustering exist, hard and soft clustering. In hard clustering, each data point is assigned to exactly one cluster (e.g.\ $k$-means clustering). In soft clustering, we have a probability distribution over points and clusters (e.g.\ Gaussian mixture clustering). Likewise, two basic approaches exist for neural net quantization, based on hard and soft quantization. We review each next.

In hard quantization, each weight is assigned to exactly one codebook value. This is the usual meaning of quantization. This is a difficult problem because, even if the loss is differentiable over the weights, the assignment makes the problem inherently combinatorial. Previous work \citep{Gong_15a,Han_15a} has run a quantization step ($k$-means) as a postprocessing step on a reference net (which was trained to minimize the loss). This is suboptimal in that it does not learn the weights, codebook and assignment jointly. We call this ``direct compression'' and discuss it in more detail in section~\ref{s:DC}. Our LC algorithm does learn the weights, codebook and assignment jointly, and converges to a local optimum of problem~\eqref{e:compression-problem}.

In soft quantization, the assignment of values to codebook entries is based on a probability distribution. This was originally proposed by \citet{NowlanHinton92a} as a way to share weights softly in a neural net with the goal of improving generalization, and has been recently revisited with the goal of compression \citep{Ullric_17a}. The idea is to penalize the loss with the negative log-likelihood of a Gaussian mixture (GM) model on the scalar weights of the net. This has the advantage of being differentiable and of coadapting the weights and the GM parameters (proportions, means, variances). However, it does not uniquely assign each weight to one mean, in fact the resulting distribution of weights is far from quantized; it simply encourages the creation of Gaussian clusters of weights, and one has to assign weights to means as a postprocessing step, which is suboptimal. The basic problem is that a GM is a good model (better than $k$-means) for noisy or uncertain data, but that is not what we have here. Quantizing the weights for compression implies a constraint that certain weights must take exactly the same value, without noise or uncertainty, and optimize the loss. We seek an optimal assignment that is truly hard, not soft. Indeed, a GM prior is to quantization what a quadratic prior (i.e., weight decay) is to sparsity: a quadratic prior encourages all weights to be small but does not encourage some weights to be exactly zero, just as a GM prior encourages weights to form Gaussian clusters but not to become groups of identical weights.

\section{Neural net quantization as constrained optimization and the ``learning-compression'' (LC) algorithm}
\label{s:LC}

As noted in the introduction, compressing a neural net optimally means finding the compressed net that has (locally) lowest loss. Our first goal is to formulate this mathematically in a way that is amenable to nonconvex optimization techniques. Following \citet{Carreir17a}, we define the following \emph{model compression as constrained optimization} problem:
\begin{equation}
  \label{e:compression-problem}
  \textcolor{blue}{\min_{\w,\bTheta}{ L(\w) } \quad \text{s.t.} \quad \w = \bDelta(\bTheta)}
\end{equation}
where $\w \in \bbR^P$ are the real-valued weights of the neural net, $L(\w)$ is the loss to be minimized (e.g.\ cross-entropy for a classification task on some training set), and the constraint $\w = \bDelta(\bTheta)$ indicates that the weights must be the result of decompressing a low-dimensional parameter vector \bTheta. This corresponds to quantization and will be described in section~\ref{s:quant}. Problem~\eqref{e:compression-problem} is equivalent to the unconstrained problem ``$\min_{\bTheta}{ L(\bDelta(\bTheta)) }$'', but this is nondifferentiable with quantization (where \bDelta\ is a discrete mapping), and introducing the auxiliary variable \w\ will lead to a convenient algorithm.

Our second goal is to solve this problem via an efficient algorithm. Although this might be done in different ways, a particularly simple one was proposed by \citet{Carreir17a} that achieves separability between the data-dependent part of the problem (the loss) and the data-independent part (the weight quantization). First, we apply a penalty method to solve~\eqref{e:compression-problem}. We consider here the augmented Lagrangian method \citep{NocedalWright06a}, where $\blambda \in \bbR^P$ are the Lagrange multiplier estimates%
\footnote{All norms are $\norm{\cdot}_2$ throughout the paper unless indicated otherwise.}:
\begin{align}
  \label{e:augLag}
  \calL_A(\w,\bTheta,\blambda;\mu) &= L(\w) - \blambda^T (\w - \bDelta(\bTheta)) + \frac{\mu}{2} \norm{\w - \bDelta(\bTheta)}^2 \\
  \label{e:augLag2}
  &= L(\w) + \frac{\mu}{2} \norm[\Big]{\w - \bDelta(\bTheta) - \frac{1}{\mu} \blambda}^2 - \frac{1}{2\mu} \norm{\blambda}^2.
\end{align}
The augmented Lagrangian method works as follows. For fixed $\mu \ge 0$, we optimize $\calL_A(\w,\bTheta,\blambda;\mu)$ over $(\w,\bTheta)$ accurately enough. Then, we update the Lagrange multiplier estimates as $\blambda \leftarrow \blambda - \mu (\w - \bDelta(\bTheta))$. Finally, we increase $\mu$. We repeat this process and, in the limit as $\mu \rightarrow \infty$, the iterates $(\w,\bTheta)$ tend to a local KKT point (typically, a local minimizer) of the constrained problem~\eqref{e:compression-problem}. A simpler but less effective penalty method, the quadratic penalty method, results from setting $\blambda = \0$ throughout; we do not describe it explicitly, see \citet{Carreir17a}.

Finally, in order to optimize $\calL_A(\w,\bTheta,\blambda;\mu)$ over $(\w,\bTheta)$, we use alternating optimization. This gives rise to the following two steps:
\begin{itemize}
\item \textbf{L step: learning}
  \begin{equation}
    \label{e:Lstep}
    \textcolor{blue}{\min_{\w}{ L(\w) + \frac{\mu}{2} \norm[\Big]{\w - \bDelta(\bTheta) - \frac{1}{\mu} \blambda}^2 }.}
  \end{equation}
  This involves optimizing a regularized version of the loss, which pulls the optimizer towards the currently quantized weights. For neural nets, it can be solved with stochastic gradient descent (SGD).
\item \textbf{C step: compression (here, quantization)}
  \begin{equation}
    \label{e:Cstep}
    \textcolor{blue}{\min_{\bTheta}{ \norm[\Big]{\w - \frac{1}{\mu} \blambda - \bDelta(\bTheta)}^2 } \quad \Longleftrightarrow \quad \bTheta = \bPi \Big(\w - \frac{1}{\mu} \blambda \Big)}.
  \end{equation}
  We describe this in section~\ref{s:quant}. Solving this problem is equivalent to quantizing optimally the current real-valued weights $\w - \frac{1}{\mu} \blambda$, and can be seen as finding their orthogonal projection $\bPi \big(\w - \frac{1}{\mu} \blambda \big)$ on the feasible set of quantized nets.
\end{itemize}
This algorithm was called the ``learning-compression'' (LC) algorithm by \citet{Carreir17a}.

We note that, throughout the optimization, there are two weight vectors that evolve simultaneously and coincide in the limit as $\mu \rightarrow \infty$: \w\ (or, more precisely, $\w - \frac{1}{\mu} \blambda$) contains real-valued, non-quantized weights (and this is what the L step optimizes over); and $\bDelta(\bTheta)$ contains quantized weights (and this is what the C step optimizes over). In the C step, $\bDelta(\bTheta)$ is the projection of the current \w\ on the feasible set of quantized vectors. In the L step, \w\ optimizes the loss while being pulled towards the current $\bDelta(\bTheta)$.

The formulation~\eqref{e:compression-problem} and the LC algorithm have two crucial advantages. The first one is that we get a convenient separation between learning and quantization which allows one to solve each step by reusing existing code. The data-dependent part of the optimization is confined within the L step. This part is the more computationally costly, requiring access to the training set and the neural net, and usually implemented in a GPU using SGD. The data-independent part of the optimization, i.e., the compression of the weights (here, quantization), is confined within the C step. This needs access only to the vector of current, real-valued weights (not to the training set or the actual neural net).

The second advantage is that the form of the C step is determined by the choice of quantization form (defined by $\bDelta(\bTheta)$), and the algorithm designer need not worry about modifying backpropagation or SGD in any way for convergence to a valid solution to occur. For example, if a new form of quantization were discovered and we wished to use it, all we have to do is put it in the decompression mapping form $\bDelta(\bTheta)$ and solve the compression mapping problem~\eqref{e:Cstep} (which depends only on the quantization technique, and for which a known algorithm may exist). This is unlike much work in neural net quantization, where various, somewhat arbitrary quantization or rounding operations are incorporated in the usual backpropagation training (see section~\ref{s:related}), which makes it unclear what problem the overall algorithm is optimizing, if it does optimize anything at all.

In section~\ref{s:quant}, we solve the compression mapping problem~\eqref{e:Cstep} for the adaptive and fixed codebook cases. For now, it suffices to know that it will involve running $k$-means with an adaptive codebook and a form of rounding with a fixed codebook.

\subsection{Geometry of the neural net quantization problem}
\label{s:geom}

\begin{figure}[p]
  \centering
  \begin{tabular}{@{}c@{\hspace{0.04\linewidth}}c@{\hspace{0.04\linewidth}}c@{}}
    \psfrag{XX}[][]{~~\caja{c}{c}{$\overline{\w}$ \\ (reference)}}
    \psfrag{FF}[l][Bl]{\caja{c}{c}{$\w^*$ (optimal \\ compressed)}}
    \psfrag{Fd}[l][Bl]{\caja{c}{c}{$\bDelta(\bTheta^{\text{DC}})$ \\ (direct \\ compression)}}
    \psfrag{R}[l][l]{\w-space}
    \psfrag{mappings}[][]{\caja{c}{c}{feasible nets $\calF_{\w}$ \\ (decompressible \\ by \bDelta)}}
    \includegraphics[height=0.35\linewidth,bb=210 580 435 838,clip]{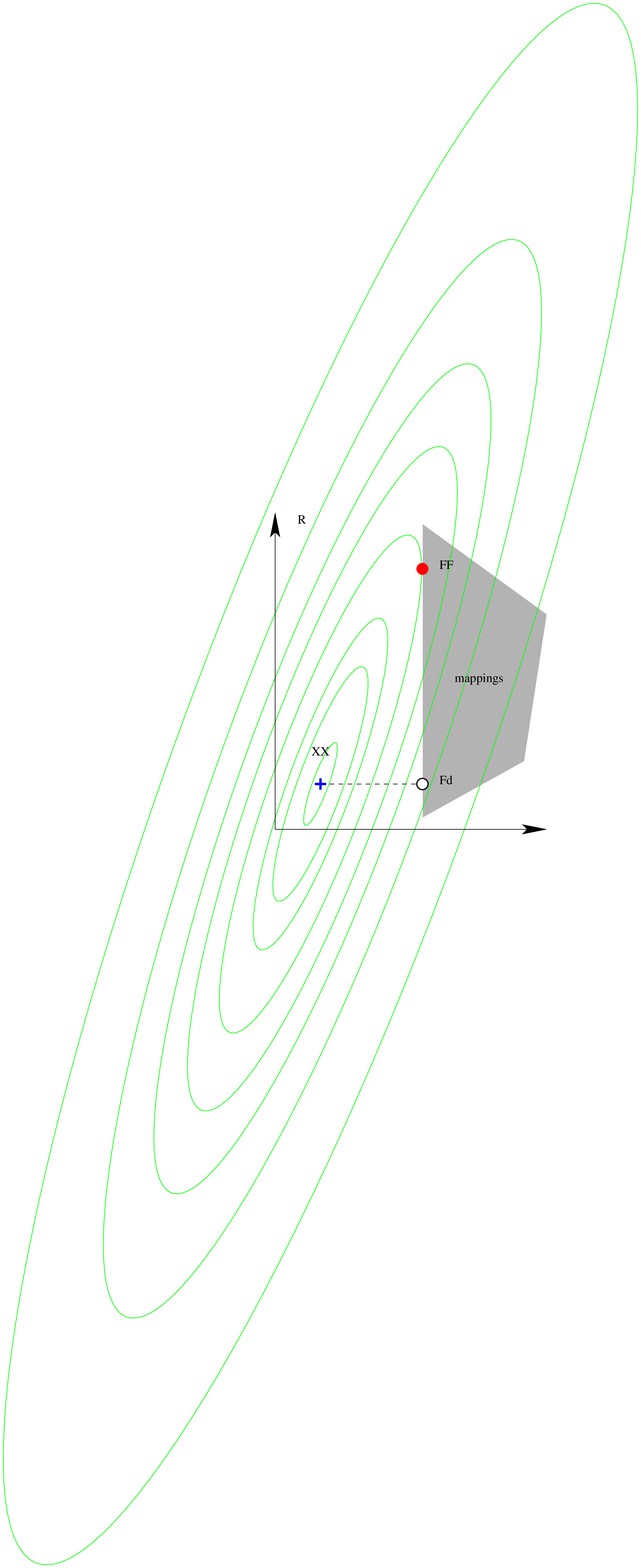} &
    \psfrag{XX1}[t][]{$\overline{\w}_1$}
    \psfrag{XX2}[t][]{$\overline{\w}_2$}
    \psfrag{FF}[l][Bl]{$\w^*$}
    \psfrag{Fd1}[l][Bl]{$\bDelta(\bTheta^{\text{DC}}_1)$}
    \psfrag{Fd2}[l][Bl]{$\bDelta(\bTheta^{\text{DC}}_2)$}
    \psfrag{R}[l][l]{\w-space}
    \psfrag{mappings}[][]{$\calF_{\w}$}
    \includegraphics[height=0.35\linewidth,bb=15 27 240 285,clip]{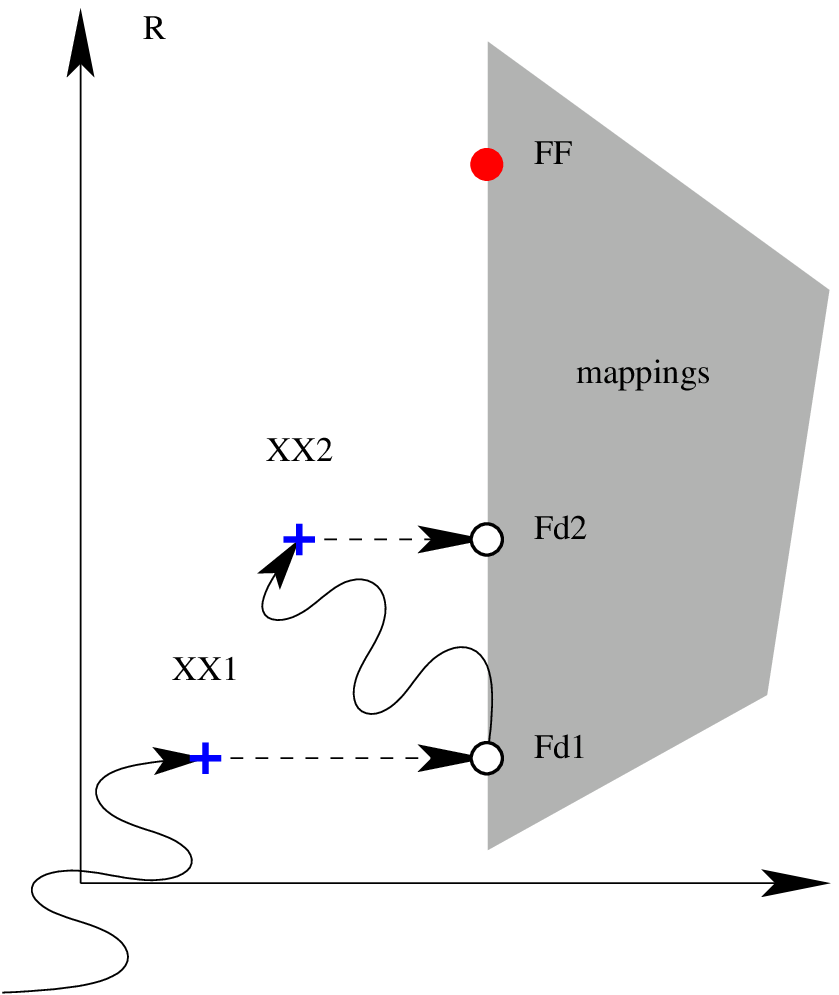} &
    \psfrag{XX}[t][]{$\overline{\w}$}
    \psfrag{FF1}[l][Bl]{$\w^*_1$}
    \psfrag{FF2}[Bl][l]{$\w^*_2$}
    \psfrag{FF3}[][l]{$\w^*_3$}
    \psfrag{FF4}[l][Bl]{$\w^*_4$}
    \psfrag{Fd1}[l][l][0.55]{$\bDelta(\bTheta^{\text{DC}}_1)$}
    \psfrag{Fd2}[l][l][0.55]{$\bDelta(\bTheta^{\text{DC}}_2)$}
    \psfrag{Fd3}[bl][l][0.55]{$\bDelta(\bTheta^{\text{DC}}_3)$}
    \psfrag{Fd4}[lt][l][0.55]{$\bDelta(\bTheta^{\text{DC}}_4)$}
    \psfrag{F1}[l][Bl]{$\calF^1_{\w}$}
    \psfrag{F2}[l][Bl]{$\calF^2_{\w}$}
    \psfrag{F3}[l][Bl]{$\calF^3_{\w}$}
    \psfrag{F4}[l][Bl]{$\calF^4_{\w}$}
    \psfrag{R}[l][l]{\w-space}
    \includegraphics[height=0.35\linewidth,bb=210 580 435 838,clip]{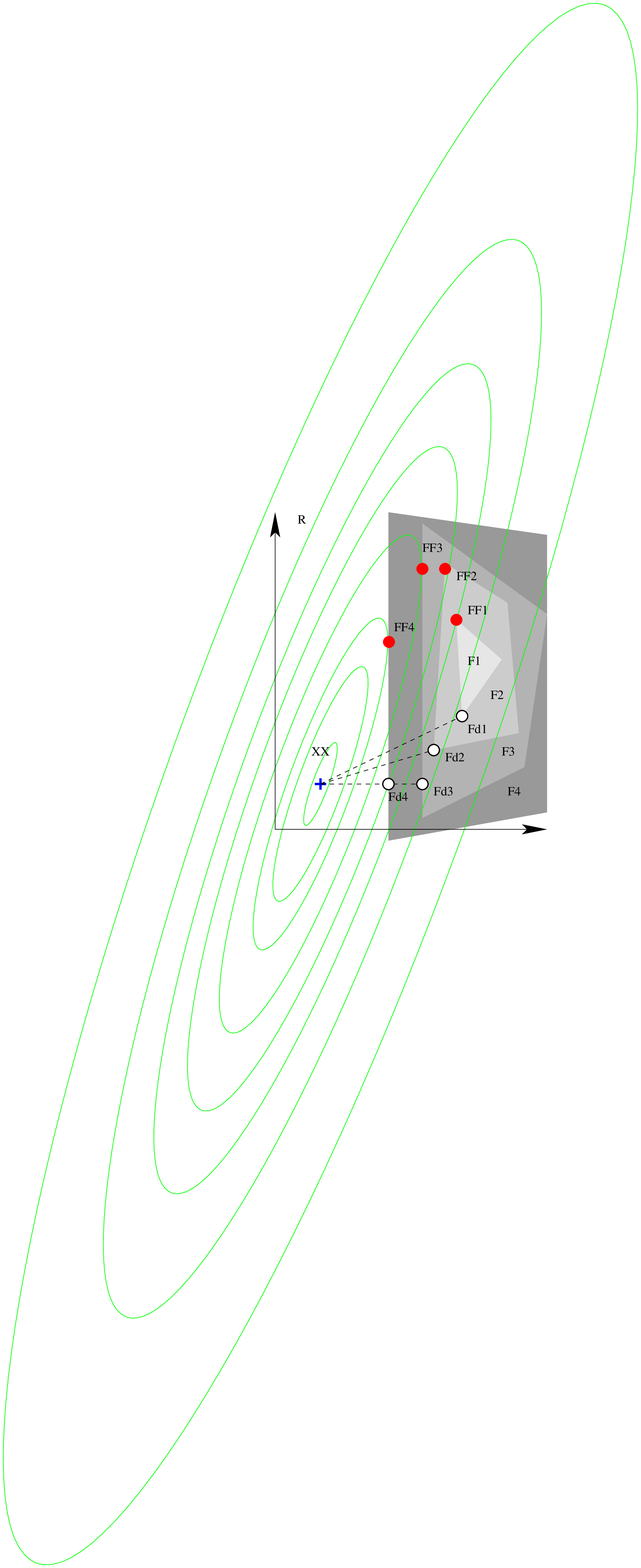}
  \end{tabular} \\[6ex]
  \begin{tabular}{@{}c@{\hspace{0.04\linewidth}}c@{}}
    \psfrag{w1}[][]{$w_1$}
    \psfrag{w2}[][]{$w_2$\raisebox{0.20\linewidth}[0pt][0pt]{\makebox[0pt][c]{\calF}}}
    \psfrag{c}[][r][1][-90]{$c$}
    \includegraphics*[height=0.33\linewidth]{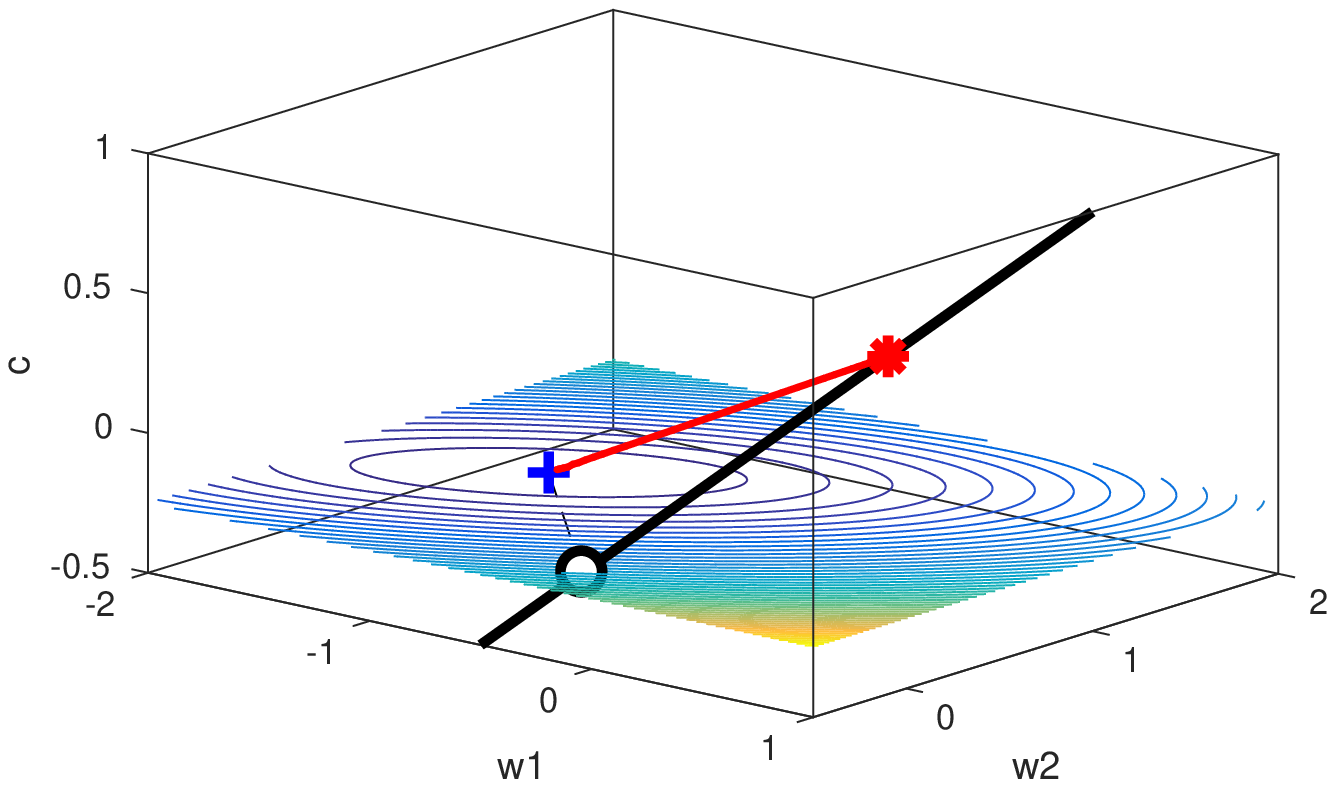} &
    \psfrag{w1}[][B]{$w_1$}
    \psfrag{w2}[][][1][-90]{$w_2$}
    \psfrag{xU}[br][b]{\raisebox{14ex}{\makebox[0pt][r]{\w-space}}\raisebox{2ex}{$\overline{\w}$}}
    \psfrag{xC}[lt][lb]{~~$\w^*$\raisebox{0.08\linewidth}[0pt][0pt]{\makebox[0pt][r]{$\calF_{\w}$}}}
    \psfrag{xF}[l][l][0.9]{~~$\bDelta(\bTheta^{\text{DC}})$$=$$\binom{c^{\text{DC}}}{c^{\text{DC}}}$}
    \includegraphics*[height=0.33\linewidth]{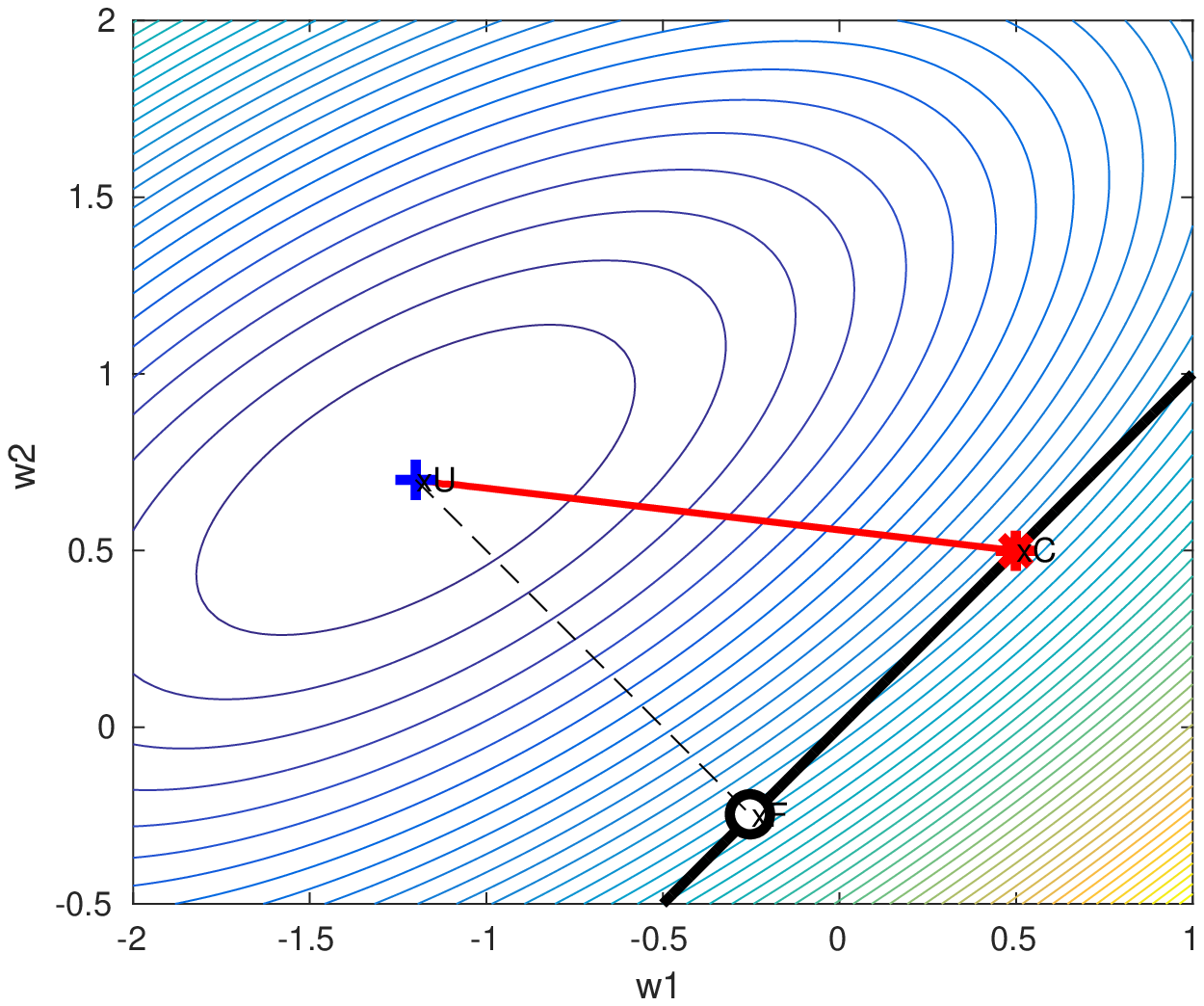}
  \end{tabular}
  \caption[caption]{Schematic representation of the idea of model compression by constrained optimization, in general (top 3 plots) and in particular for quantization (bottom 2 plots). The top of this figure is adapted from \citet{Carreir17a}.\\
    \emph{Plots 1--3 (top row)}: illustration of the uncompressed model space (\w-space $= \bbR^P$), the contour lines of the loss $L(\w)$ (green lines), and the set of compressed models (the feasible set $\calF_{\w} = \{\w \in \bbR^P\mathpunct{:}\ \w = \bDelta(\bTheta) \text{ for } \bTheta \in \bbR^Q\}$, grayed areas), for a generic compression technique \bDelta. The \bTheta-space is not shown. $\overline{\w}$ optimizes $L(\w)$ but is infeasible (no \bTheta\ can decompress into it). The direct compression $\w^{\text{DC}} = \bDelta(\bTheta^{\text{DC}})$ is feasible but not optimal compressed (not optimal in the feasible set). $\w^* = \bDelta(\bTheta^*)$ is optimal compressed. Plot 2 shows two local optima $\overline{\w}_1$ and $\overline{\w}_2$ of the loss $L(\w)$, and their respective DC points (the contour lines are omitted to avoid clutter). Plot 3 shows several feasible sets, corresponding to different compression levels ($\calF^1_{\w}$ is most compression).\\
    \emph{Plots 4--5 (bottom row)}: illustration when \bDelta\ corresponds to quantization, in the particular case of a codebook of size $K=1$ and a 2-weight net, so $\w = (w_1,w_2) \in \bbR^2$, $\bTheta = c \in \bbR$ and $\bDelta(\bTheta) = \binom{c}{c} \in \bbR^2$. Plot 4 is the joint space $(\w,c)$ and plot 5 is its projection in \w-space (as in plot 1). In plot 4, the black line is the feasible set $\calF = \{(\w,\bTheta) \in \bbR^P \times \bbR^Q\mathpunct{:}\ \w = \bDelta(\bTheta)\}$, corresponding to the constraints $w_1 = w_2 = c$. In plot 5, the black line is the feasible set $\calF_{\w} = \{\w \in \bbR^P\mathpunct{:}\ \w = \bDelta(\bTheta) \text{ for } \bTheta \in \bbR^Q\}$, corresponding to the constraint $w_1 = w_2$. The red line is the quadratic-penalty method path $(\w(\mu),c(\mu))$, which for this simple case is a straight line segment from the point $(\overline{\w},c^{\text{DC}})$ to the solution $(\w^*,c^*)$. We mark three points: blue \textcolor{blue}{$+$} represents the reference net $\overline{\w}$ at the DC codebook $\bTheta = c^{\text{DC}}$ (the beginning of the path); red \textcolor{red}{$\ast$} is the solution $(\w^*,c^*)$ (the end of the path); and white $\circ$ is the direct compression point $(\bDelta(\bTheta^{\text{DC}}),\bTheta^{\text{DC}}) = \smash{\big( \binom{c^{\text{DC}}}{c^{\text{DC}}},c^{\text{DC}} \big)}$.}
  \label{f:LC-illustration}
\end{figure}

Problem~\eqref{e:compression-problem} can be written as $\min_{\w,\bTheta}{ L(\w) }$ s.t.\ $\w,\bTheta \in \calF$, where the objective function is the loss $L(\w)$ on the real-valued weights and the feasible set on \w\ and the low-dimensional parameters \bTheta\ is:
\begin{equation}
  \label{e:feasible-set}
  \calF = \{(\w,\bTheta) \in \bbR^P \times \bbR^Q\mathpunct{:}\ \w = \bDelta(\bTheta)\}.
\end{equation}
We also define the feasible set in \w-space:
\begin{equation}
  \label{e:feasible-set-w}
  \calF_{\w} = \{\w \in \bbR^P\mathpunct{:}\ \w = \bDelta(\bTheta) \text{ for } \bTheta \in \bbR^Q\}
\end{equation}
which contains all high-dimensional models \w\ that can be obtained by decompressing some low-dimensional model \bTheta. Fig.~\ref{f:LC-illustration} (plots 1--3) illustrates the geometry of the problem in general.

Solving the C step requires minimizing (where we write \w\ instead of $\w - \frac{1}{\mu} \blambda$ for simplicity of notation):
\begin{equation}
  \label{e:compression-mapping}
  \textcolor{blue}{\bPi(\w) = \argmin_{\bTheta}{\norm{\w - \bDelta(\bTheta)}^2}.}
\end{equation}
We call $\bDelta\mathpunct{:}\ \bTheta \in \bbR^Q \rightarrow \w \in \bbR^P$ the \emph{decompression mapping} and $\bPi\mathpunct{:}\ \w \in \bbR^P \rightarrow \bTheta \in \bbR^Q$ the \emph{compression mapping}. In quantization, this has the following meaning:
\begin{itemize}
\item $\bTheta = \{\calC,\Z\}$ consists of the codebook (if the codebook is adaptive) and the assignments of weight-to-codebook-entries. The assignments can be encoded as 1-of-$K$ vectors $\Z^T = (\z_1,\dots,\z_P)$ or directly as $P$ indices in $\{1,\dots,K\}$ for a codebook with $K$ entries.
\item The decompression mapping $\w_{\calC} = \bDelta(\calC,\Z)$ uses the codebook and assignments as a lookup table to generate a real-valued but quantized weight vector $\w_{\calC}$. This vector is used in the L step as a regularizer.
\item The compression mapping $\{\calC,\Z\} = \bPi(\w)$ learns optimally a codebook and assignments given a real-valued, non-quantized weight vector \w\ (using $k$-means or a form of rounding, see section~\ref{s:quant}). All the C step does is solve for the compression mapping.
\end{itemize}
As shown by \citet{Carreir17a}, the compression mapping $\bPi(\w)$ finds the orthogonal projection of \w\ on the feasible set $\calF_{\w}$, which we call $\w_{\calC}$.

For quantization, the geometry of the constrained optimization formulation is as follows. The feasible set \calF\ can be written as the union of a combinatorial number of linear subspaces $\calS_j$ (containing the origin), where $\calS_j$ is of the form $\{ \w_i = \c_k,\ \forall i=1,\dots,P,\ k \in \{1,\dots,K\} \}$. Each such subspace defines a particular assignment of the $P$ weights to the $K$ centroids $\calC = \{\c_1,\dots,\c_K\}$. There are $K^P$ assignments. If we knew the optimal assignment, the feasible set would be a single linear subspace, and the weights could be eliminated (using $\w_i = \c_k$) to yield an unconstrained objective $L(\calC)$ of $K$ tunable vectors (shared weights in neural net parlance), which would be simple to optimize. What makes the problem hard is that we do not know the optimal assignment. Depending on the dimensions $P$ and $K$, these subspaces may look like lines, planes, etc., always passing through the origin in $(\w,\calC)$ space. Geometrically, the union of these $K^P$ subspaces is a feasible set with both a continuous structure (within each subspace) and a discrete one (the number of subspaces is finite but very large).

Fig.~\ref{f:LC-illustration} (plots 4--5) shows the actual geometry for the case of a net with $P = 2$ weights and a codebook with $K=1$ centroid. This can be exactly visualized in 3D $(w_1,w_2,c)$ because the assignment variables $z_{11} = z_{21} = 1$ are redundant and can be eliminated: $\min_{w_1,w_2,c}{ L(w_1,w_2) }$ s.t.\ $w_1 = c$, $w_2 = c$. The compression mapping is easily seen to be $\bPi(\w) = \frac{w_1+w_2}{2} = c$, and $\bDelta(\bPi(\w)) = \frac{w_1+w_2}{2} \binom{1}{1}$ is indeed the orthogonal projection of \w\ onto the diagonal line $w_1 = w_2$ in \w-space (the feasible set). This particular case is, however, misleading in that the constraints involve a single linear subspace rather than the union of a combinatorial number of subspaces. It can be solved simply and exactly by setting $w_1 = w_2 = c$ and eliminating variables into $L(w_1,w_2) = L(c,c)$.

\subsection{Convergence of the LC algorithm}
\label{s:conv}

Convergence of the LC algorithm to a local KKT point (theorem 5.1 in \citealp{Carreir17a}) is guaranteed for smooth problems (continuously differentiable loss $L(\w)$ and decompression mapping $\bDelta(\bTheta)$) if $\mu \rightarrow \infty$ and optimization of the penalty function~\eqref{e:augLag} is done accurately enough for each $\mu$. However, in quantization the decompression mapping $\bDelta(\bTheta)$ is discrete, given by a lookup table, so the theorem does not apply. 

In fact, neural net quantization is an NP-complete problem even in simple cases. For example, consider least-squares linear regression with weights in $\{-1,+1\}$. This corresponds to binarization of a single-layer, linear neural net. The loss $L(\w)$ is quadratic, so the optimization problem is a binary quadratic problem over the weights, which is NP-complete \citep{GareyJohnson79a}. However, the LC algorithm will still converge to a ``local optimum'' in the same sense that the $k$-means algorithm is said to converge to a local optimum: the L step cannot improve given the C step, and vice versa. While this will generally not be the global optimum of problem~\eqref{e:compression-problem}, it will be a good solution in that the loss will be low (because the L step continuously minimizes it in part), and the LC algorithm is guaranteed to converge to a weight vector \w\ that satisfies the quantization constraints (e.g.\ weights in $\{-1,+1\}$ for binarization). Our experiments confirm the effectiveness of the LC algorithm for quantization, consistently outperforming other approaches over a range of codebook types and sizes.

\subsection{Practicalities of the LC algorithm}

We give pseudocode for three representative cases of the resulting LC algorithms: adaptive codebook (fig.~\ref{f:pseudocode-adaptive}), fixed codebook (fig.~\ref{f:pseudocode-fixed}) and binarization with global scale (fig.~\ref{f:pseudocode-bin-scale}).

\begin{figure}[t]
  \centering
  \setlength{\fboxsep}{1ex}
  \framebox{%
    \begin{minipage}[c]{0.80\linewidth}
      \begin{tabbing}
        n \= n \= n \= n \= n \= \kill
        \underline{\textbf{input}} \caja{t}{l}{$K \ge 1$ (codebook size), \\ training data and neural net architecture with weights \w} \\
        $\w \leftarrow \overline{\w} = \argmin_{\w}{ L(\w) }$ \` {\small\textsf{reference net}} \\
        $(\calC,\Z) \leftarrow \text{$k$-means}(\overline{\w})$ \` {\small\textsf{learn codebook and assignments}} \\
        $\w_{\calC} \leftarrow \bDelta(\calC,\Z)$ \` {\small\textsf{quantized reference net}} \\
        $\blambda \leftarrow \0$ \\
        \underline{\textbf{for}} $\mu = \mu_0 < \mu_1 < \dots < \infty$ \+ \\
        $\w \leftarrow \argmin_{\w}{ L(\w) + \smash{\frac{\mu}{2} \norm{\smash{\w - \w_{\calC} - \frac{1}{\mu} \blambda}}^2} }$ \` {\small\textsf{L step: learn net}} \\
        $(\calC,\Z) \leftarrow \text{$k$-means}\big( \w - \frac{1}{\mu} \blambda \big)$ \` {\small\textsf{C step: learn codebook and assignments\dots}} \\
        $\w_{\calC} \leftarrow \bDelta(\calC,\Z)$ \` {\small\textsf{\dots and quantize reference net}} \\
        $\blambda \leftarrow \blambda - \mu (\w - \w_{\calC})$ \` {\small\textsf{Lagrange multipliers}} \\
        \textbf{if} $\norm{\w - \w_{\calC}}$ is small enough \textbf{then} exit the loop \- \\
        \underline{\textbf{return}} $\w_{\calC},\calC,\Z$
      \end{tabbing}
    \end{minipage}
  }
  \caption{Pseudocode for the LC algorithm for quantization of scalar weights with an adaptive codebook, augmented-Lagrangian version. When ran on the reference net $\overline{\w}$, $k$-means is initialized from $k$-means++; when ran in the C step, $k$-means is initialized from the previous iteration's codebook \calC. The C step compression mapping $\bTheta = \bPi(\w)$ is $(\calC,\Z) = \text{$k$-means}(\w)$, where $\calC = \{c_1,\dots,c_K\} \subset \bbR$ is the codebook and $\Z = (z_{ik}) \in \{0,1\}^{PK}$ the assignments, satisfying $\smash{\sum^K_{k=1}{ z_{ik} }} = 1$ for each $i = 1,\dots,P$ (i.e., we use a 1-of-$K$ representation). The quantized (compressed) weights $\w_{\calC} = \bDelta(\calC,\Z) \in \bbR^P$ result from setting the $i$th weight to its assigned codebook entry, $c_{\kappa(i)}$ where $\kappa(i) = k$ if $z_{ik} = 1$.}
  \label{f:pseudocode-adaptive}
\end{figure}

\begin{figure}[t]
  \centering
  \setlength{\fboxsep}{1ex}
  \framebox{%
    \begin{minipage}[c]{0.80\linewidth}
      \begin{tabbing}
        n \= n \= n \= n \= n \= \kill
        \underline{\textbf{input}} \caja{t}{l}{$\calC = \{c_1,\dots,c_K\} \subset \bbR$ (codebook), \\ training data and neural net architecture with weights \w} \\
        $\w \leftarrow \overline{\w} = \argmin_{\w}{ L(\w) }$ \` {\small\textsf{reference net}} \\
        $\kappa(i) \leftarrow \argmin_{k=1,\dots,K}{\abs{\overline{w}_i - c_k}},\ i=1,\dots,P$ \` {\small\textsf{assignments}} \\
        $\w_{\calC} \leftarrow \bDelta(\calC,\kappa)$ \` {\small\textsf{quantized reference net}} \\
        $\blambda \leftarrow \0$ \\
        \underline{\textbf{for}} $\mu = \mu_0 < \mu_1 < \dots < \infty$ \+ \\
        $\w \leftarrow \argmin_{\w}{ L(\w) + \smash{\frac{\mu}{2} \norm{\smash{\w - \w_{\calC} - \frac{1}{\mu} \blambda}}^2} }$ \` {\small\textsf{L step: learn net}} \\
        $\kappa(i) \leftarrow \argmin_{k=1,\dots,K}{\abs{w_i - \smash{\frac{1}{\mu}} \lambda_i - c_k}},\ i=1,\dots,P$ \` {\small\textsf{C step: assignments\dots}} \\
        $\w_{\calC} \leftarrow \bDelta(\calC,\kappa)$ \` {\small\textsf{\dots and quantized net}} \\
        $\blambda \leftarrow \blambda - \mu (\w - \w_{\calC})$ \` {\small\textsf{Lagrange multipliers}} \\
        \textbf{if} $\norm{\w - \w_{\calC}}$ is small enough \textbf{then} exit the loop \- \\
        \underline{\textbf{return}} $\w_{\calC},\calC,\kappa$
      \end{tabbing}
    \end{minipage}
  }
  \caption{Pseudocode for the LC algorithm for quantization of scalar weights with a fixed codebook, augmented-Lagrangian version. For simplicity of notation, we now represent the assignments (of weights to codebook entries) directly by an index $\kappa(i) \in \{1,\dots,K\}$ for each $i = 1,\dots,P$. The quantized weights $\w_{\calC} = \bDelta(\calC,\kappa) \in \bbR^P$ result from setting the $i$th weight to its assigned codebook entry, $c_{\kappa(i)}$.}
  \label{f:pseudocode-fixed}
\end{figure}

\begin{figure}[t]
  \centering
  \setlength{\fboxsep}{1ex}
  \framebox{%
    \begin{minipage}[c]{0.80\linewidth}
      \begin{tabbing}
        n \= n \= n \= n \= n \= \kill
        \underline{\textbf{input}} training data and neural net architecture with weights \w \\
        $\w \leftarrow \overline{\w} = \argmin_{\w}{ L(\w) }$ \` {\small\textsf{reference net}} \\
        $a \leftarrow \smash{\frac{1}{P} \sum^P_{i=1}{ \abs{\overline{w}_i} }}$ \` {\small\textsf{scale}} \\
        $\w_{\calC} \leftarrow a \, \sgn{\overline{\w}}$ \` {\small\textsf{binarized reference net}} \\
        $\blambda \leftarrow \0$ \\
        \underline{\textbf{for}} $\mu = \mu_0 < \mu_1 < \dots < \infty$ \+ \\
        $\w \leftarrow \argmin_{\w}{ L(\w) + \smash{\frac{\mu}{2} \norm{\smash{\w - \w_{\calC} - \frac{1}{\mu} \blambda}}^2} }$ \` {\small\textsf{L step: learn net}} \\
        $a \leftarrow \smash{\frac{1}{P} \sum^P_{i=1}{ \abs{w_i} }}$ \` {\small\textsf{C step: scale\dots}} \\
        $\w_{\calC} \leftarrow a \, \sgn{\w}$ \` {\small\textsf{\dots and binarized net}} \\
        $\blambda \leftarrow \blambda - \mu (\w - \w_{\calC})$ \` {\small\textsf{Lagrange multipliers}} \\
        \textbf{if} $\norm{\w - \w_{\calC}}$ is small enough \textbf{then} exit the loop \- \\
        \underline{\textbf{return}} $\w_{\calC},a$
      \end{tabbing}
    \end{minipage}
  }
  \caption{Pseudocode for the LC algorithm for binarization of scalar weights into $\{-1,+1\}$ with an adaptive scale $a > 0$, augmented-Lagrangian version. The sign function applies elementwise to compute the binarized weights $\w_{\calC}$.}
  \label{f:pseudocode-bin-scale}
\end{figure}

As usual with path-following algorithms, ideally one would follow the path of iterates $(\w(\mu),\bTheta(\mu))$ closely until $\mu \rightarrow \infty$, by increasing the penalty parameter $\mu$ slowly. In practice, in order to reduce computing time, we increase $\mu$ more aggressively by following a multiplicative schedule $\mu_k = \mu_0 a^k$ for $k = 0,1,2\dots$ where $\mu_0>0$ and $a>1$. However, it is important to use a small enough $\mu_0$ that allows the algorithm to explore the solution space before committing to specific assignments for the weights.

The L step with a large training set typically uses SGD. As recommended by \citet{Carreir17a}, we use a clipped schedule $\{\eta'_{t}\}^{\infty}_{t=0}$ for the learning rates of the form $\eta'_{t} = \smash{\min{ \big( \eta_{t},\frac{1}{\mu} \big) }},\ t = 0,1,2\dots$, where $t$ is the epoch index and $\smash{\{\eta_{t}\}^{\infty}_{t=0}}$ is a schedule for the reference net (i.e., for $\mu = 0$). This ensures convergence and avoids erratic updates as $\mu$ becomes large.

We initialize $\blambda = \0$ and $(\w,\bTheta) = (\overline{\w},\bTheta^{\text{DC}})$, i.e., to the reference net and direct compression, which is the exact solution for $\mu \rightarrow 0^+$, as we show in the next section. We stop the LC algorithm when $\norm{\w - \bDelta(\calC,\Z)}$ is smaller than a set tolerance, i.e., when the real-valued and quantized weights are nearly equal. We take as solution $\w_{\calC} = \bDelta(\calC,\Z)$, i.e., the quantized weights using the codebook \calC\ and assignments \Z.

The runtime of the C step is negligible compared to that of the L step. With a fixed codebook, the C step is a simple assignment per weight. With an adaptive codebook, the C step runs $k$-means, each iteration of which is linear on the number of weights $P$. The number of iterations that $k$-means runs is a few tens in the first $k$-means (initialized by $k$-means++, on the reference weights) and just about one in subsequent C steps (because $k$-means is warm-started), as seen in our experiments. So the runtime is dominated by the L steps, i.e., by optimizing the loss.

\subsection{Direct compression and iterated direct compression}
\label{s:DC}

The quadratic-penalty and augmented-Lagrangian methods define a path of iterates $(\w(\mu),\bTheta(\mu))$ for $\mu \ge 0$ that converges to a local solution as $\mu \rightarrow \infty$. The beginning of this path is of special importance, and was called \emph{direct compression (DC)} by \citet{Carreir17a}. Taking the limit $\mu \rightarrow 0^+$ and assuming an initial $\blambda=\0$, we find that $\w(0^+) = \argmin_{\w}{ L(\w) } \equiv \overline{\w}$ and $\bTheta(0^+) = \bPi(\overline{\w}) = \argmin_{\bTheta}{ \smash{\norm{\overline{\w} - \bDelta(\bTheta)}^2} } \equiv \smash{\bTheta^{\text{DC}}}$. Hence, this corresponds to training a reference, non-quantized net $\overline{\w}$ and then quantizing it regardless of the loss (or equivalently projecting $\overline{\w}$ on the feasible set). As illustrated in fig.~\ref{f:LC-illustration}, this is suboptimal (i.e., it does not produce the compressed net with lowest loss), more so the farther the reference is from the feasible set. This will happen when the feasible set is small, i.e., when the codebook size $K$ is small (so the compression level is high). Indeed, our experiments show that for large $K$ (around 32 bits/weight) then DC is practically identical to the result of the LC algorithm, but as $K$ decreases (e.g.\ 1 to 4 bits/weight) then the loss of DC becomes larger and larger than that of the LC algorithm.

A variation of direct compression consists of ``iterating'' it, as follows. We first optimize $L(\w)$ to obtain $\overline{\w}$ and then quantize it with $k$-means into $\bTheta^{\text{DC}}$. Next, we optimize $L(\w)$ again but initializing \w\ from $\bDelta(\bTheta^{\text{DC}})$, and then we compress it; etc. This was called ``iterated direct compression (iDC)'' by \citet{Carreir17a}. iDC should not improve at all over DC if the loss optimization was exact and there was a single optimum: it simply would cycle forever between the reference weights $\overline{\w}$ and the DC weights $\bDelta(\bTheta^{\text{DC}})$. However, in practice iDC may improve somewhat over DC, for two reasons. 1) With local optima of $L(\w)$, we might converge to a different optimum after the quantization step (see fig.~\ref{f:LC-illustration} plot 2). However, at some point this will end up cycling between some reference net (some local optimum of $L(\w)$) and its quantized net. 2) In practice, SGD-based optimization of the loss with large neural nets is approximate; we stop SGD way before it has converged. This implies the iterates never fully reach $\overline{\w}$, and keep oscillating forever somewhere in between $\overline{\w}$ and $\bDelta(\bTheta^{\text{DC}})$.

DC and iDC have in fact been proposed recently for quantization, although without the context that our constrained optimization framework provides. \citet{Gong_15a} applied $k$-means to quantize the weights of a reference net, i.e., DC. The ``trained quantization'' of \citet{Han_15a} tries to improve over this by iterating the process, i.e., iDC. In our experiments, we verify that neither DC not iDC converge to a local optimum of problem~\eqref{e:compression-problem}, while our LC algorithm does.

\section{Solving the C step: compression by quantization}
\label{s:quant}

The C step consists of solving the optimization problem of eq.~\eqref{e:compression-mapping}: $\bPi(\w) = \argmin_{\bTheta}{\norm{\w - \bDelta(\bTheta)}^2}$, where $\w \in \bbR^P$ is a vector of real-valued weights. This is a quadratic distortion (or least-squares error) problem, and this was caused by selecting a quadratic penalty in the augmented Lagrangian~\eqref{e:augLag}. It is possible to use other penalties (e.g.\ using the $\ell_1$ norm), but the quadratic penalty gives rise to simpler optimization problems, and we focus on it in this paper. We now describe how to write quantization as a mapping \bDelta\ in parameter space and how to solve the optimization problem~\eqref{e:compression-mapping}.

Quantization consists of approximating real-valued vectors in a training set by vectors in a codebook. Since in our case the vectors are weights of a neural net, we will write the training set as $\{\w_1,\dots,\w_P\}$. Although in practice with neural nets we quantize scalar weight values directly (not weight vectors), we develop the formulation using vector quantization for generality. Hence, if we use a codebook $\calC = \{\c_1,\dots,\c_K\}$ with $K \ge 1$ entries, the number of bits used to store each weight vector $\w_i$ is $\ceil{\log_2{K}}$.

We consider two types of quantization: using an adaptive codebook, where we learn the optimal codebook for the training set; and using a fixed codebook, which is then not learned (although we will consider learning a global scale).

\subsection{Adaptive codebook}
\label{s:quant-adaptive}

The decompression mapping is a table lookup $\w_i = \c_{\kappa(i)}$ for each weight vector $i=1,\dots,P$ in the codebook $\calC = \{\c_1,\dots,\c_K\}$, where $\kappa\mathpunct{:}$ $\{1,\dots,P\}$ $\rightarrow$ $\{1,\dots,K\}$ is a discrete mapping that assigns each weight vector to one codebook vector. The compression mapping results from finding the best (in the least-squares sense) codebook \calC\ and mapping $\kappa$ for the ``dataset'' $\w_1,\dots,\w_P$, i.e., from solving the optimization problem
\begin{equation}
  \label{e:quant-mapping}
  \min_{\calC,\kappa}{ \sum^P_{i=1}{ \norm{\w_i - \c_{\kappa(i)}}^2 } } \qquad \equiv \textcolor{blue}{\qquad \min_{\calC,\Z}{ \sum^{P,K}_{i,k=1}{ z_{ik} \norm{\w_i - \c_k}^2 } } \quad \text{s.t.} \quad
  \begin{cases}
    \Z \in \{0,1\}^{P\times K} \\
    \sum^K_{k=1}{ z_{ik} } = 1,\ i = 1,\dots,P
  \end{cases}}
\end{equation}
which we have rewritten equivalently using binary assignment variables $\Z^T = (\z_1,\dots,\z_P)$. This follows by writing $\c_{\kappa(i)} = \smash{\sum^K_{k=1}{ z_{ik} \c_k }}$ where $z_{ik} = 1$ if $k = \kappa(i)$ and $0$ otherwise, and verifying by substituting the $z_{ik}$ values that the following holds:
\begin{equation*}
  \norm{\w_i - \c_{\kappa(i)}}^2 = \norm[\Big]{\smash{\w_i - \sum^K_{k=1}{ z_{ik} \c_k }}}^2 = \sum^K_{k=1}{ z_{ik} \norm{\w_i - \c_k}^2 }.
\end{equation*}
So in this case the low-dimensional parameters are \textcolor{blue}{$\bTheta = \{\calC,\Z\}$}, the decompression mapping can be written elementwise as \textcolor{blue}{$\w_i = \c_{\kappa(i)} = \smash{\sum^K_{k=1}{ z_{ik} \c_k }}$} for $i = 1,\dots,P$, and the compression mapping \textcolor{blue}{$\{\calC,\Z\} = \bPi(\w)$} results from running the $k$-means algorithm. The low-dimensional parameters are of two types: the assignments $\z_1,\dots,\z_P$ are ``private'' (each weight $\w_i$ has its own $\z_i$), and the codebook \calC\ is ``shared'' by all weights. In the pseudocode of fig.~\ref{f:pseudocode-adaptive}, we write the optimally quantized weights as $\w_{\calC} = \bDelta(\calC,\Z)$.

Problem~\eqref{e:quant-mapping} is the well-known quadratic distortion problem \citep{GershoGray92a}. It is NP-complete and it is typically solved approximately by $k$-means using a good initialization, such as that of $k$-means++ \citep{ArthurVassil07a}. As is well known, $k$-means is an alternating optimization algorithm that iterates the following two steps: in the assignment step we update the assignments $\z_1,\dots,\z_P$ independently given the centroids (codebook); in the centroid step we update the centroids $\c_1,\dots,\c_K$ independently by setting them to the mean of their assigned points. Each iteration reduces the distortion or leaves it unchanged. The algorithm converges in a finite number of iterations to a local optimum where \Z\ cannot improve given \calC\ and vice versa.

In practice with neural nets we quantize scalar weight values directly, i.e., each $w_i$ is a real value. Computationally, $k$-means is considerably faster with scalar values than with vectors. If the vectors have dimension $D$, with $P$ data points and $K$ centroids, each iteration of $k$-means takes $\calO(PKD)$ runtime because of the assignment step (the centroid step is $\calO(PD)$, by scanning through the $P$ points and accumulating each mean incrementally). But in dimension $D=1$, each iteration can be done exactly in $\calO(P \log{K})$, by using a binary search over the sorted centroids in the assignment step, which then takes $\calO(K \log{K})$ for sorting and $\calO(P \log{K})$ for assigning, total $\calO(P \log{K})$. 

\subsubsection{Why $k$-means?}

The fact that we use $k$-means in the C step is not an arbitrary choice of a quantization algorithm (among many possible such algorithms we could use instead). It is a necessary consequence of two assumptions: 1) The fact that we want to assign weights to elements of a codebook, which dictates the form of the decompression mapping $\w = \bDelta(\calC,\Z)$. This is not really an assumption because any form of quantization works like this. 2) That the penalty used in the augmented Lagrangian is quadratic, so that the C step is a quadratic distortion problem.

We could choose a different penalty instead of the quadratic penalty $\smash{\norm{\w - \bDelta(\calC,\Z)}^2_2}$, as long as it is zero if the constraint $\w = \bDelta(\calC,\Z)$ is satisfied and positive otherwise (for example, the $\ell_1$ penalty). In the grand scheme of things, the choice of penalty is not important, because the role of the penalty is to enforce the constraints gradually, so that in the limit $\mu \rightarrow \infty$ the constraints are satisfied and the weights are quantized: $\w = \bDelta(\calC,\Z)$. Any penalty satisfying the positivity condition above will achieve this. The choice of penalty does have two effects: it may change the local optimum we converge to (although it is hard to have control on this); and, more importantly, it has a role in the optimization algorithm used in the L and C steps: the quadratic penalty is easier to optimize. As an example, imagine we used the $\ell_1$ penalty $\norm{\w - \bDelta(\calC,\Z)}_1$. This means that the L step would have the form:
\begin{equation*}
  \min_{\w}{ L(\w) + \frac{\mu}{2} \norm[\Big]{\w - \bDelta(\bTheta) - \frac{1}{\mu} \blambda}_1 },
\end{equation*}
that is, an $\ell_1$-regularized loss. This is a nonsmooth problem. One can develop algorithms to optimize it, but it is harder than with the quadratic regularizer. The C step would have the form (again we write \w\ instead of $\w - \smash{\frac{1}{\mu}} \blambda$ for simplicity of notation):
\begin{equation*}
  \min_{\bTheta}{ \norm{\w - \bDelta(\bTheta)}_1 } \qquad \Longleftrightarrow \qquad \min_{\calC,\Z}{ \sum^{P,K}_{i,k=1}{ z_{ik} \norm{\w_i - \c_k}_1 } } \quad \text{s.t.} \quad
  \begin{cases}
    \Z \in \{0,1\}^{P\times K} \\
    \sum^K_{k=1}{ z_{ik} } = 1,\ i = 1,\dots,P.
  \end{cases}
\end{equation*}
With scalar weights $w_1,\dots,w_P$, this can be solved by alternating optimization as in $k$-means: the assignment step is identical, but the centroid step uses the median instead of the mean of the points assigned to each centroid ($k$-medians algorithm). There are a number of other distortion measures developed in the quantization literature \citep[section~10.3]{GershoGray92a} that might be used as penalty and are perhaps convenient with some losses or applications. With a fixed codebook, as we will see in the next section, the form of the C step is the same regardless of the penalty.

On the topic of the choice of penalty, a possible concern one could raise is that of outliers in the data. When used for clustering, $k$-means is known to be sensitive to outliers and nonconvexities of the data distribution. Consider the following situations, for simplicity using just $K=1$ centroid in 1D. First, if the dataset has an outlier, it will pull the centroid towards it, away from the rest of the data (note this is not a local optima issue; this is the global optimum). For compression purposes, it may seem a waste of that centroid not to put it where most of the data is. With the $\ell_1$ penalty, the centroid would be insensitive to the outlier. Second, if the dataset consists of two  separate groups, the centroid will end up in the middle of both, where there is no data, for both $k$-means and the $\ell_1$ penalty. Again, this may seem a waste of the centroid. Other clustering algorithms have been proposed to ensure the centroids lie where there is distribution mass, such as the $k$-modes algorithm \citep{CarreirWang13a,WangCarreir14c}. However, these concerns are misguided, because neural net compression is not a data modeling problem: one has to consider the overall LC algorithm, not the C step in isolation. While in the C step the centroids approach the data (the weights), in the L step the weights approach the centroids, and in the limit $\mu \rightarrow \infty$ both coincide, the distortion is zero and there are no outliers. It is of course possible that the LC algorithm converge to a bad local optimum of the neural net quantization, which is an NP-complete problem, but this can happen for various reasons. In section~\ref{s:expts:regression} of the experiments we run the LC algorithm in a model whose weights contain clear outliers and demonstrate that the solution found makes sense.

\subsection{Fixed codebook}
\label{s:quant-fixed}

Now, we consider quantization using a fixed codebook%
\footnote{We can also achieve \emph{pruning} together with quantization by having one centroid be fixed to zero. We study this in more detail in a future paper.},
i.e., the codebook entries $\calC = \{\c_1,\dots,\c_K\}$ are fixed and we do not learn them, we learn only the weight assignments $\Z^T = (\z_1,\dots,\z_P)$. In this way we can derive algorithms for compression of the weights based on approaches such as binarization or ternarization, which have been also explored in the literature of neural net compression, implemented as modifications to backpropagation (see section~\ref{s:related:quant-fixed}).

The compression mapping $\bPi(\w)$ of eq.~\eqref{e:compression-mapping} now results from solving the optimization problem
\begin{equation}
  \label{e:quant-fixed-mapping}
  \textcolor{blue}{\min_{\Z}{ \sum^{P,K}_{i,k=1}{ z_{ik} \norm{\w_i - \c_k}^2 } } \quad \text{s.t.} \quad
  \begin{cases}
    \Z \in \{0,1\}^{P\times K} \\
    \sum^K_{k=1}{ z_{ik} } = 1,\ i = 1,\dots,P.
  \end{cases}}
\end{equation}
This is not NP-complete anymore, unlike in the optimization over codebook and assignments jointly in~\eqref{e:quant-mapping}. It has a closed-form solution for each $\z_i$ separately where we assign $\w_i$ to \textcolor{blue}{$\kappa(i) = \smash{\argmin_{k = 1,\dots,K}{ \norm{\w_i - \c_k}^2 }}$}, with ties broken arbitrarily, for $i = 1,\dots,P$. That is, each weight $\w_i$ is compressed as its closest codebook entry $\c_{\kappa(i)}$ (in Euclidean distance). Therefore, we can write the compression mapping $\bTheta = \bPi(\w)$ explicitly as \textcolor{blue}{$\bPi(\w_i) = \smash{\c_{\kappa(i)}}$} separately for each weight $\w_i$, $i = 1,\dots,P$.

So in this case the low-dimensional parameters are \textcolor{blue}{$\bTheta = \Z$} (or \textcolor{blue}{$\bTheta = \{\kappa(1),\dots,\kappa(P)\}$}), the decompression mapping can be written elementwise as \textcolor{blue}{$\w_i = \c_{\kappa(i)} = \smash{\sum^K_{k=1}{ z_{ik} \c_k }}$} for $i = 1,\dots,P$ (as with the adaptive codebook), and the compression mapping \textcolor{blue}{$\Z = \bPi(\w)$} can also be written elementwise as \textcolor{blue}{$\z_i = \bPi(\w_i)$} (or \textcolor{blue}{$\kappa(i) = \Pi(\w_i)$}) for $i = 1,\dots,P$. The low-dimensional parameters are all private (the assignments $\z_1,\dots,\z_P$ or $\kappa(1),\dots,\kappa(P)$). The codebook \calC\ is shared by all weights, but it is not learned. In the pseudocode of fig.~\ref{f:pseudocode-fixed}, we use the notation $\w_{\calC} = \bDelta(\calC,\kappa) = (\c_{\kappa(1)},\dots,\c_{\kappa(P)})$ to write the optimally quantized weights.

This simplifies further in the scalar case, i.e., when the weights $\w_i$ to be quantized are scalars. Here, we can write the codebook $\calC = \{c_1,c_2,\dots,c_K\}$ as an array of scalars sorted increasingly, $-\infty < c_1 < c_2 < \dots < c_K < \infty$. The elementwise compression mapping $\Pi(w_i) = \kappa(i) = \argmin_{k = 1,\dots,K}{ \abs{w_i - c_k} }$ can be written generically for $t \in \bbR$ as:
\begin{equation}
  \label{e:quant-fixed-mapping-scalar}
  \textcolor{blue}{
    \Pi(t) =
    \begin{cases}
      1, & \text{if } t < \frac{1}{2}(c_1+c_2) \\
      2, & \text{if } \frac{1}{2}(c_1+c_2) \le t < \frac{1}{2}(c_2+c_3) \\
      \dots \\
      K, & \text{if } \frac{1}{2}(c_{K-1}+c_K) \le t
    \end{cases}}
\end{equation}
since the codebook defines Voronoi cells that are the intervals between midpoints of adjacent centroids. This can be written more compactly as $\Pi(t) = \kappa$ where $\kappa \in \{1,\dots,K\}$ satisfies $\frac{1}{2} (c_{\kappa-1} + c_{\kappa}) \le t < \frac{1}{2} (c_{\kappa} + c_{\kappa+1})$ and we define $c_0 = -\infty$ and $c_{K+1} = \infty$. Computationally, this can be done in $\calO(\log{K})$ using a binary search, although in practice $K$ is small enough that a linear search in $\calO(K)$ makes little difference. To use the compression mapping $\Pi(t)$ in the C step of the LC algorithm given in section~\ref{s:LC}, $t$ equals either a scalar weight $w_i$ for the quadratic-penalty method, or a shifted scalar weight $w_i - \smash{\frac{1}{\mu}} \lambda_i$ for the augmented Lagrangian method. The L step of the LC algorithm always takes the form given in eq.~\eqref{e:Lstep}.

Again, this quantization algorithm in the C step is not an arbitrary choice, it follows necessarily from the way any codebook-based quantization works. Furthermore, and unlike the adaptive codebook case, with scalar weights the solution~\eqref{e:quant-fixed-mapping-scalar} is independent of the choice of penalty, because the order of the real numbers is unique (so using a quadratic or an $\ell_1$ penalty will result in the same step).

\paragraph{Application to binarization, ternarization and powers-of-two}

Some particular cases of the codebook are of special interest because their implementation is very efficient: binary $\{-1,+1\}$, ternary $\{-1,0,+1\}$ and general powers-of-two $\{0,\pm 1,\pm 2^{-1},\dots,\pm 2^{-C}\}$. These are all well known in digital filter design, where one seeks to avoid floating-point multiplications by using fixed-point binary arithmetic and powers-of-two or sums of powers-of-two multipliers (which result in shift or shift-and-add operations instead). This accelerates the computation and requires less hardware.

We give the solution of the C step for these cases in fig.~\ref{f:fixed-codebook} (see proofs in the appendix). Instead of giving the compression mapping $\Pi(t)$, we give directly a \emph{quantization operator} $q\mathpunct{:}\ \bbR \rightarrow \calC$ that maps a real-valued weight to its optimal codebook entry. Hence, $q$ corresponds to compressing then decompressing the weights, elementwise: $q(t) = \Delta(\calC,\Pi(t))$, where $t \in \bbR$ is a scalar weight. In the expressions for $q(t)$, we define the floor function for $t \in \bbR$ as $\floor{t} = i$ if $i \le t < i+1$ and $i$ is integer, and the sign function as follows:
\begin{equation}
  \label{e:sgn}
  \sgn{t} =
  \begin{cases}
    -1, & \text{if } t < 0 \\
    +1, & \text{if } t \ge 0.
  \end{cases}
\end{equation}
Note that the generic $k$-means algorithm (which occurs in the C step of our LC algorithm) solves problem~\eqref{e:quant-fixed-mapping-scalar}, and hence its particular cases, exactly in one iteration: the centroid step does nothing (since the centroids are not learnable) and the assignment step is identical to the expressions for $\Pi(t)$ in eq.~\eqref{e:quant-fixed-mapping-scalar} or for $q(t)$ in fig.~\ref{f:fixed-codebook}. However, the expressions in fig.~\ref{f:fixed-codebook} are more efficient, especially for the powers-of-two case, which runs in $\calO(1)$ (while the generic $k$-means assignment step would run in $\calO(\log{C})$).

\begin{figure}[p]
  \centering
  \begin{tabular}[c]{@{}c@{}}
    \psfrag{x}[t][t]{$t$}
    \psfrag{y}[][][1][-90]{$q(t)$}
    \psfrag{bin}[r][r]{\textcolor{green}{binarization}~}
    \psfrag{ter}[tl][bl]{\textcolor{red}{ternarization}}
    \psfrag{pow2}{~\textcolor{blue}{\caja{c}{c}{powers of two \\ ($C=5$)}}}
    \includegraphics[width=0.60\linewidth]{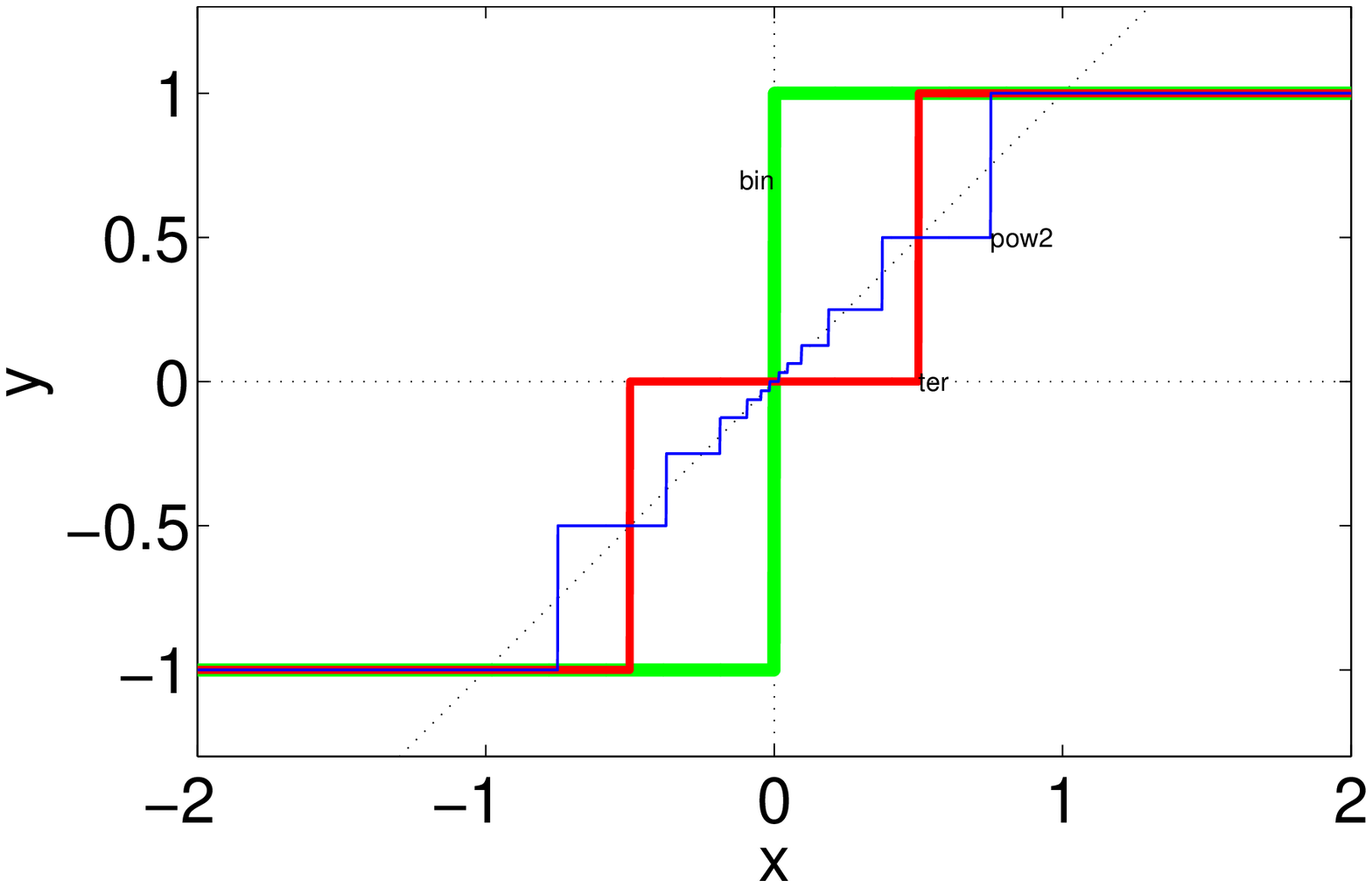}
  \end{tabular} \\[3ex]
  \begin{tabular}[c]{@{}lcl@{}}
    \toprule
    Name & codebook \calC & \multicolumn{1}{c@{}}{\caja{c}{c}{scalar quantization operator \\ $q(t) = \alpha(t) \, \sgn{t}$}} \\
    \midrule
    Binarization & $\{-1,{+1}\}$ & $\alpha(t) = 1$ \\
    Binarization with scale$^{\dagger}$ & $\{-a,{+a}\}$ & $\alpha(t) = a$ \\
    \\
    Ternarization & $\{-1,0,{+1}\}$ & $\alpha(t) =
    \begin{cases}
      0, & \abs{t} < \frac{1}{2} \\
      1, & \abs{t} \ge \frac{1}{2}.
    \end{cases}$ \\
    Ternarization with scale$^{\dagger}$ & $\{-a,0,{+a}\}$ & $\alpha(t) =
    \begin{cases}
      0, & \abs{t} < \frac{1}{2} a \\
      a, & \abs{t} \ge \frac{1}{2} a.
    \end{cases}$ \\
    \\
    Powers of two & \caja{c}{c}{$\{0,\pm 1,\pm 2^{-1},\dots,\pm 2^{-C}\}$ \\ where $C \ge 0$ is integer} & $\alpha(t) =
    \begin{cases}
      0, & f > C+1 \\
      1, & f \le 0 \\
      2^{-C}, & f \in (C,C+1] \\
      2^{-\floor{f + \log_2{\frac{3}{2}}}}, & \text{otherwise}
    \end{cases}$ \\
    & & where $f = -\log_2{\abs{t}}$. \\
    \bottomrule \\
    \multicolumn{3}{@{}l@{}}{$^{\dagger}$The scale is $a = \displaystyle\frac{1}{j^*} \sum^{j^*}_{i=1}{ \abs{w_i} }$ with $j^* =
      \begin{cases}
        P, & \text{for binarization} \\
        \displaystyle\argmax_{1 \le j \le P}{ \frac{1}{\sqrt{j}} \sum^j_{i=1}{ \abs{w_i} } }, & \text{for ternarization.}
      \end{cases}$}
  \end{tabular}
  \caption{Scalar quantization operator $q\mathpunct{:}\ \bbR \rightarrow \calC$ using a fixed codebook $\calC = \{c_1,\dots,c_K\} \subset \bbR$ for some particular cases of interest (the general case without scale is given by eq.~\eqref{e:quant-fixed-mapping-scalar}). The input to the quantization is a set of $P$ real-valued weights; for the ternarization with scale these weights must be sorted in decreasing magnitude: $\abs{w_1} \ge \abs{w_2} \ge \dots \ge \abs{w_P}$. In $q(t)$, $t$ represents any weight $w_i$, and $q(w_i)$ gives the optimally quantized $w_i$. This solves the C step in the LC algorithm for quantization. See proofs in the appendix.}
  \label{f:fixed-codebook}
\end{figure}

\subsubsection{Fixed codebook with adaptive scale}
\label{s:quant-fixed:scale}

Fixed codebook values such as $\{-1,+1\}$ or $\{-1,0,+1\}$ may produce a large loss because the good weight values may be quite bigger or quite smaller than $\pm 1$. One improvement is to rescale the weights, or equivalently rescale the codebook elements, by a scale parameter $a \in \bbR$, which is itself learned. The low-dimensional parameters now are \textcolor{blue}{$\bTheta = \{a,\Z\}$}, where $a$ is a shared parameter and the $\z_1,\dots,\z_P$ are private. The decompression mapping can be written elementwise as \textcolor{blue}{$\w_i = a \, \c_{\kappa(i)} = a \, \smash{\sum^K_{k=1}{ z_{ik} \c_k }}$} for $i = 1,\dots,P$. The compression mapping \textcolor{blue}{$\{a,\Z\} = \bPi(\w)$} results from solving the optimization problem
\begin{equation}
  \label{e:quant-fixed-mapping-scale}
  \textcolor{blue}{\min_{\Z,a}{ \sum^{P,K}_{i,k=1}{ z_{ik} \norm{\w_i - a \, \c_k}^2 } } \quad \text{s.t.} \quad
  \begin{cases}
    \Z \in \{0,1\}^{P\times K} \\
    \sum^K_{k=1}{ z_{ik} } = 1,\ i = 1,\dots,P.
  \end{cases}}
\end{equation}
In general, this can be solved by alternating optimization over \Z\ and $a$:
\begin{itemize}
\item Assignment step: assign $\w_i$ to \textcolor{blue}{$\kappa(i) = \smash{\argmin_{k = 1,\dots,K}{ \norm{\w_i - a \, \c_k}^2 }}$} for $i = 1,\dots,P$.
\item Scale step: \textcolor{blue}{$a = \smash{\big( \sum^{P,K}_{i,k=1}{ z_{ik} \w^T_i \c_k} \big) / \big( \sum^{P,K}_{i,k=1}{ z_{ik} \norm{\c_k}^2} \big)}$}.
\end{itemize}
Like $k$-means, this will stop in a finite number of iterations, and may converge to a local optimum. With scalar weights, each iteration is $\calO(P \log{K})$ by using binary search in the assignment step and incremental accumulation in the scale step.

\paragraph{Application to binarization and ternarization with scale}

For some special cases we can solve problem~\eqref{e:quant-fixed-mapping-scale} exactly, without the need for an iterative algorithm. We give the solution for binarization and ternarization with scale in fig.~\ref{f:fixed-codebook} (see proofs in the appendix). Again, we give directly the scalar quantization operator $q\mathpunct{:}\ \bbR \rightarrow \calC$. The form of the solution is a rescaled version of the case without scale, where the optimal scale $a > 0$ is the average magnitude of a certain set of weights. Note that, given the scale $a$, the weights can be quantized elementwise by applying $q$, but solving for the scale involves all weights $w_1,\dots,w_P$.

Some of our quantization operators are equal to some rounding procedures used in previous work on neural net quantization: binarization (without scale) by taking the sign of the weight is well known, and our formula for binarization with scale is the same as in \citet{Rasteg_16a}. Ternarization with scale was considered by \citep{Li_16b}, but the solution they give is only approximate; the correct, optimal solution is given in our theorem~\ref{th:ter-scale}. As we have mentioned before, those approaches incorporate rounding in the backpropagation algorithm in a heuristic way and the resulting algorithm does not solve problem~\eqref{e:compression-problem}. In the framework of the LC algorithm, the solution of the C step (the quantization operator) follows necessarily; there is no need for heuristics.

It is possible to consider more variations of the above, such as a codebook $\calC = \{-a,+b\}$ or $\{-a,0,+b\}$ with learnable scales $a,b > 0$, but there is little point to it. We should simply use a learnable codebook $\calC = \{c_1,c_2\}$ or $\{c_1,0,c_2\}$ without restrictions on $c_1$ or $c_2$ and run $k$-means in the C step.

Computing the optimal scale $a$ with $P$ weights has a runtime $\calO(P)$ in the case of binarization with scale and $\calO(P \log{P})$ in the case of ternarization with scale. In ternarization, the sums can be done cumulatively in $\calO(P)$, so the total runtime is dominated by the sort, which is $\calO(P \log{P})$. It may be possible to avoid the sort using a heap and reduce the total runtime to $\calO(P)$.

\section{Experiments}
\label{s:expts}

We evaluate our learning-compression (LC) algorithm for quantizing neural nets of different sizes with different compression levels (codebook sizes $K$), in several tasks and datasets: linear regression on MNIST and classification on MNIST and CIFAR10. We compare LC with direct compression (DC) and iterated direct compression (iDC), which correspond to the previous works of \citet{Gong_15a} and \citet{Han_15a}, respectively. By using $K=2$ codebook values, we also compare with BinaryConnect \citep{Courbar_15a}, which aims at learning binary weights. In summary, our experiments 1) confirm our theoretical arguments about the behavior of (i)DC, and 2) show that LC achieves comparable loss values (in training and test) to those algorithms with low compression levels, but drastically outperforms all them at high compression levels (which are the more desirable in practice). We reach the maximum possible compression (1 bit/weight) without significant error degradation in all networks we describe (except in the linear regression case).

We used the Theano \citep{Theano16a} and Lasagne \citep{Dielem_15a} libraries. Throughout we use the augmented Lagrangian, because we found it not only faster but far more robust than the quadratic penalty, in particular in setting the SGD hyperparameters. We initialize all algorithms from a reasonably (but not necessarily perfectly) well-trained reference net. The initial iteration ($\mu=0$) for LC gives the DC solution. The C step (also for iDC) consists of $k$-means ran till convergence, initialized from the previous iteration's centroids (warm-start). For the first compression, we use the $k$-means++ initialization \citep{ArthurVassil07a}. This first compression may take several tens of $k$-means iterations, but subsequent ones need very few, often just one (figs.~\ref{f:regression} and~\ref{f:kmeans-its}).

We report the loss and classification error in training and test. We only quantize the multiplicative weights in the neural net, not the biases. This is because the biases span a larger range than the multiplicative weights, hence requiring higher precision, and anyway there are very few biases in a neural net compared to the number of multiplicative weights.

We calculate \emph{compression ratios} as
\begin{equation}
  \label{e:compression-ratio}
  \rho(K) = \text{\#bits(reference) / \#bits(quantized)}
\end{equation}
where:
\begin{itemize}
\item \#bits(reference) $= (P_1+P_0) b$;
\item \#bits(quantized) $= P_1 \ceil{\log_2{K}} + (P_0+K) b$, where $K b$ is the size of the codebook;
\item $P_1$ and $P_0$ are the number of multiplicative weights and biases, respectively;
\item $K$ is the codebook size;
\item and we use 32-bit floats to represent real values (so $b=32$). Note that it is important to quote the base value of $b$ or otherwise the compression ratio is arbitrary and can be inflated. For example, if we set $b = 64$ (double precision) all the compression ratios in our experiments would double.
\end{itemize}
Since for our nets $P_0 \ll P_1$, we have $\rho(K) \approx b/\log_2{K}$.

\subsection{Interplay between loss, model complexity and compression level}

Firstly, we conduct a simple experiment to understand the interplay between loss, model complexity and compression level, here given by classification error, number of hidden units and codebook size, respectively. One important reason why compression is practically useful is that it may be better to train a large, accurate model and compress it than to train a smaller model and not compress it in the first place (there has been some empirical evidence supporting this, e.g.\ \citealp{Denil_13a}). Also, many papers show that surprisingly large compression levels are possible with some neural nets (in several of our experiments with quantization, we can quantize all the way to one bit per weight with nearly no loss degradation). Should we expect very large compression levels without loss degradation in general?

The answer to these questions depends on the relation between loss, model complexity and compression. Here, we explore this experimentally in a simple setting: a classification neural net with inputs of dimension $D$, outputs of dimension $d$ (number of classes) and $H$ hidden, tanh units, fully connected, trained to minimize the average cross-entropy. We use our LC algorithm to quantize the net using a codebook of size $K$. The size $C(K,H)$ in bits of the resulting nets is as follows (assuming floating-point values of $b = 32$ bits). For the reference (non-quantized net, ``$K = \infty$''), $C(\infty,H) = (D+d) H b$ (multiplicative weights) plus $(H+d)b$ (biases), total $C(\infty,H) \approx (D+d) H b$. For a quantized net, this is the sum of $(D+d) H \log_2{K}$ (for the quantized weights), $(H+d) b$ (for the non-quantized biases) and $K b$ (for the codebook), total $C(K,H) \approx (D+d) H \log_2{K}$.

We explore the space of optimal nets over $H$ and $K$ in order to determine what the best operational point $(K^*,H^*)$ is in order to achieve a target loss $L_{\text{max}}$ with the smallest net, that is, we want to solve the following optimization problem:
\begin{equation*}
  \min_{K,H}{ C(K,H) } \qquad \text{s.t.} \qquad L(K,H) \le L_{\text{max}}.
\end{equation*}

We use the entire MNIST training set of 60\,000 handwritten digit images, hence $D = 784$ and $d = 10$. We train a reference net of $H$ units for $H \in \{2,\dots,40\}$ and compress it using a codebook of size $K$ for $\log_2{K} \in \{1,\dots,8\}$. The training procedure is exactly as for the LeNet300 neural net discussed later. Fig.~\ref{f:L-K-H} plots the loss $L(K,H)$ and size $C(K,H)$ for each resulting net using $(K,H)$, and the best operational point $(K^*,H^*)$ for target loss values $L_{\text{max}} \in \{0.005,0.01,0.05,0.3\}$.

\begin{figure}[t!]
  \centering
  \psfrag{N}[][]{number of hidden units $H$}
  \psfrag{logk}[][]{$\log_2{K}$}
  \psfrag{loss1}[l][l]{{\small $L(K,H) \leq 0.3$}}
  \psfrag{loss2}[l][l]{{\small $L(K,H) \leq 0.05$}}
  \psfrag{loss3}[l][l]{{\small $L(K,H) \leq 0.01$}}
  \psfrag{loss4}[l][l]{{\small $L(K,H) \leq 0.005$}}
  \psfrag{inf}[t][t]{\caja[0.8]{t}{c}{$\infty$ \\ {\tiny (reference)}}}
  \begin{tabular}{@{}c@{\hspace{0.0\linewidth}}c@{\hspace{0.0\linewidth}}c@{}}
    optimal loss $L(K,H)$ & \caja{c}{c}{model size $C(K,H)$ \\ in megabytes ($2^{23}$ bits)} & $L(K,H)$ and contours of $C(K,H)$ \\[2ex]
    \includegraphics*[height=0.41\linewidth,bb=33 0 379 428,clip]{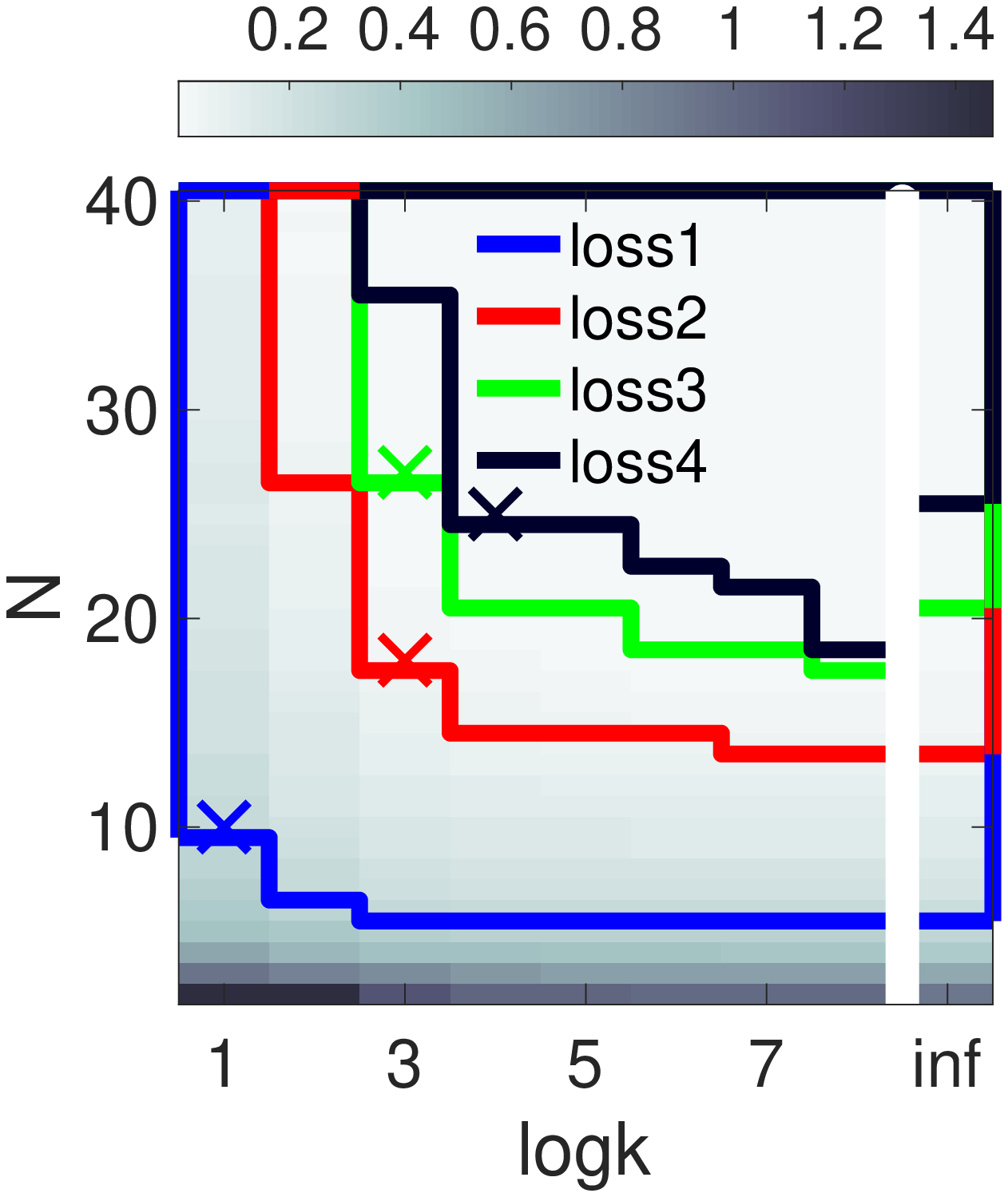} &
    \psfrag{N}{}
    \includegraphics*[height=0.41\linewidth,bb=33 0 379 428,clip]{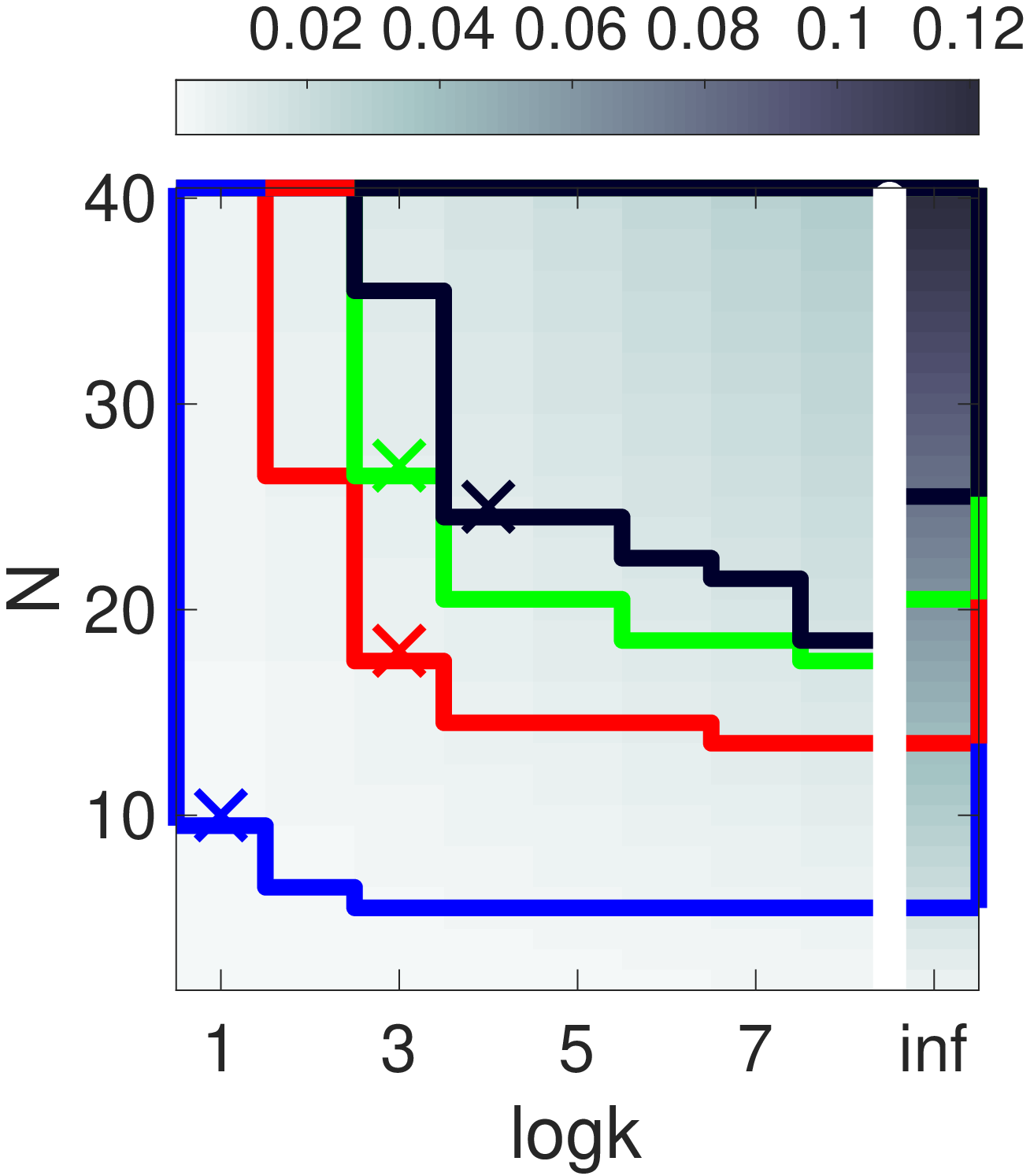} &
    \psfrag{N}{}
    \includegraphics*[height=0.41\linewidth,bb=33 0 379 428,clip]{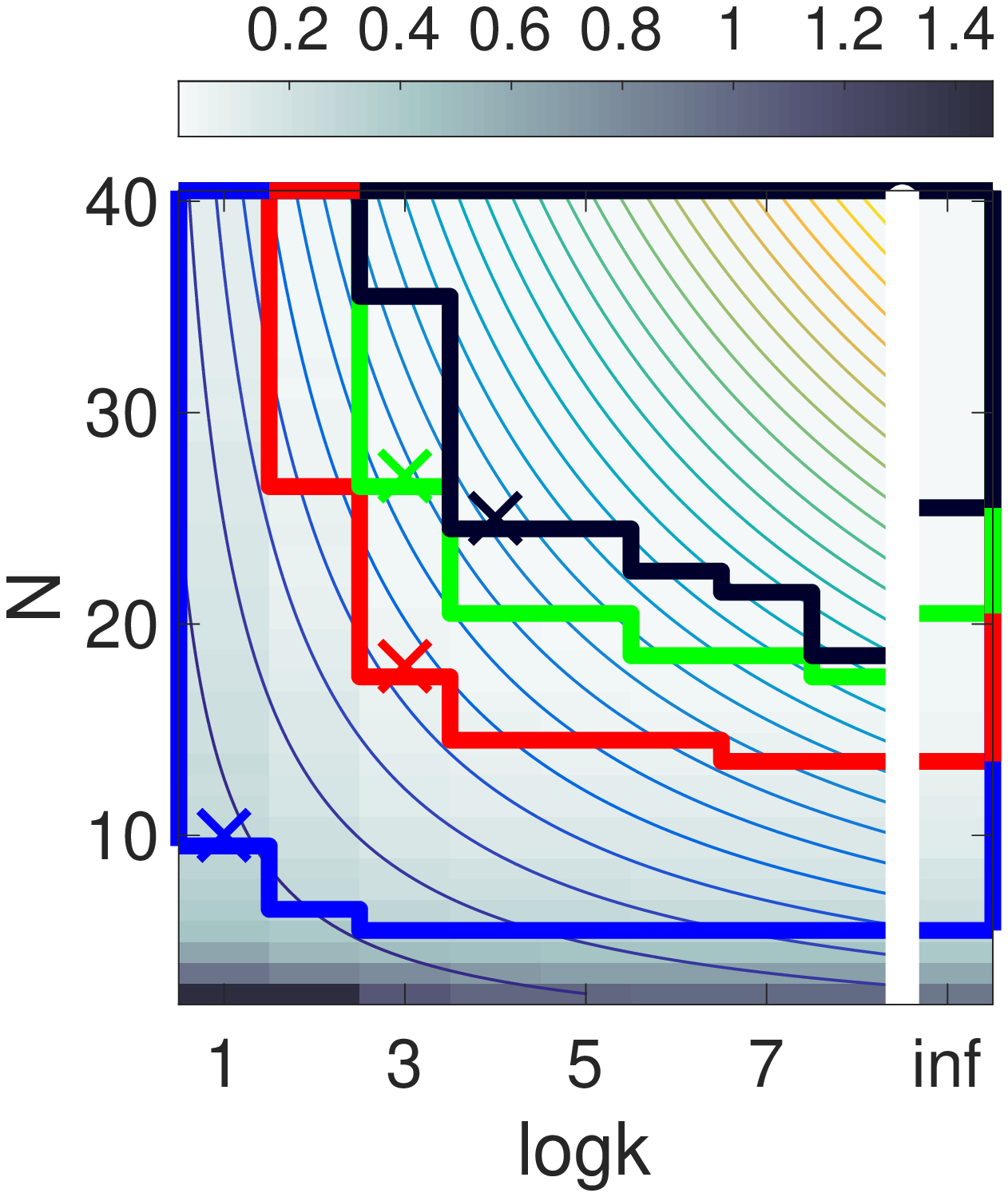}
  \end{tabular}
  \caption{Relation between loss, model complexity and quantization level in a single-layer neural net for MNIST classification. We train nets of different complexity (number of hidden units $H \in \{2,\dots,40\}$) and compress them with our LC algorithm using different codebook sizes ($\log_2{K} \in \{1,\dots,8\}$, with ``$\infty$'' meaning no compression, i.e., the reference net). \emph{Left plot}: the resulting loss $L(K,H)$. We show in color four level sets $L_{\text{max}} \in \{0.005,0.01,0.05,0.3\}$ of $L$, i.e., the points $(K,H)$ satisfying $L(K,H) \le L_{\text{max}}$. \emph{Middle plot}: the net size $C(K,H)$, with the same level sets over $L$. The color markers $\times$ identify the best operational point $(K^*,H^*)$ within each level set, i.e., the point $(K,H)$ having smallest size $C(K,H)$ such that $L(K,H) \le L_{\text{max}}$. \emph{Right plot}: like the left plot but with the contours of $C(K,H)$ superimposed.}
  \label{f:L-K-H}
\end{figure}

Within a given level set, the points with large $H$ and $\log_2{K} = 1$ (top left in the plot) correspond to the regime ``train a large reference net and compress it maximally''; the points with small $H$ and $\log_2{K} = \infty$ (bottom right in the plot) correspond to the regime ``train a small reference net and do not compress it''; and intermediate points $(K,H)$ correspond to intermediate model sizes and moderate compression levels. As the plot shows, if the target loss is large (i.e., we do not require high classification accuracy) then maximal compression is optimal; but as the target loss increases (i.e., we require more accurate models), then the optimal point $(K^*,H^*)$ moves towards intermediate compression levels. If we require no loss degradation whatsoever, this might be only achievable without compression. Therefore, in general it is not clear what the optimal regime will be, and solving this model selection problem in practice will involve some trial and error of model sizes and compression levels. However, it will be often be the case that significant compression is achievable if we can tolerate a minor loss degradation. Compression also simplifies model selection: the neural net designer can simply overestimate the size of the net required to achieve a target loss, and let the compression find a smaller net with a similar loss. So it seems clear that \emph{a good approximate strategy is to take a large enough model and compress it as much as possible}.

\subsection{Quantizing linear regression, with a non-gaussian weight distribution}
\label{s:expts:regression}

\begin{figure}[p]
  \centering
  \psfrag{ref}[l][l]{{\scriptsize reference}}
  \psfrag{idc}[l][l]{iDC}
  \psfrag{lcc}[l][l]{LC}
  \psfrag{iterations}{}
  \psfrag{counts}[t][]{weight density}
  \psfrag{weights}{}
  \psfrag{t0}[l][l]{0}
  \psfrag{t1}[l][l]{1}
  \psfrag{t3}[l][l]{30}
  \psfrag{kmeans}[t][]{\# $k$-means its.}
  \begin{tabular}{@{}c@{}c@{}c@{}}
    \psfrag{error}[b][b]{\caja{b}{c}{$K = 4$ \\ training loss $L$}}
    \includegraphics*[height=0.27\linewidth]{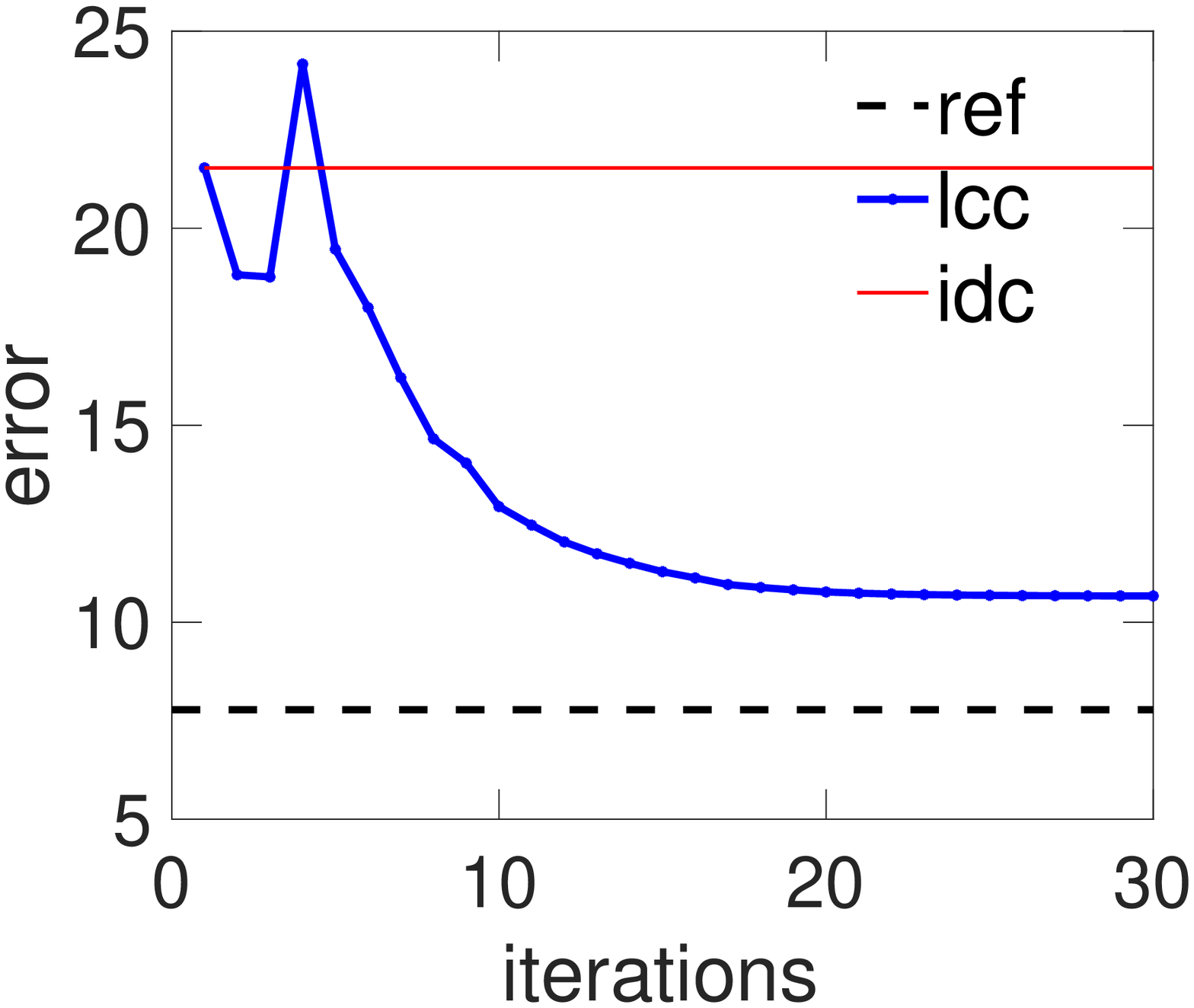} &
    \includegraphics*[height=0.27\linewidth]{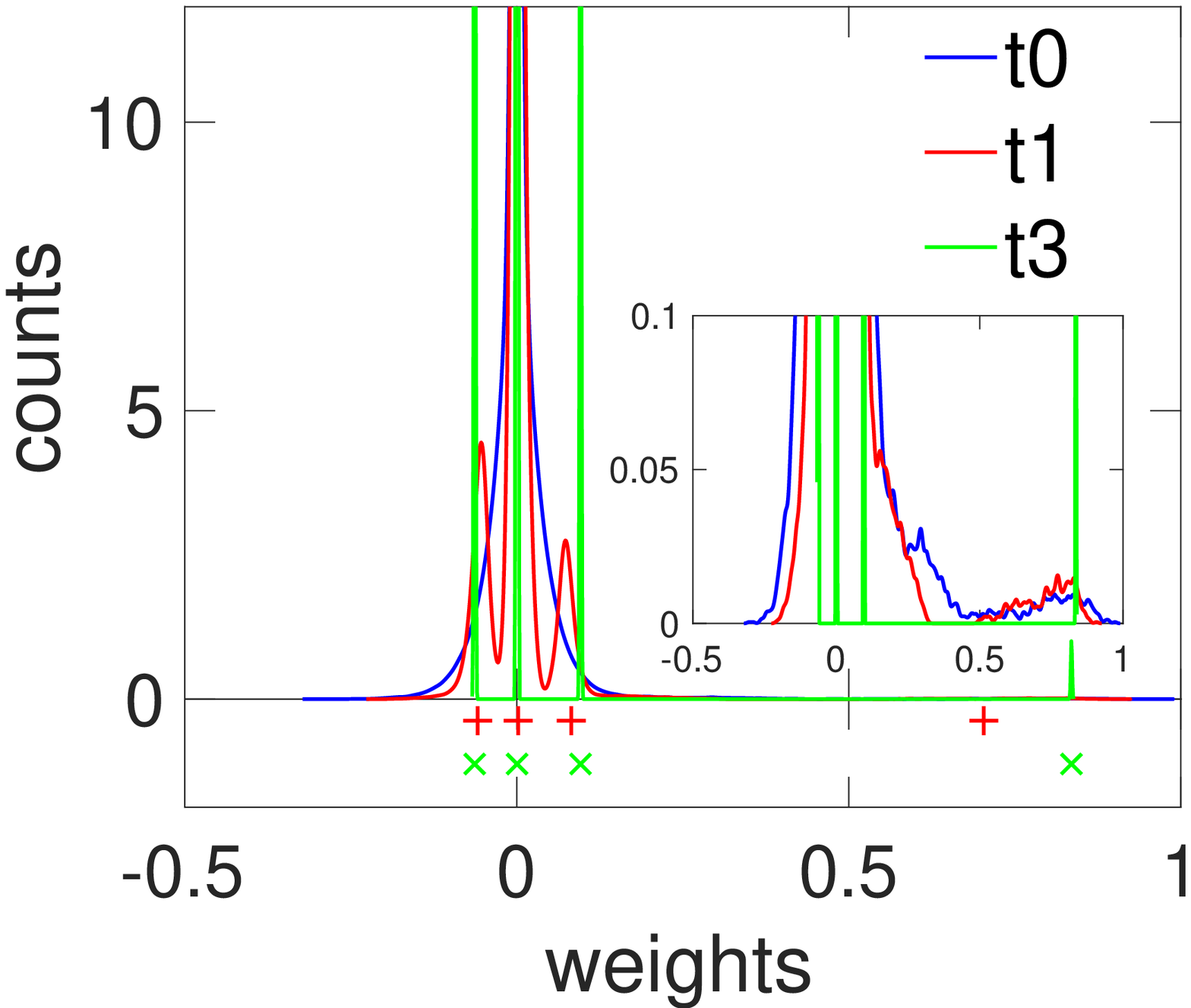} &
    \includegraphics*[height=0.27\linewidth]{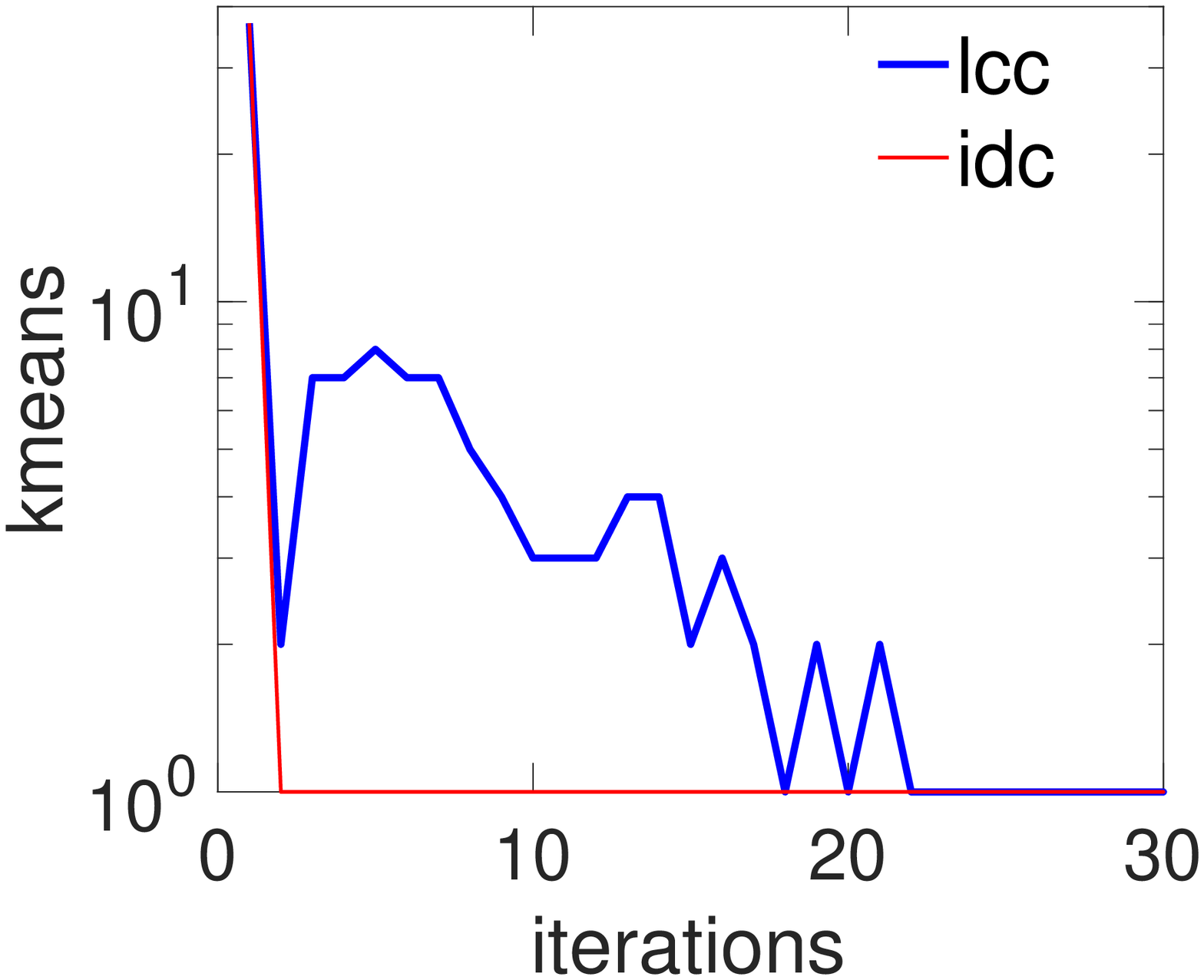} \\
    \psfrag{error}[b][b]{\caja{b}{c}{$K = 2$ \\ training loss $L$}}
    \psfrag{iterations}[][]{LC iterations}
    \includegraphics*[height=0.27\linewidth]{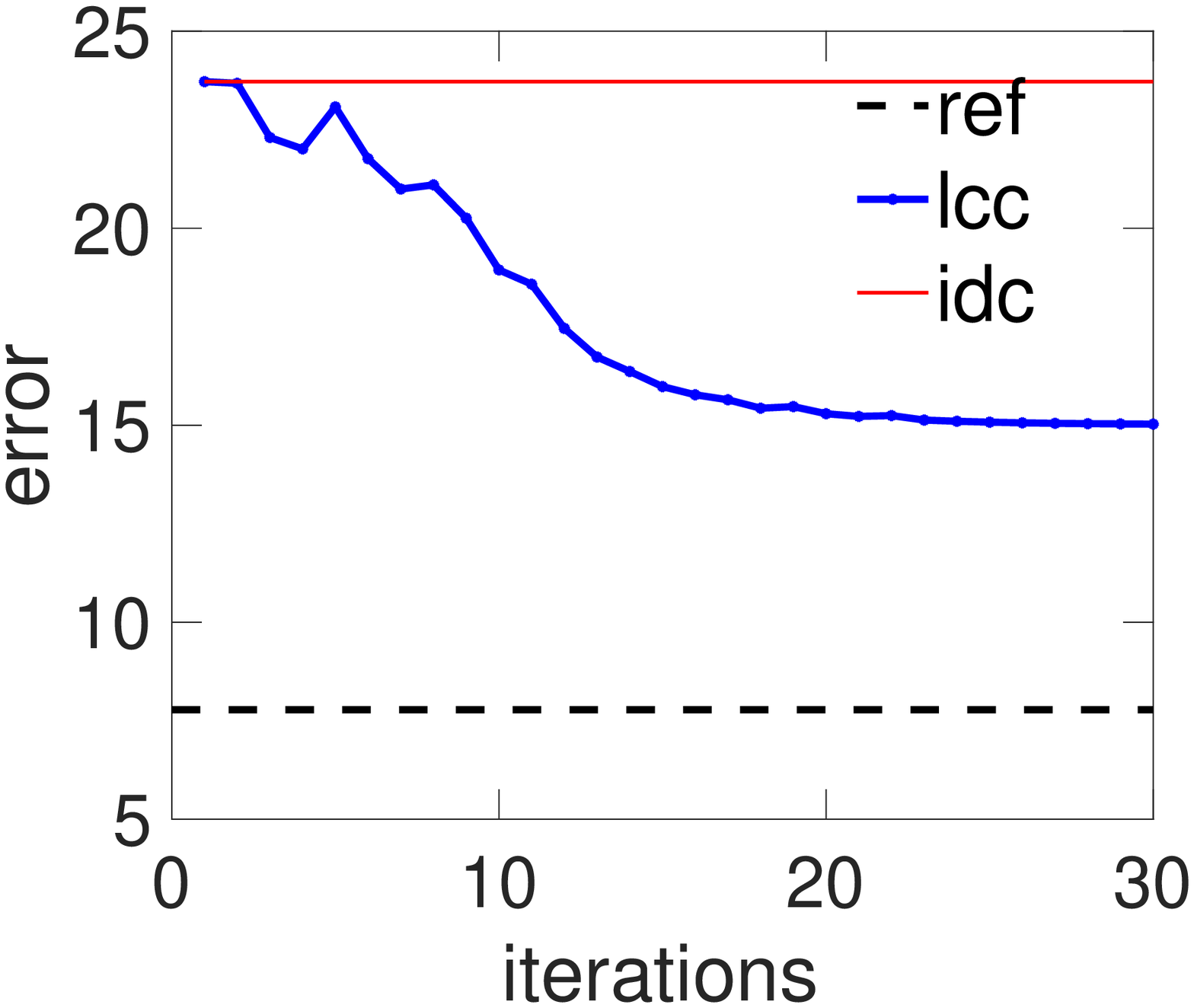} &
    \psfrag{weights}[][B]{weights $w_i$}
    \includegraphics*[height=0.27\linewidth]{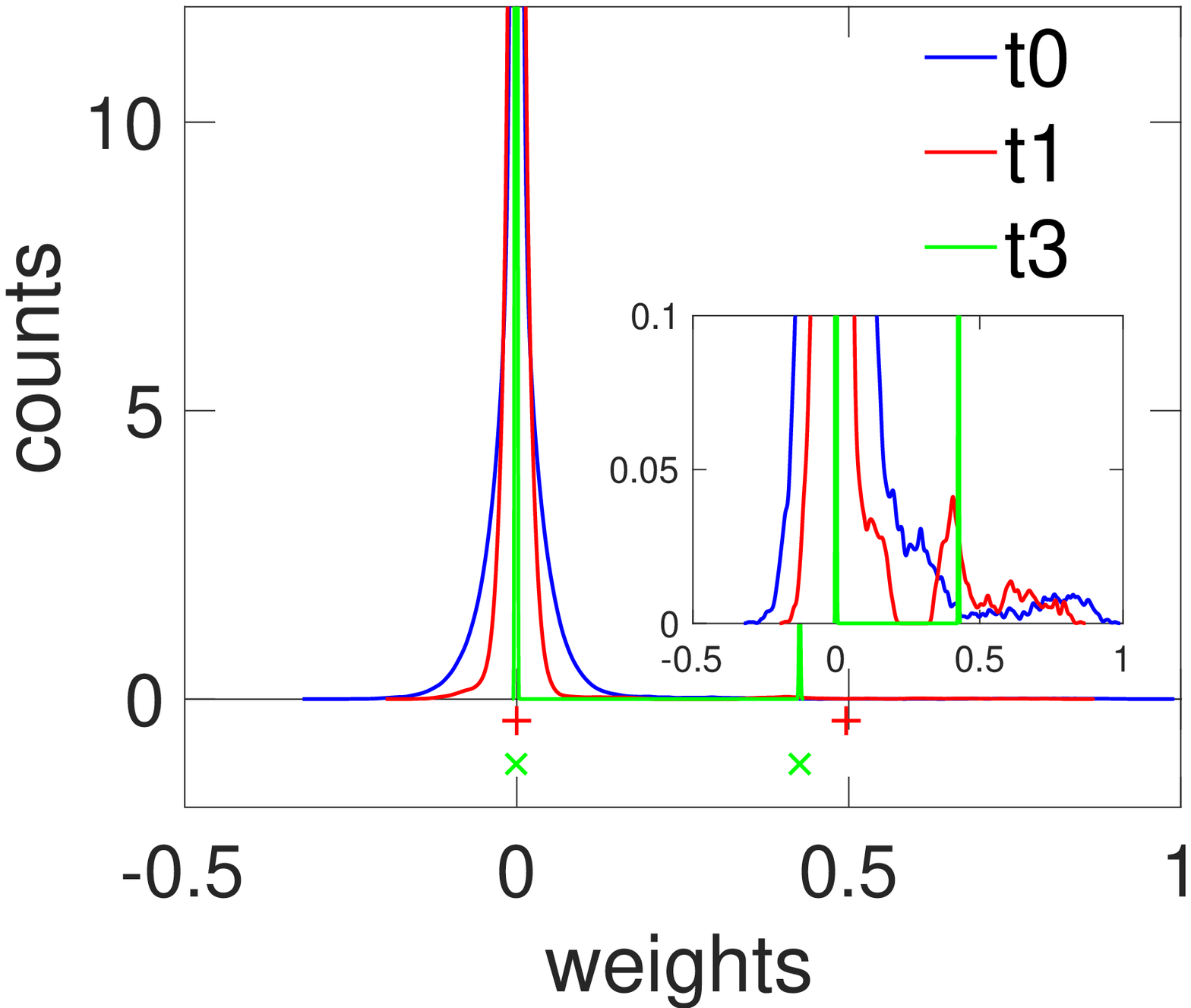} &
    \psfrag{iterations}[][]{LC iterations}
    \includegraphics*[height=0.27\linewidth]{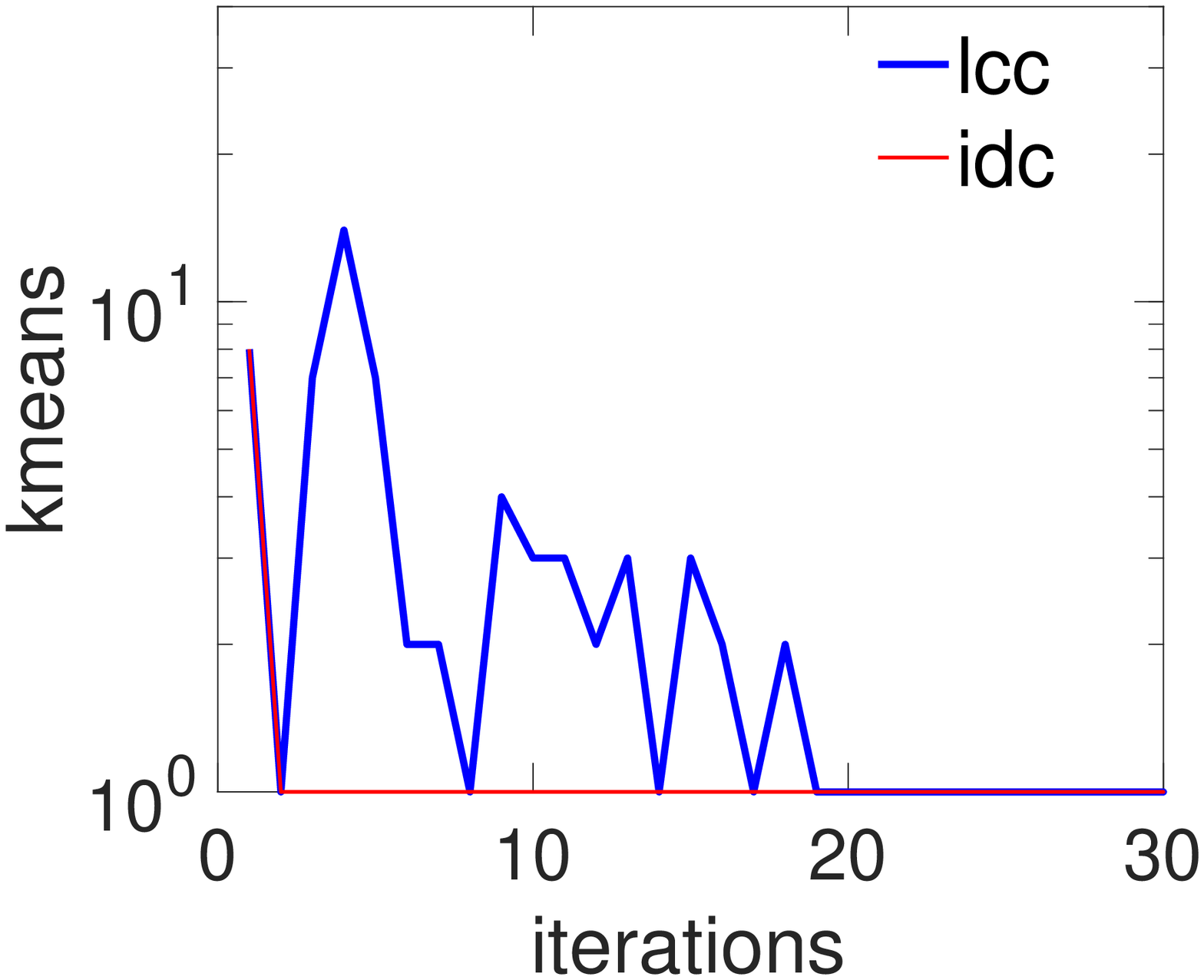}
  \end{tabular} \\[5ex]
  \begin{tabular}[c]{@{}c@{\hspace{3ex}}c@{}}
    \begin{tabular}[c]{@{}l@{}}
      $\displaystyle L(\W,\b) = \frac{1}{N} \sum^N_{n=1}{ \norm{\y_n - \W\x_n - \b}^2 }$ \\[3ex]
    \caja{c}{l}{$\y_n$: MNIST image ($28 \times 28$) \\ $\x_n$: MNIST image ($14 \times 14$), resized}
    \end{tabular} &
    \small
    \begin{tabular}[c]{@{}ccrrrr@{}}
      \toprule
      codebook & compression & \multicolumn{4}{c@{}}{\makebox[15ex]{\dotfill}loss $L$\makebox[15ex]{\dotfill}} \\
      size $K$ & ratio $\rho$ & \multicolumn{1}{c}{DC} & \multicolumn{1}{c}{iDC} & \multicolumn{1}{c}{LC} & \multicolumn{1}{c@{}}{reference} \\
      \midrule
      $4$ (2 bits/weight) & $\approx\times 16$ & 21.531 & 21.531 & 10.666 & 7.779 \\
      $2$ (1 bit/weight) & $\approx\times 32$ & 23.721 & 23.721 & 15.026 & 7.779 \\
      \bottomrule
    \end{tabular}
  \end{tabular}
  \caption{Regression problem using a codebook of size $K = 4$ (row 1) and $K = 2$ (row 2), and training loss for each method (table). \emph{Column 1}: loss over iterations. \emph{Column 2}: weight distribution of the reference model (iteration 0, before quantization), direct compression DC (iteration 1), and the LC algorithm (iteration 30), using a kernel density estimate with manually selected bandwidth. The inset enlarges vertically the distributions to show the small cluster structure. The locations of the codebook centroids are shown below the distributions as markers: $+$ are the centroids fitted to the reference model and $\times$ the centroids at the end of the LC algorithm. \emph{Column 3}: number of $k$-means iterations in each C step of the LC algorithm.}
  \label{f:regression}
\end{figure}

This experiment has two goals: 1) to verify in a controlled setting without local optima and with exact L and C steps that DC and iDC are identical to each other and significantly worse than LC. 2) To test our LC algorithm with a weight distribution that is far from Gaussian (unlike the weight distributions that typically arise with deep nets, which look zero-mean Gaussian). The problem is a simulated ``super-resolution'' task, where we want to recover a high-resolution image from a low-resolution one, by training a linear regression on pairs $(\x_n,\y_n) =$ (low-res, high-res), i.e., the loss is $L(\W,\b) = \smash{\frac{1}{N} \sum^N_{n=1}{ \norm{\y_n - \W\x_n - \b}^2 }}$, with weights \W\ and biases \b. We construct each low-resolution image \x\ by bicubic interpolation (using Matlab) of the high-resolution image \y. Ignoring border effects and slight nonlinearities, this means that each pixel (component) of \x\ is approximately a linear combination with constant coefficients of its corresponding pixels in \y. Hence, we can write the mapping from high to low resolution approximately as a linear mapping $\y = \A\x + \aa$, where $\aa = \0$ and the $i$th row of \A\ contains a few nonzero weights (the coefficients of the linear combination for pixel $x_i$). The ground-truth recovery matrix that optimizes the loss is then $\W = \A^+$, and it has a similar structure: roughly, each row contains only a few nonzeros, whose values are about the same across rows. We also add Gaussian noise when generating each low-resolution image $\x_n$, which spreads the optimal weights $w_{ij}$ around the ideal values above and also spreads the biases around zero. In summary, this means that the reference model weights have a clustered distribution, with a large cluster around zero, and a few small clusters at positive values.

To construct the dataset, we randomly selected $N=1\,000$ MNIST images $\y_n$ of $28 \times 28$, resized them as above to $14\times 14$ and added Gaussian noise to generate the $\x_n$, so that \W\ is of $784 \times 196$ ($P_1 =$ 153\,664 weights) and \b\ of $784 \times 1$ ($P_0 =$ 784). We compress \W\ using a codebook of size $K=2$ (1 bit per weight value, $\rho \approx \times$32) or $K=4$ (2 bits per weight value, $\rho \approx \times$16). The reference model and the L step have a single closed-form solution given by a linear system. For the LC algorithm, we increase $\mu_k = a^k \mu_0$ with $\mu_0 = 10$ and $a = 1.1$ for 30 iterations.

Fig.~\ref{f:regression} shows the results for codebook sizes $K=2$ and $4$. Firstly, as expected, DC and iDC do not change past the very first iteration, while LC achieves a much lower loss. The reference weight distribution (blue curve) shows a large cluster at zero and small clusters around 0.25 and 0.75 (see the inset). The small clusters correspond to the (inverse) bicubic interpolation coefficients, and it is crucial to preserve them in order to achieve a low loss. The LC algorithm indeed does a good job at this: with $K=2$ it places one centroid at zero and the other around $0.4$ (striking a balance between the small clusters); with $K=4$ it places one centroid at zero, two near the small clusters, and a fourth around $-0.1$. Note that the location of these centroids does not correspond to the reference model (the reference model quantization are the red $+$ markers), because the LC algorithm optimizes the centroids and the weights to achieve the lowest loss.
% Matlab "simulation" of rescaling using linear interpolation for an "image"
% of 1x(n*K) to 1xK:
% K=10; n=3; a=[ones(1,n)/n zeros(1,n*(K-1))]; AA=gallery('circul',a);
% A=AA(1:n:end,:), pinv(A)

\subsection{Quantizing LeNet neural nets for classification on MNIST}
\label{s:expts:LeNet}

We randomly split the MNIST training dataset (60k grayscale images of $28 \times 28$, 10 digit classes) into training (90\%) and test (10\%) sets. We normalize the pixel grayscales to [0,1] and then subtract the mean. We compress all layers of the network but each layer has its own codebook of size $K$. The loss is the average cross-entropy. To train a good reference net, we use Nesterov's accelerated gradient method \citep{Nester83a} with momentum 0.9 for 100k minibatches, with a carefully fine-tuned learning rate $0.02 \cdot 0.99^j$ running 2k iterations for every $j$ (each a minibatch of 512 points). The $j$th L step parameters are given below for each net. For LC we also update $\mu$ and $\blambda$ at the end of the C step. Our LC algorithm uses a multiplicative schedule for the penalty parameter $\mu_j = \mu_0 a^j$ with $\mu_0 = 9.76 \cdot 10^{-5}$ and $a = 1.1$, for $0 \le j \le 30$. The batch size is 512 points for all methods.

We use the following neural nets, whose structure is given in table~\ref{t:LeNet}:
\begin{description}
\item[LeNet300] This is a 3-layer densely connected feedforward net \citep{LeCun_98a} with tanh activations, having 266\,610 learnable parameters ($P_1 =$ 266\,200 weights and $P_0 =$ 410 biases). The $j$th L step (for LC and for iDC) runs 2k SGD iterations with momentum 0.95 and learning rate $0.1 \cdot 0.99^j$. We also trained a BinaryConnect net using the code of \citet{Courbar_15a} with deterministic rounding (without batch normalization), with $\alpha = 0.001$ and $\beta = 0.98$ after every 2k minibatch iterations, initialized from the reference net and trained for 120k minibatch iterations.
\item[LeNet5] This is a variation of the original LeNet5 convolutional net described in \citet{LeCun_98a}. It is included in Caffe%
\footnote{\url{https://github.com/BVLC/caffe/blob/master/examples/mnist/lenet.prototxt}}
\citep{Jia_14a} and was used by \citet{Han_15a}. It has ReLU activations \citep{NairHinton10a}, dropout \citep{Srivas_14a} with $p = 0.5$ on the densely connected layers, and softmax outputs, total 431\,080 trainable parameters ($P_1 =$ 430\,500 weights and $P_0 =$ 580 biases). The $j$th L step (for LC and for iDC) runs 4k SGD iterations with momentum 0.95 and learning rate $\alpha \cdot 0.99^j$, where $\alpha=0.02$ for codebook sizes $K = 2,4,8$ and $\alpha=0.01$ for $K = 16,32,64$. This is because $\alpha = 0.02$ lead to divergence on iDC, even though LC was still able to converge.
\end{description}

\begin{table}[t!]
  \centering
  \begin{tabular}[t]{@{}ll@{}}
    \multicolumn{2}{c}{LeNet300} \\
    \toprule
    Layer & Connectivity \\
    \midrule
    Input & $28 \times 28$ image \\
    1 & \caja{t}{l}{fully connected, 300 neurons, \\ followed by tanh} \\
    2 & \caja{t}{l}{fully connected, 100 neurons, \\ followed by tanh} \\
    \caja{t}{l}{3 \\ (output)} & \caja{t}{l}{fully connected, 10 neurons, \\ followed by softmax} \\
    \midrule
    \multicolumn{2}{c}{$P_1 =$ 266\,200 weights, $P_0 =$ 410 biases} \\
    \bottomrule
  \end{tabular}
  \hspace{5ex}
  \begin{tabular}[t]{@{}ll@{}}
    \multicolumn{2}{c}{LeNet5} \\
    \toprule
    Layer & Connectivity \\
    \midrule
    Input & $28 \times 28$ image \\
    1 & \caja{t}{l}{convolutional, 20 $5 \times 5$ filters (stride=1), \\ total 11\,520 neurons, followed by ReLU} \\
    2 & \caja{t}{l}{max pool, $2 \times 2$ window (stride=2), \\ total 2\,280 neurons} \\
    3 & \caja{t}{l}{convolutional, 50 $5 \times 5$ filters (stride=1), \\ total 3\,200 neurons, followed by ReLU} \\
    4 & \caja{t}{l}{max pool, $2 \times 2$ window (stride=2), \\ total 800 neurons} \\
    5 & \caja{t}{l}{fully connected, 500 neurons and dropout \\ with $p=0.5$, followed by ReLU} \\
    \caja{t}{l}{6 \\ (output)} & \caja{t}{l}{fully connected, 10 neurons and dropout \\ with $p=0.5$, followed by softmax} \\
    \midrule
    \multicolumn{2}{c}{$P_1 =$ 430500 weights, $P_0 =$ 580 biases} \\
    \bottomrule
  \end{tabular}
  \caption{Structure of the LeNet300 and LeNet5 neural nets trained on the MNIST dataset.}
  \label{t:LeNet}
\end{table}

\begin{figure}[p]
  \centering
  \psfrag{ref}[l][l]{ref}
  \psfrag{idc}[l][l]{iDC}
  \psfrag{lcc}[l][l]{LC}
  \psfrag{x1}[l][l]{$\times$30}
  \psfrag{x2}[l][l]{$\times$15}
  \psfrag{x3}[l][l]{$\times$10}
  \psfrag{x4}[l][l]{$\times$7.9}
  \psfrag{x5}[l][l]{$\times$6.3}
  \psfrag{x6}[l][l]{$\times$5.2}
  \psfrag{iterations}{}
  \begin{tabular}{@{}c@{\hspace{1ex}}c@{}c@{}}
    & LeNet300 & LeNet5 \\
    \rotatebox{90}{\hspace{18ex}training loss $L$} &
    \psfrag{time}[][]{time (s)}
    \includegraphics*[width=0.48\linewidth]{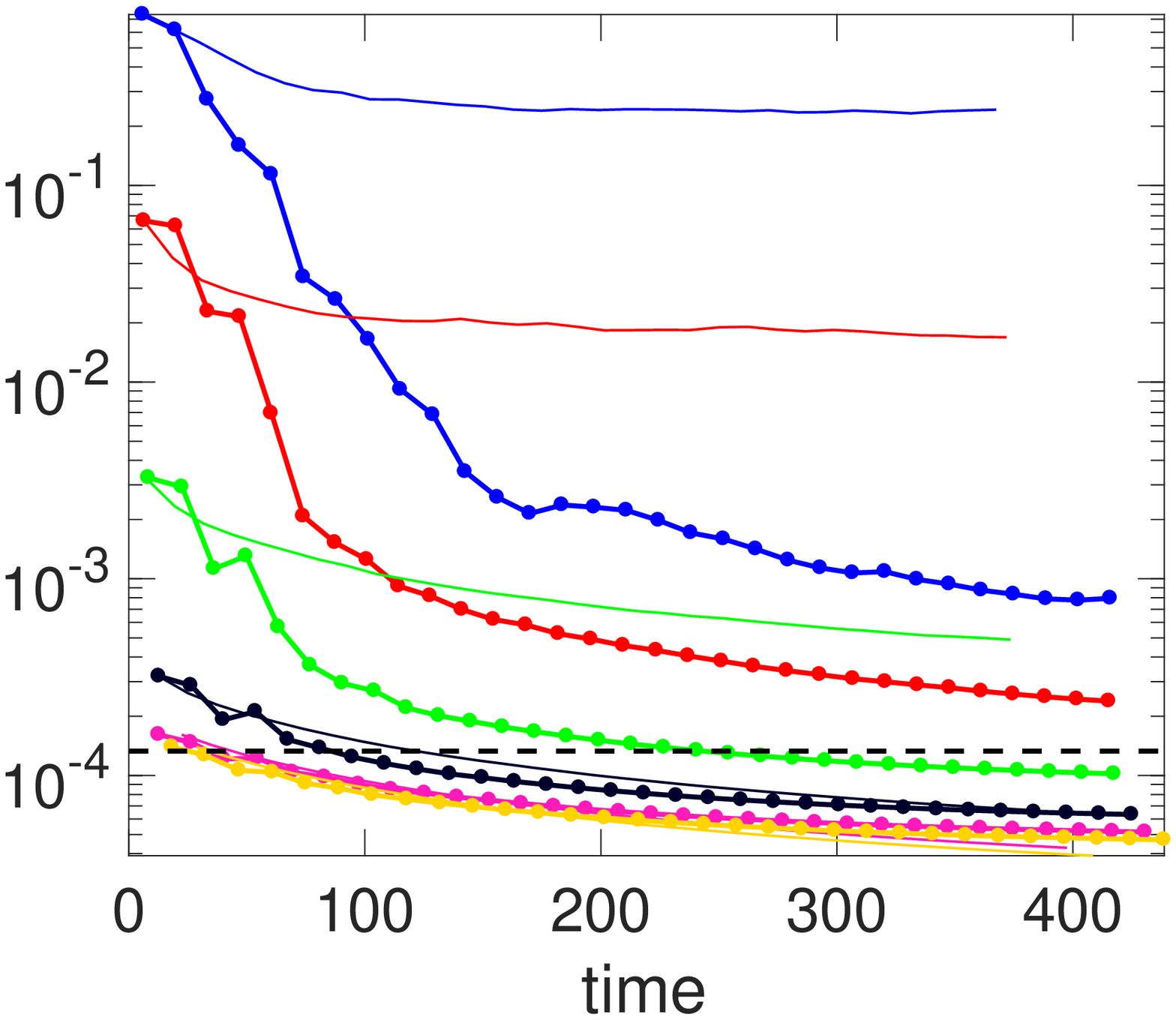} &
    \psfrag{time}[][]{time (s)}
    \includegraphics*[width=0.48\linewidth,bb=20 4 505 421]{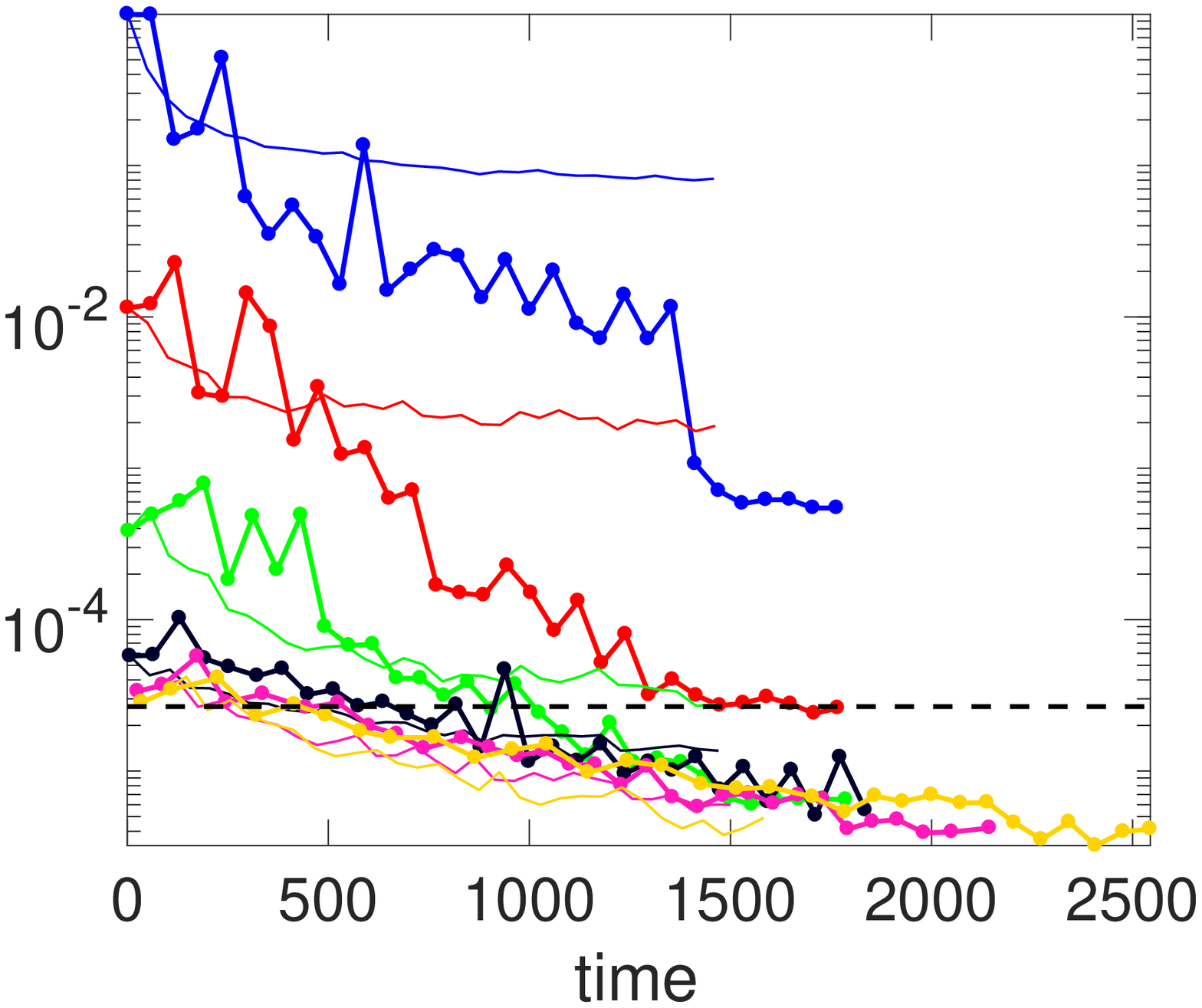} \\[2ex]
    \rotatebox{90}{\hspace{18ex}training loss $L$} &
    \includegraphics*[width=0.48\linewidth]{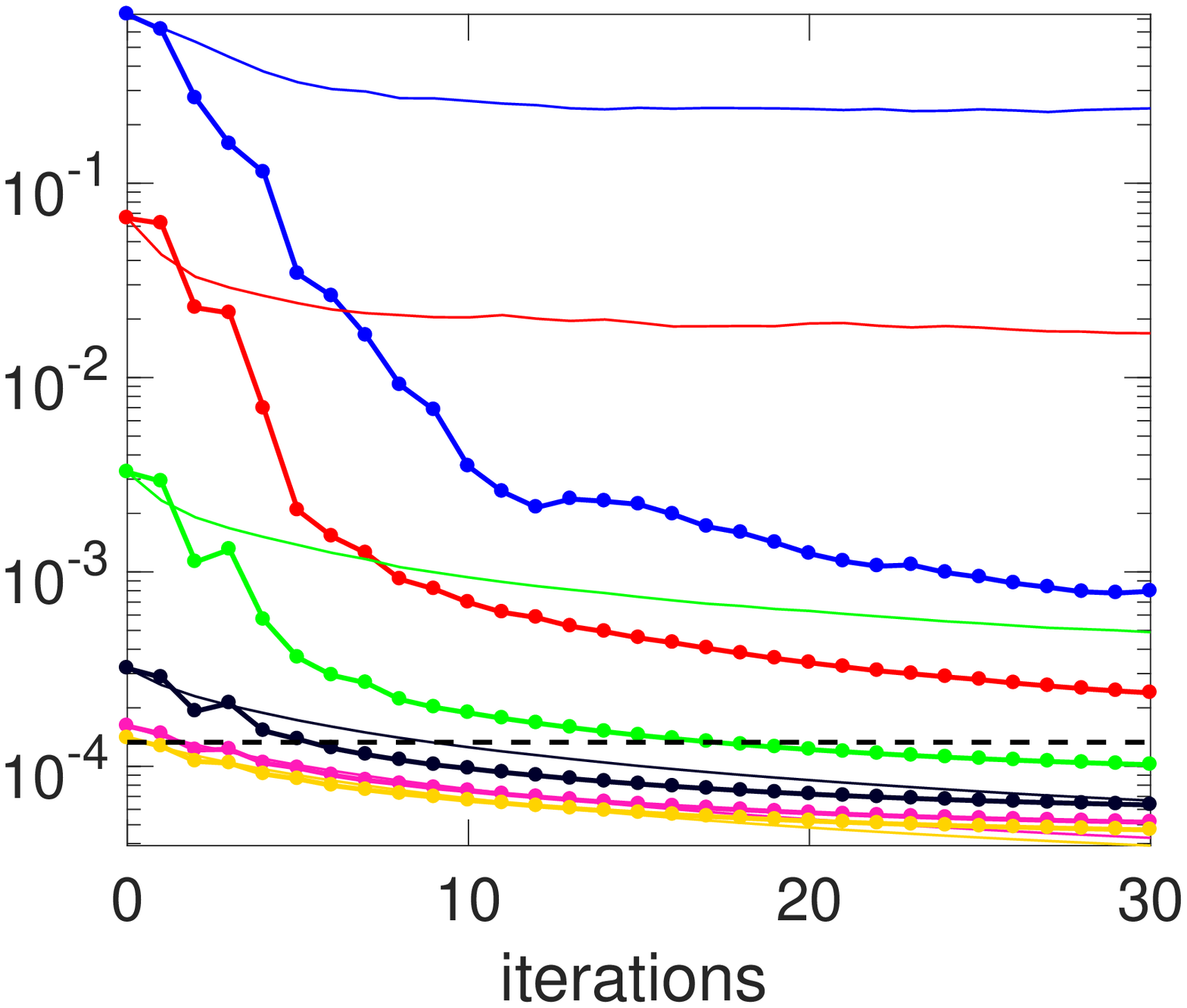} &
    \includegraphics*[width=0.48\linewidth]{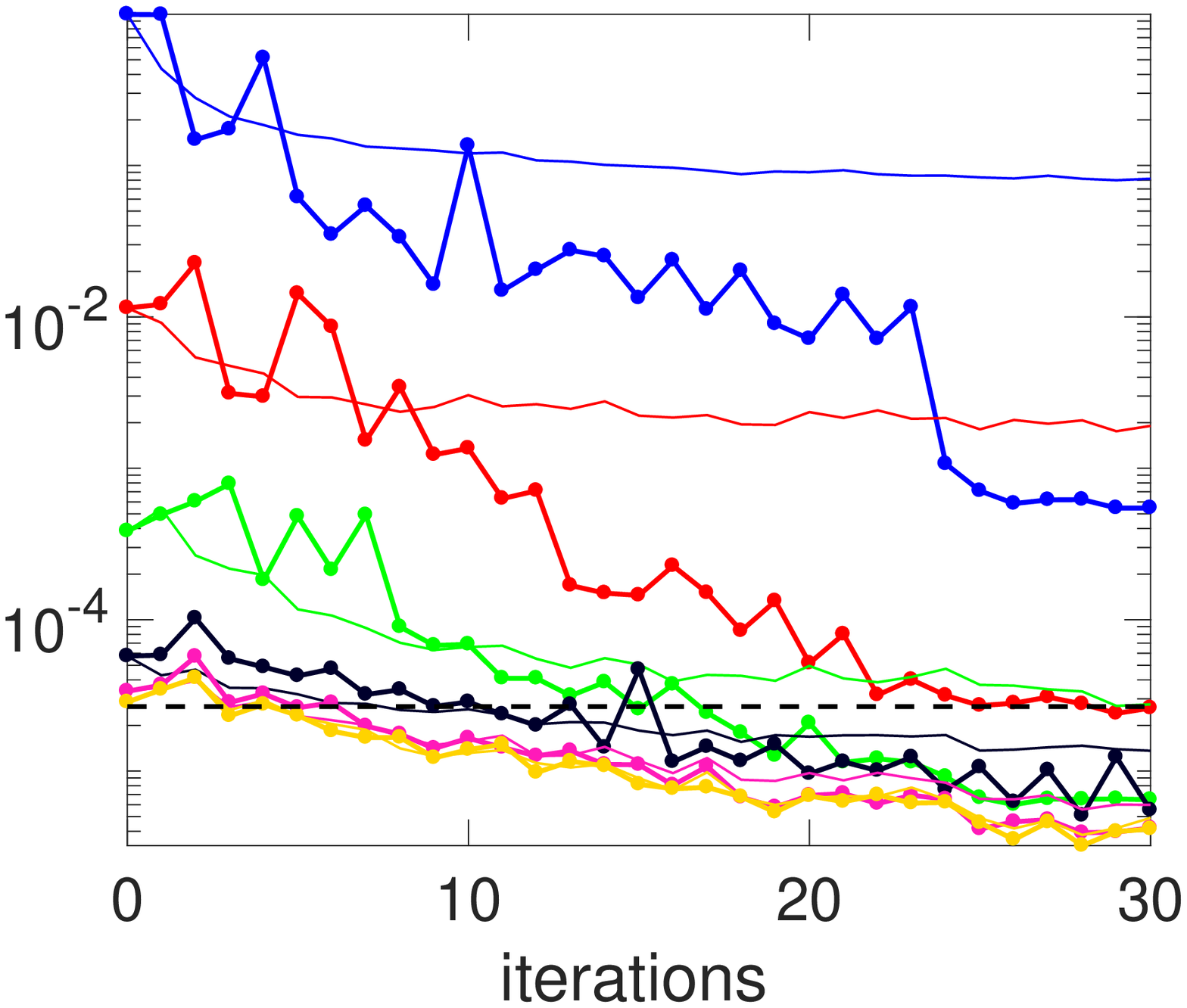} \\[-3ex]
    \rotatebox{90}{\hspace{18ex}test error $E_{\text{test}}$ (\%)} &
    \psfrag{iterations}[][B]{SGD iterations $\times$2k}
    \includegraphics*[width=0.48\linewidth]{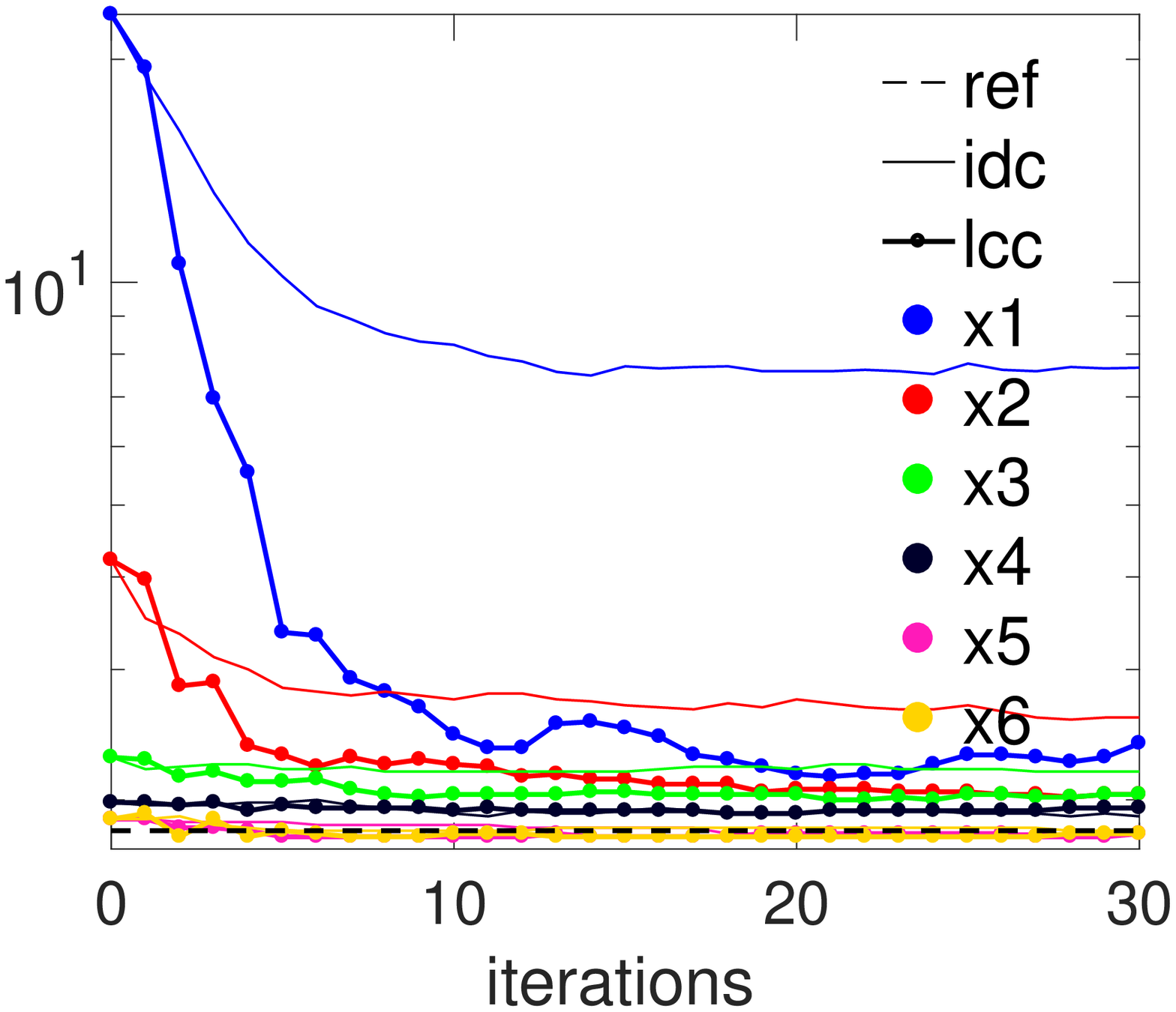} &
    \psfrag{iterations}[][B]{SGD iterations $\times$4k}
    \includegraphics*[width=0.48\linewidth]{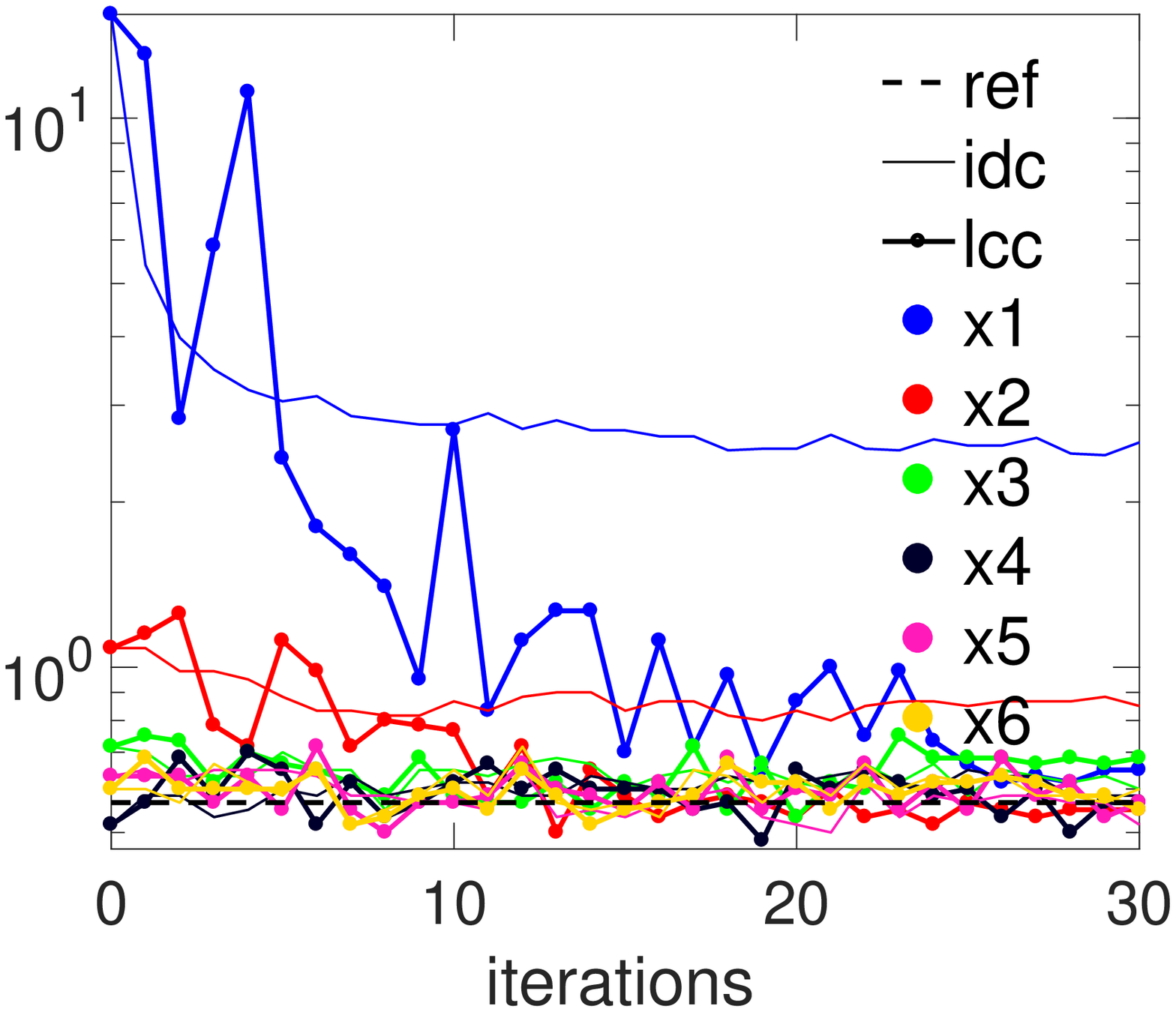}
  \end{tabular}
  \caption{Learning curves over runtime and over SGD iterations (each marker $\bullet$ indicates one LC iteration, containing 2\,000 or 4\,000 SGD iterations), with different compression ratios (codebook sizes $K$), on the LeNet neural nets. Reference: dashed black horizontal line, LC: thick lines with markers, iDC: thin lines.}
  \label{f:learning-curves}
\end{figure}

\begin{figure}[p]
  \centering
  \begin{tabular}{@{}c|lr|rrr|crr|crr@{}}
    \toprule
    & & & \multicolumn{3}{c|}{LC algorithm} & \multicolumn{3}{c|}{direct compression, DC}& \multicolumn{3}{c}{iterated DC, iDC}  \\
    & \multicolumn{1}{c}{$\rho$} & \multicolumn{1}{c|}{$K$} & $\log{L}$ & $E_{\text{train}}$  & $E_{\text{test}}$ & $\log{L}$ &  $E_{\text{train}}$ & $E_{\text{test}}$ & $\log{L}$ &  $E_{\text{train}}$ & $E_{\text{test}}$ \\  
    \midrule
    & \multicolumn{2}{c|}{reference} & -3.87 & 0 & 2.28 & & & & & & \\
    & $\times5.3$  & 64 & -4.33 & 0     & 2.25 & -3.85 & 0 & 2.28  & -4.41 & 0     & 2.25  \\
    & $\times6.3$  & 32 & -4.29 & 0     & 2.25 & -3.79 & 0 & 2.24  & -4.37 & 0     & 2.24 \\
    \raisebox{0pt}[0pt][0pt]{\rotatebox{90}{\makebox[0pt][c]{LeNet300}}}
    & $\times7.9$  & 16 & -4.20 & 0     & 2.25 & -3.50 & 0 & 2.26  & -4.18 & 0     & 2.20 \\
    & $\times10.5$ & 8  & -3.99 & 0     & 2.29 & -2.48 & 0.07 & 2.46  & -3.31 & 0.004 & 2.34 \\
    & $\times15.6$ & 4  & -3.62 & 0     & 2.44 & -1.18 & 2.21 & 4.58  & -1.77 & 0.543 & 3.23 \\
    & $\times30.5$ & 2  & -3.10 & 0.009 & 2.42 & -0.13 & 23.02 & 23.68 & -0.61 & 5.993 & 7.98 \\
    \midrule
    & \multicolumn{2}{c|}{reference} & -4.58 & 0 & 0.54 & & & & & & \\
    & $\times5.3$  & 64 & -5.38 & 0     & 0.47 & -4.54 & 0 & 0.52  & -5.31 & 0     & 0.52 \\
    & $\times6.3$  & 32 & -5.38 & 0     & 0.48 & -4.47 & 0 & 0.49  & -5.22 & 0     & 0.49 \\
    \raisebox{0pt}[0pt][0pt]{\rotatebox{90}{\makebox[0pt][c]{LeNet5}}}
    & $\times7.9$  & 16 & -5.26 & 0     & 0.54 & -4.24 & 0 & 0.49  & -4.87 & 0     & 0.49 \\
    & $\times10.5$ & 8  & -5.19 & 0     & 0.45 & -3.42 & 0 & 0.58  & -4.56 & 0     & 0.54 \\
    & $\times15.7$ & 4  & -4.58 & 0     & 0.53 & -1.94 & 0.29 & 0.94  & -2.45 & 0.05  & 0.66 \\
    & $\times30.7$ & 2  & -3.26 & 0.006 & 0.57 & -0.00 & 15.77 & 15.62 & -1.09 & 1.92  & 2.56 \\
    \bottomrule
  \end{tabular} \\[5ex]
  \psfrag{dc}[l][l]{DC}
  \psfrag{idc}[l][l]{iDC}
  \psfrag{lc}[l][l]{LC}
  \psfrag{lenet300}[l][l]{LeNet300}
  \psfrag{lenet5}[l][l]{LeNet5}
  \psfrag{trainloss}[B][]{training loss $L$}
  \psfrag{valerr}[B][]{test error $E_{\text{test}}$ (\%)}
  \psfrag{codebook}[][]{$\xleftarrow{\hspace{5ex}}$ codebook size $K$ $\xrightarrow{\hspace{5ex}}$}
  \begin{tabular}{@{}c@{\hspace{0.05\linewidth}}c@{}}
    \hspace{5ex}$\xleftarrow{\hspace{5ex}}$ compression ratio $\rho$ $\xrightarrow{\hspace{5ex}}$ & \hspace{5ex}$\xleftarrow{\hspace{5ex}}$ compression ratio $\rho$ $\xrightarrow{\hspace{5ex}}$ \\
    \includegraphics*[width=0.475\linewidth]{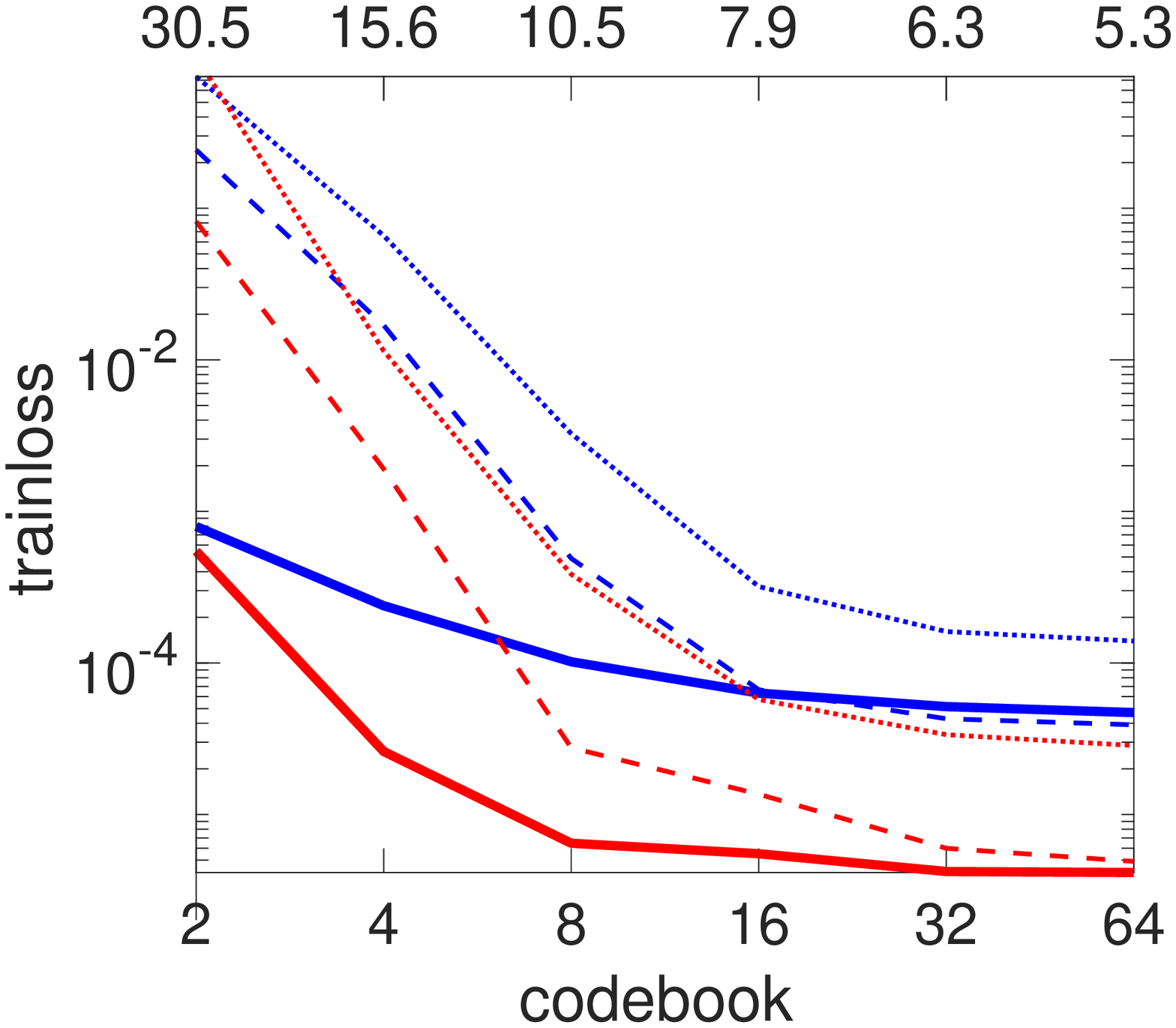} &
    \includegraphics*[width=0.475\linewidth]{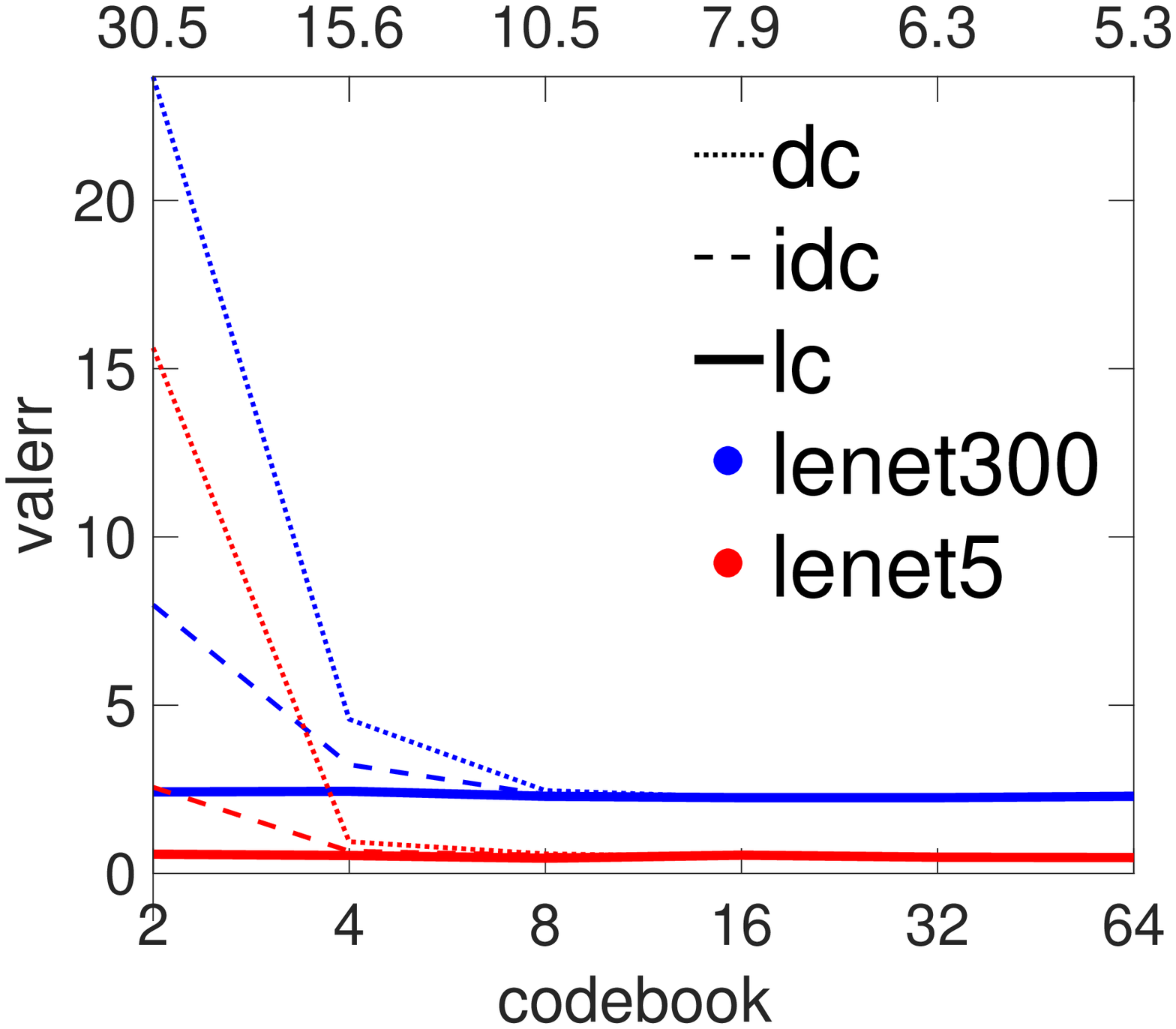}
  \end{tabular}
  \caption{Compression results for the LeNet neural nets using different algorithms, for different codebook sizes $K$ and corresponding compression ratio $\rho$, in two forms: tabular (top) and graph (bottom). We report the training loss $\log_{10}{L}$ and training and test classification error $E_{\text{train}}$ and $E_{\text{test}}$ (\%). The curves show the tradeoff between error vs compression ratio. Each LeNet neural net is shown in a different color, and each algorithm is shown in a different line type (LC: thick solid, iDC: thin dashed, DC: thin dotted).}
  \label{f:tradeoff}
\end{figure}

\begin{figure}[t!]
  \centering
  \psfrag{layer1}[l][l]{Layer 1}
  \psfrag{layer2}[l][l]{Layer 2}
  \psfrag{layer3}[l][l]{Layer 3}
  \psfrag{iterations}[][]{SGD iterations $\times$2k}
  \begin{tabular}{@{}c@{\hspace{0.05\linewidth}}c@{}}
    learning-compression (LC) algorithm & iterated direct compression (iDC) \\
    \psfrag{kmeans}[B][]{number of $k$-means iterations}
    \includegraphics*[width=0.475\linewidth]{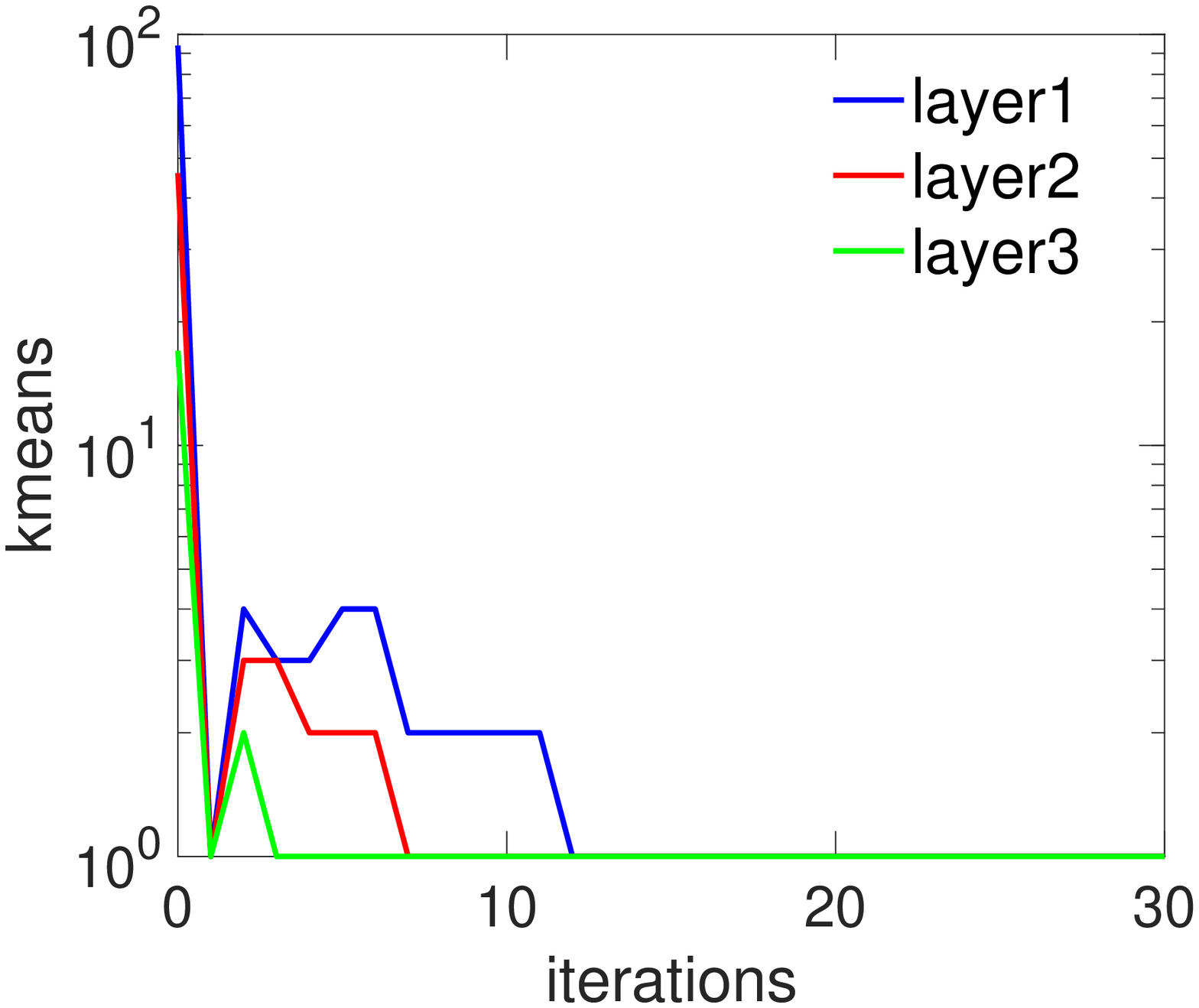} &
    \psfrag{kmeans}[B][]{number of $k$-means iterations}
    \includegraphics*[width=0.475\linewidth]{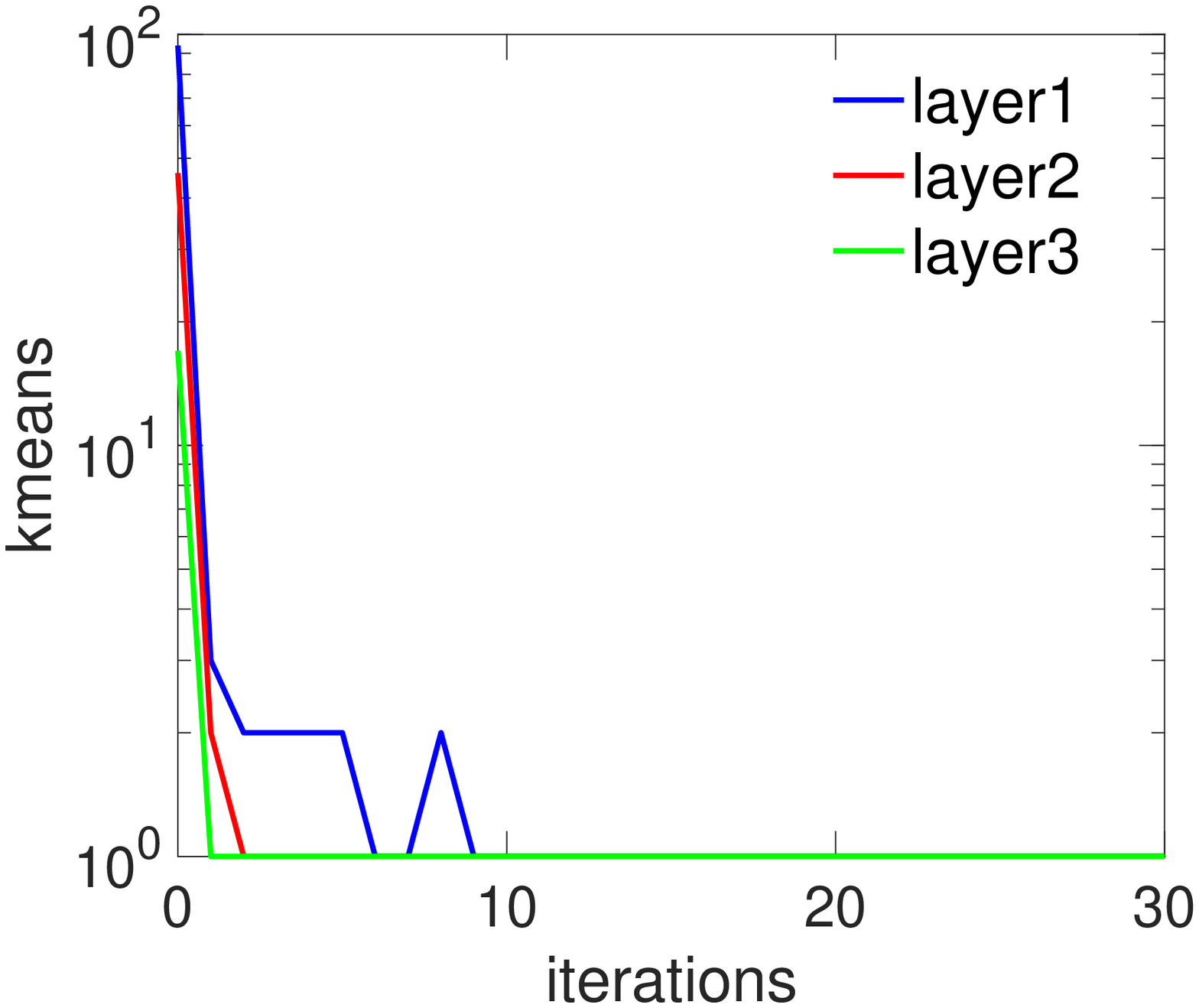}
  \end{tabular}
  \caption{Number of iterations ran within $k$-means over training for LeNet300 with a $K=4$ codebook.}
  \label{f:kmeans-its}
\end{figure}

The following figures and tables show the results. Fig.~\ref{f:learning-curves} shows the learning curves and fig.~\ref{f:tradeoff} the error vs compression tradeoff. The runtime for iDC and LC is essentially the same, as is their loss and error at low compression levels. But for high compression LC is distinctly superior. When using $K=2$ (1 bit per weight), LC also outperforms BinaryConnect, as shown in table~\ref{t:binary}. Note the two weight values found by LC considerably differ from $\pm 1$ and depend on the layer, namely $\{0.089,-0.091\}$ in layer 1, $\{0.157,-0.155\}$ in layer 2 and $\{0.726,-0.787\}$ in layer 3. Indeed, forcing weights to be $\pm 1$ is more limiting, and does not have a practical advantage over using two arbitrary values in terms of storage or runtime when applying the compressed net to an input.

\begin{table}[b!]
  \centering
  \begin{tabular}[c]{@{}c@{\hspace{10ex}}c@{}}
    \begin{tabular}[c]{@{}lrrr@{}}
      \toprule
      method & \multicolumn{1}{c}{$\log{L}$} & \multicolumn{1}{c}{$E_{\text{train}}$} & \multicolumn{1}{c}{$E_{\text{test}}$} \\
      \midrule
      reference & -3.87 & 0\% & 2.28\% \\
      LC algorithm & -3.10 & 0.009\% & 2.42\% \\
      BinaryConnect & -2.33 & 0.14\%  & 3.76\% \\
      \bottomrule
    \end{tabular} &
    \begin{tabular}[c]{@{}cc@{}}
      \toprule
      layer & codebook values for LC \\
      \midrule
      1 & $\calC = \{0.089, -0.091\}$ \\
      2 & $\calC = \{0.157, -0.155\}$ \\
      3 & $\calC = \{0.726, -0.787\}$ \\
      \bottomrule
    \end{tabular}
  \end{tabular}
  \caption{Binarization results in LeNet300 using the LC algorithm and BinaryConnect. LC uses a codebook of size $K=2$ (1 bit per weight) and the resulting codebook values are shown on the right (the BinaryConnect values are always $\pm 1$). The compression ratio is $\rho \approx \times 30.5$ for all methods.}
  \label{t:binary}
\end{table}

Note how the LC training loss need not decrease monotonically (fig.~\ref{f:learning-curves}). This is to be expected: the augmented Lagrangian minimizes eq.~\eqref{f:tradeoff} for each $\mu$, not the actual loss, but it does approach a local optimum in the limit when $\mu \rightarrow \infty$. Also note how some compressed nets actually beat the reference. This is because the latter was close but not equal to a local optimum, due to the long training times required by SGD-type algorithms. Since the compressed nets keep training, they gain some accuracy over the reference.

Figs.~\ref{f:weights-centroids-LC}--\ref{f:weights-centroids-iDC} show the evolution of weights and codebook centroids for $K=4$ for iDC and LC, for LeNet300. While LC converges to a feasible local optimum (note the delta-like weight distribution centered at the centroids), iDC does not. As we argued earlier, this is likely because iDC oscillates in a region half way between the reference net and its direct compression. These oscillations are very noticeable in the weight trajectories (right plots) for iDC in layers 1--2: the weights do not change their centroid assignment but oscillate around the centroid's value. In contrast, in LC some weights do change their centroid assignment (this happens because the L step moves the weights jointly in $P$-dimensional space), and robustly converge to a centroid. The weight distribution of BinaryConnect (not shown) is also far from the values $\pm 1$ to which the weights should tend.

Fig.~\ref{f:centroids-K} shows the final locations of the centroids for iDC and LC, for codebook sizes $K = 2$ to $64$, for LeNet300. Although the general distribution of the centroids is similar for both methods, there are subtle differences (particularly for small $K$) which, as seen in fig.~\ref{f:tradeoff}, translate into a significantly lower error for LC. For large enough $K$, the centroid distribution resembles that of the reference net. This is to be expected, because, with enough centroids, the optimally quantized network will be very similar to the reference and all 3 algorithms (DC, iDC and LC) will give the same result, namely quantizing the reference weights directly (regardless of the loss). Hence the centroid distribution will be the optimal quantization of the reference's weight distribution. Since the latter is roughly Gaussian for the LeNet300 net, the centroids will reflect this, as seen in fig.~\ref{f:centroids-K} for the larger $K$. However, for compression purposes we are interested in the small-$K$ region, and here the final weight distribution can significantly differ from the reference. For all values of $K$, the distribution of the centroids for LC is usually symmetric around zero (but sometimes its mean is significantly different from zero, as in layer 3), and its spread increases from layer 1 to layer 2 to layer 3 in range but not in standard deviation. Although some of these observations may carry over to other types of neural nets, we emphasize that the reference weight distribution and hence the centroid distribution strongly depend on the problem (dataset, model and even particular local optimum found). Indeed, the clustered distribution for the regression problem of fig.~\ref{f:regression} was very different from Gaussian. Therefore, it seems risky to anticipate what the optimal codebook distribution may be. Finding a really accurate net must be done for each particular problem with a careful optimization using a good algorithm, and the same is true for finding a really well compressed net.

\begin{figure}[p]
  \centering
  \psfrag{t0}[l][l]{0}
  \psfrag{t1}[l][l]{1}
  \psfrag{t3}[l][l]{30}
  \psfrag{weights}{}
  \psfrag{iterations}{}
  \begin{tabular}{@{}c@{}c@{}}
    \psfrag{counts}[][]{density, layer 1}
    \includegraphics*[width=0.49\linewidth]{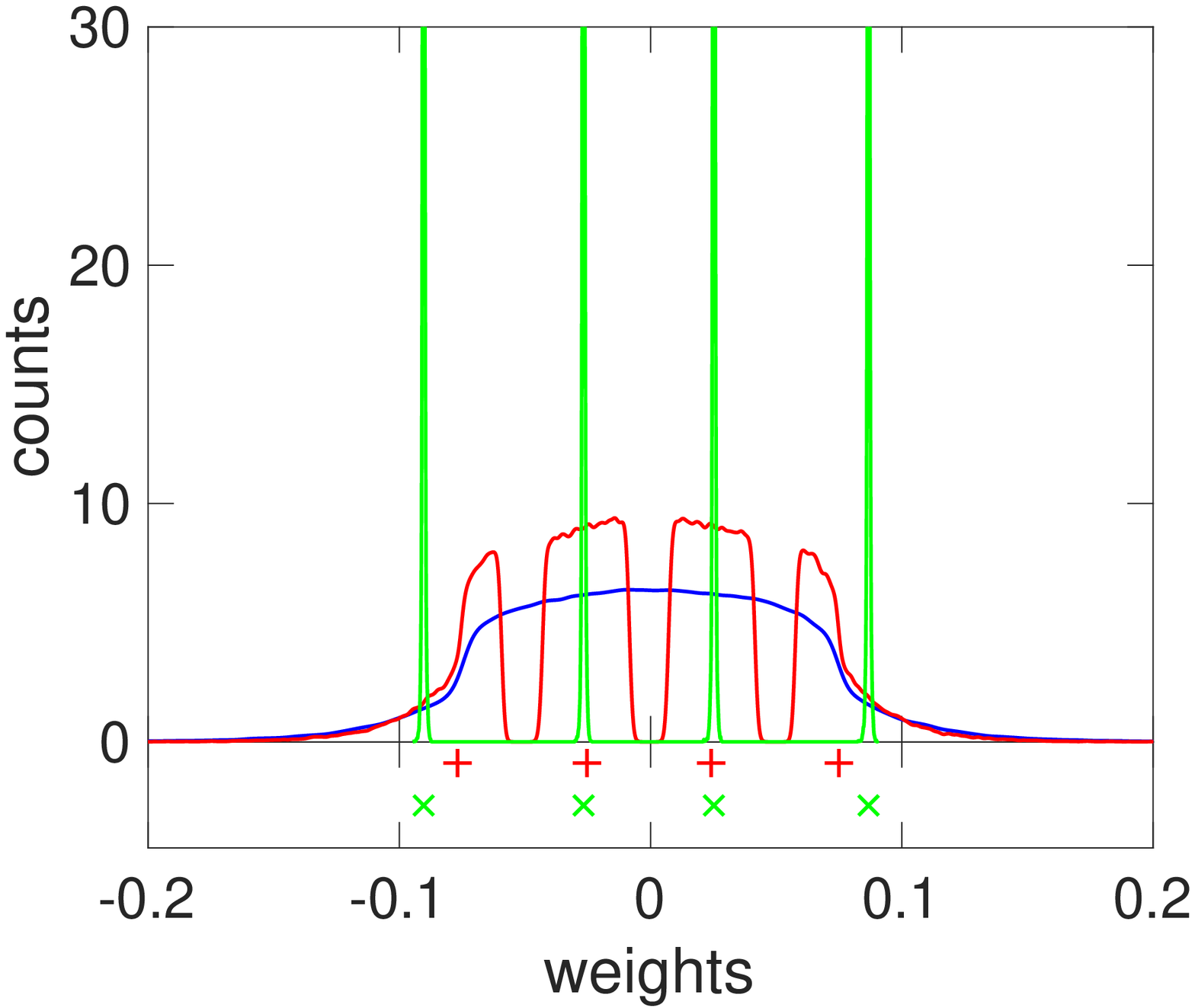} &
    \includegraphics*[width=0.49\linewidth]{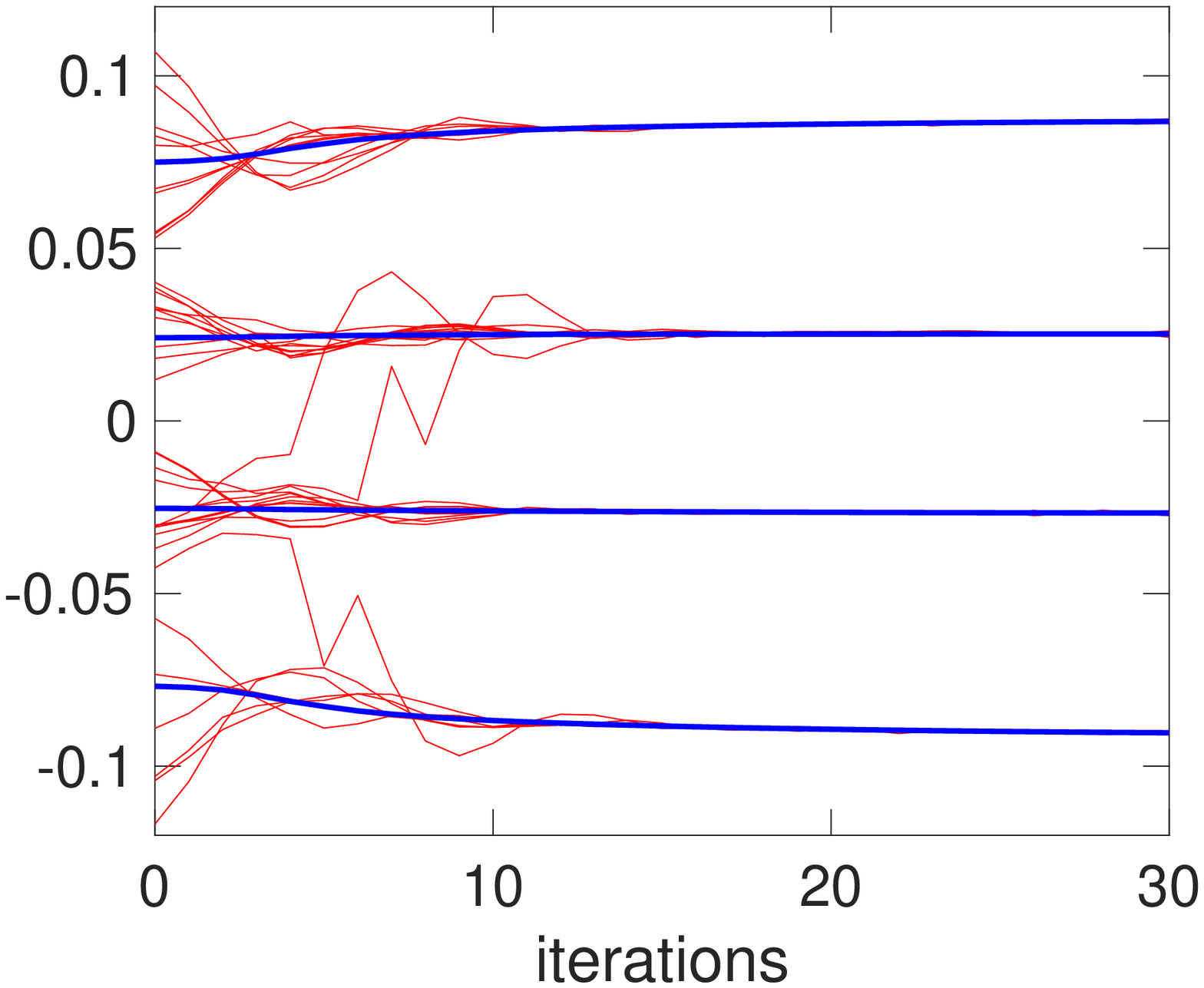} \\[-2ex]
    \psfrag{counts}[][]{density, layer 2}
    \includegraphics*[width=0.49\linewidth]{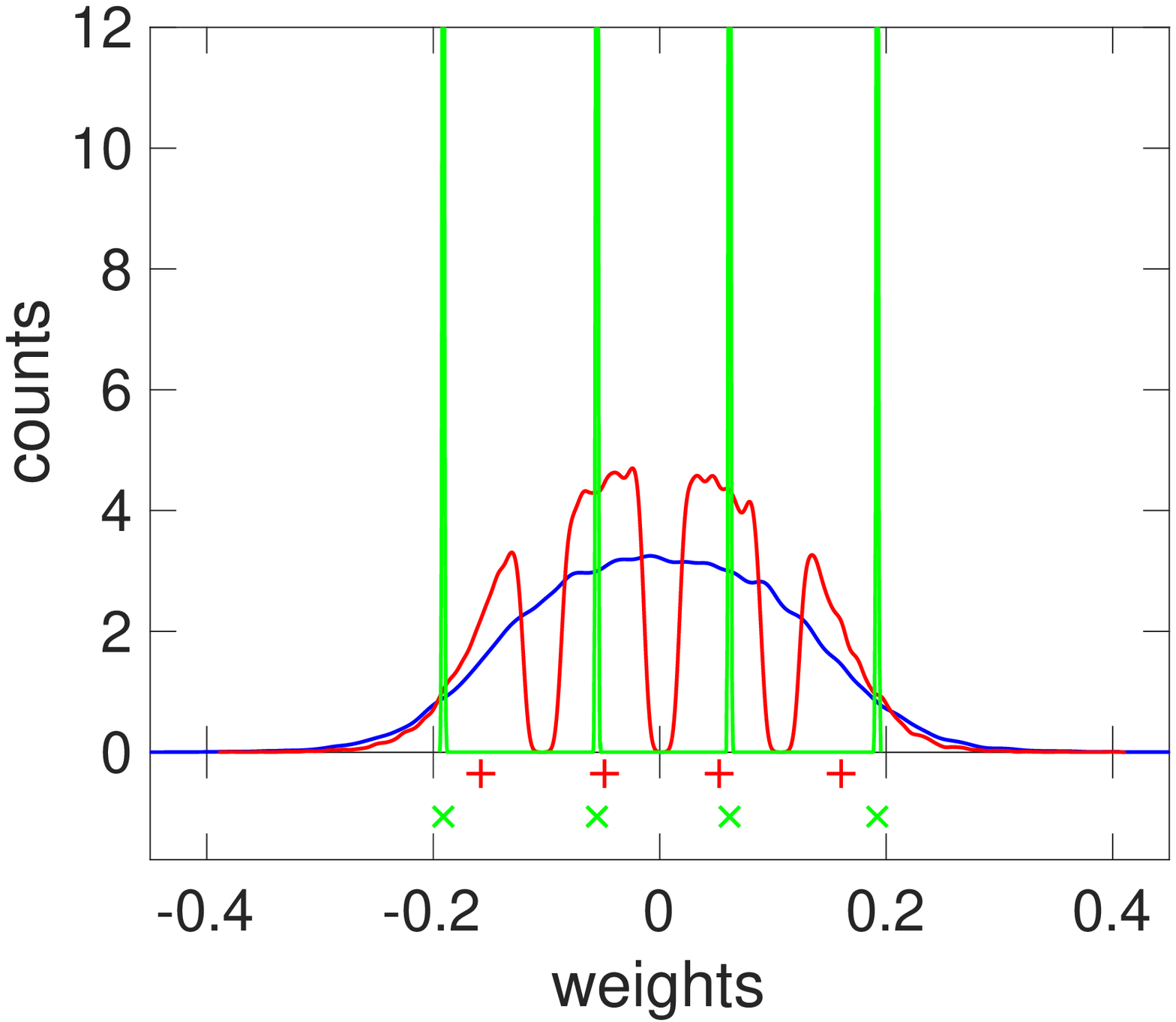} &
    \includegraphics*[width=0.49\linewidth]{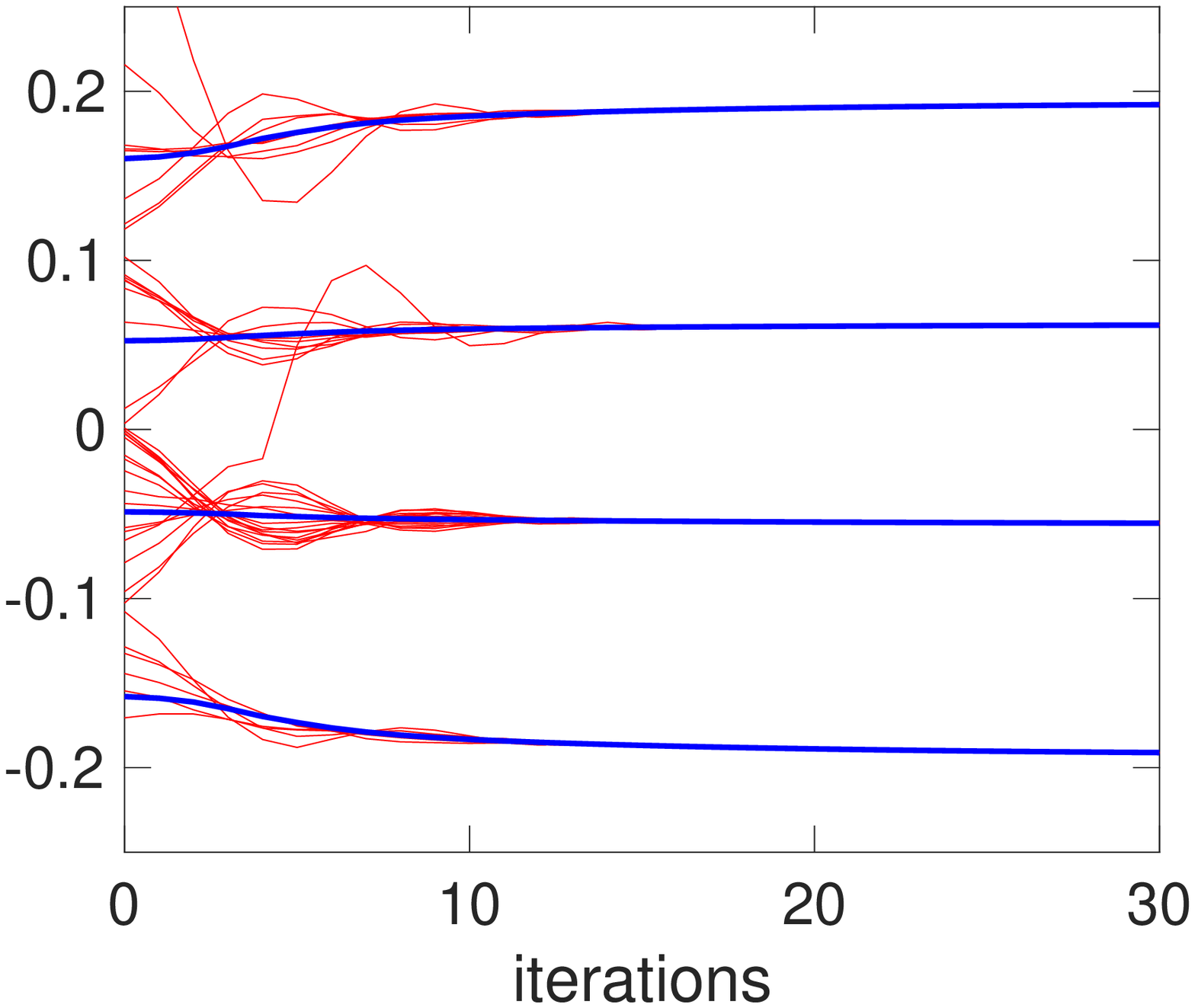} \\[-2ex]
    \psfrag{counts}[][]{density, layer 3}
    \psfrag{weights}[][]{weights $w_i$}
    \includegraphics*[width=0.49\linewidth]{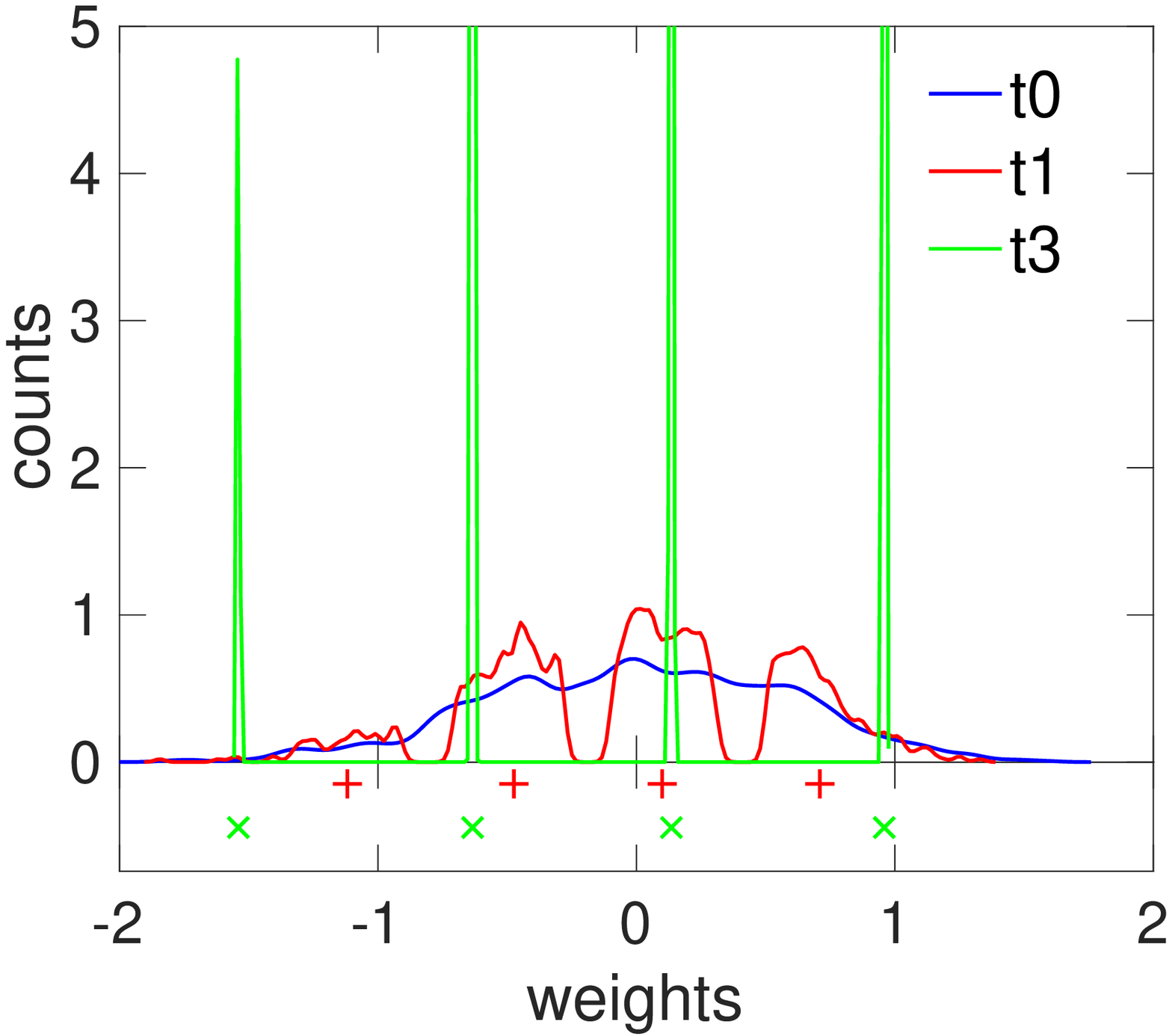} &
    \psfrag{iterations}[][]{SGD iterations $\times$2k}
    \includegraphics*[width=0.49\linewidth]{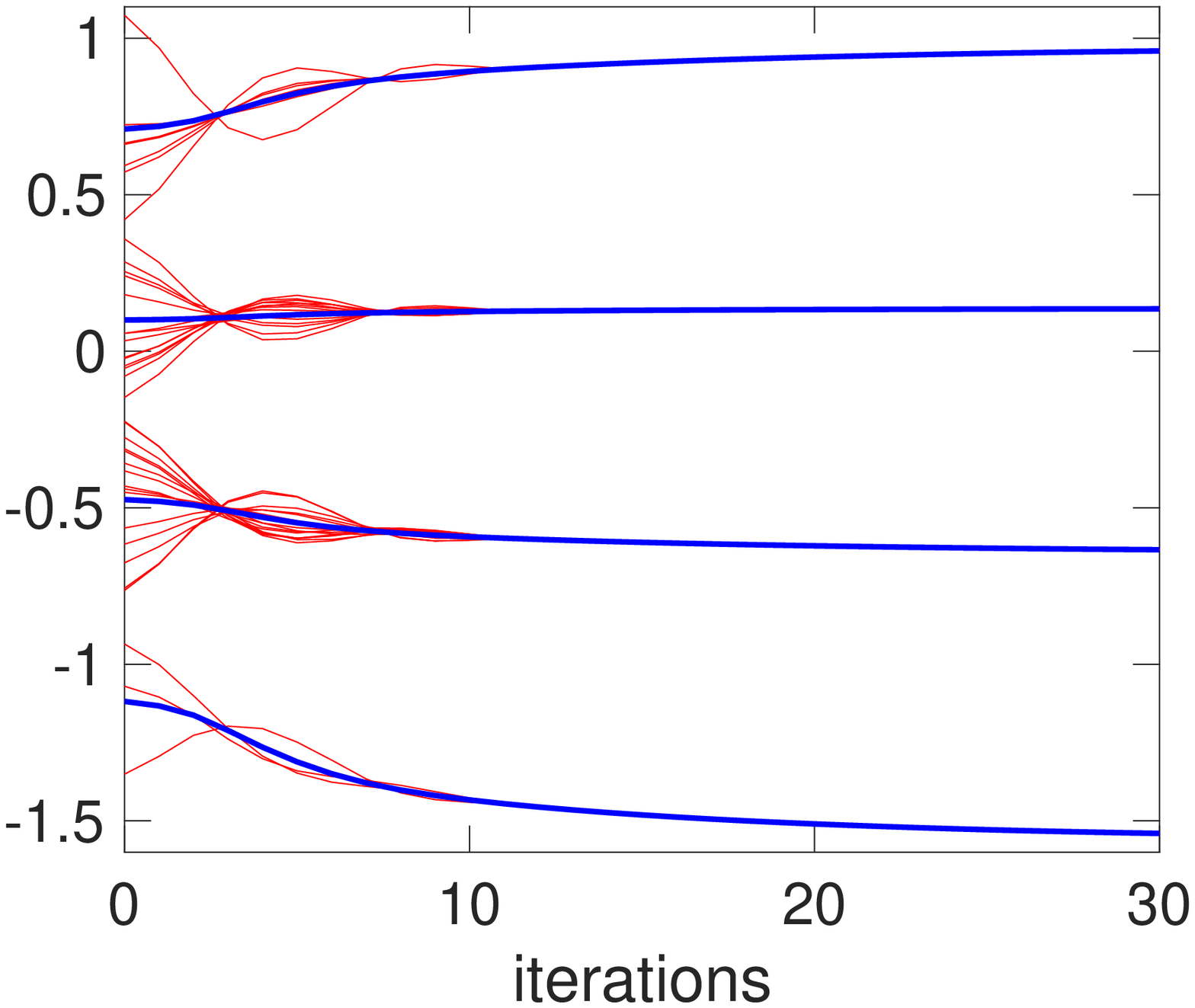}
  \end{tabular}
  \caption{Evolution of the centroid distribution over iterations for the LC algorithm for each layer of LeNet300 ($K=4$). \emph{Left}: weight distribution for LC iterations 0, 1 and 30, using a kernel density estimate with manually selected bandwidth. The locations of the codebook centroids are shown below the distributions as markers: $+$ are the centroids fitted to the reference net and $\times$ the centroids at the end of the LC algorithm. \emph{Right}: codebook centroids $c_k$ (blue) and 40 randomly chosen weights $w_i$ (red).}
  \label{f:weights-centroids-LC}
\end{figure}

\begin{figure}[p]
  \centering
  \psfrag{t0}[l][l]{0}
  \psfrag{t1}[l][l]{1}
  \psfrag{t3}[l][l]{30}
  \psfrag{weights}{}
  \psfrag{iterations}{}
  \begin{tabular}{@{}c@{}c@{}}
    \psfrag{counts}[][]{density, layer 1}
    \includegraphics*[width=0.49\linewidth]{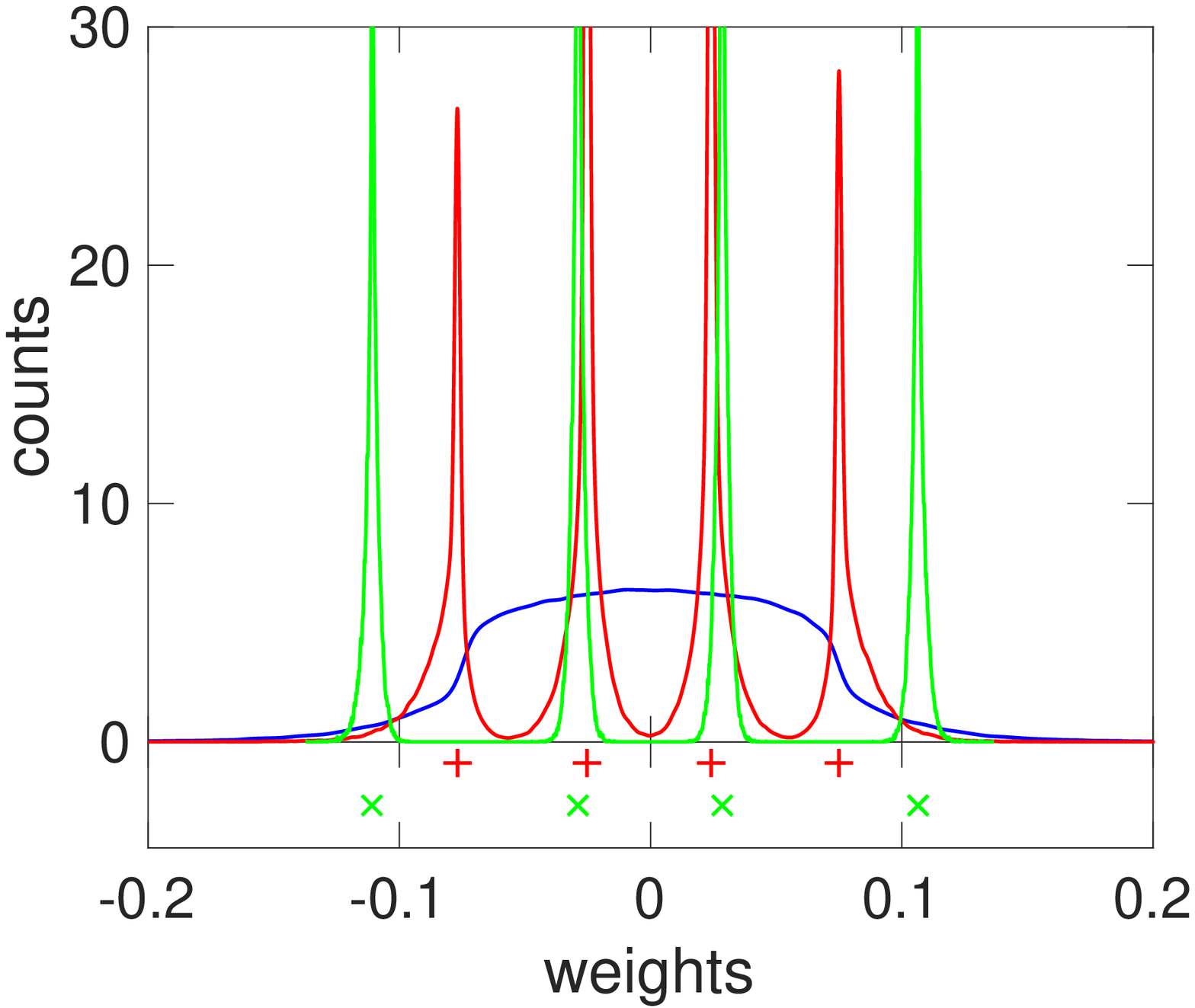} &
    \includegraphics*[width=0.49\linewidth]{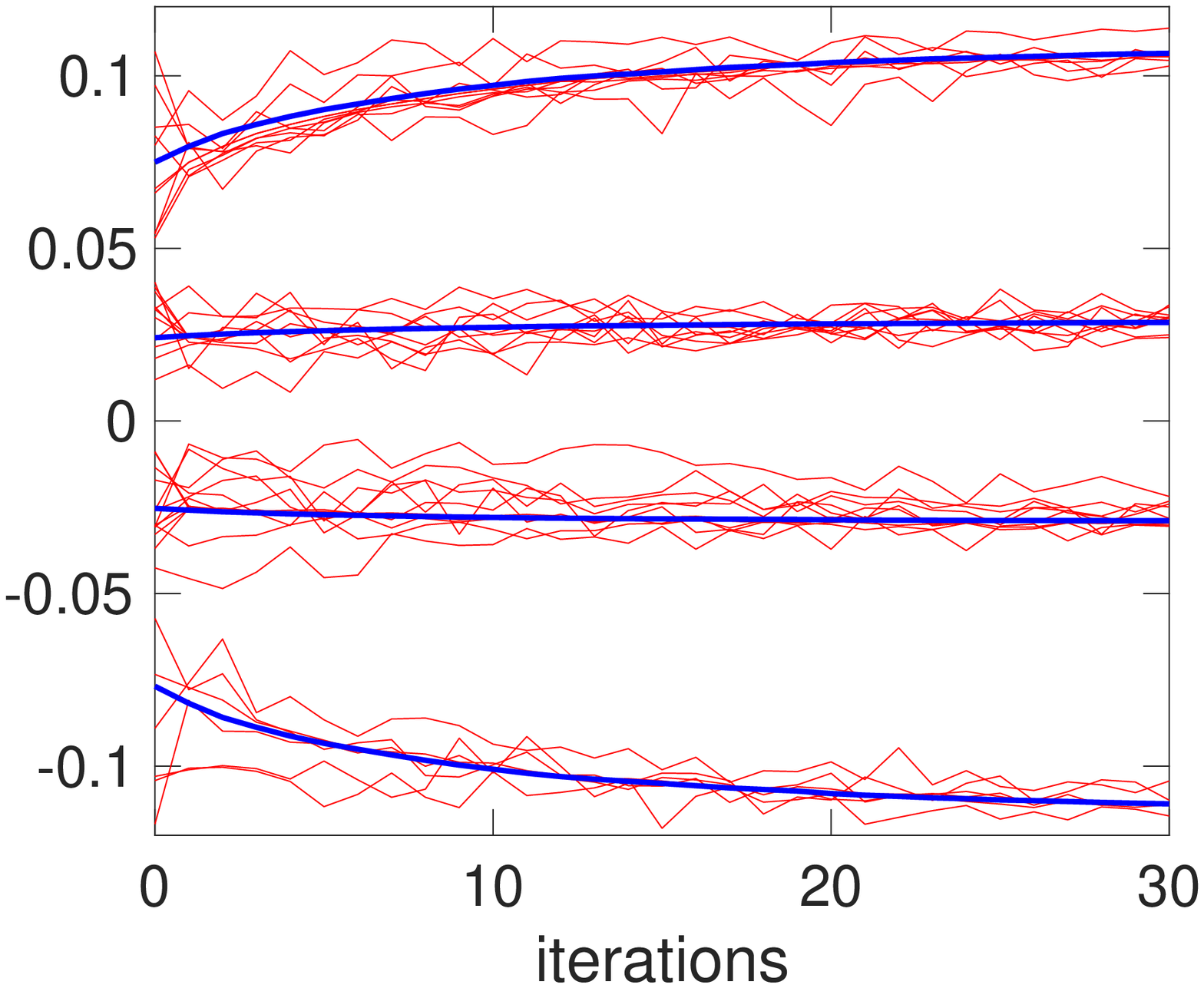} \\[-2ex]
    \psfrag{counts}[][]{density, layer 2}
    \includegraphics*[width=0.49\linewidth]{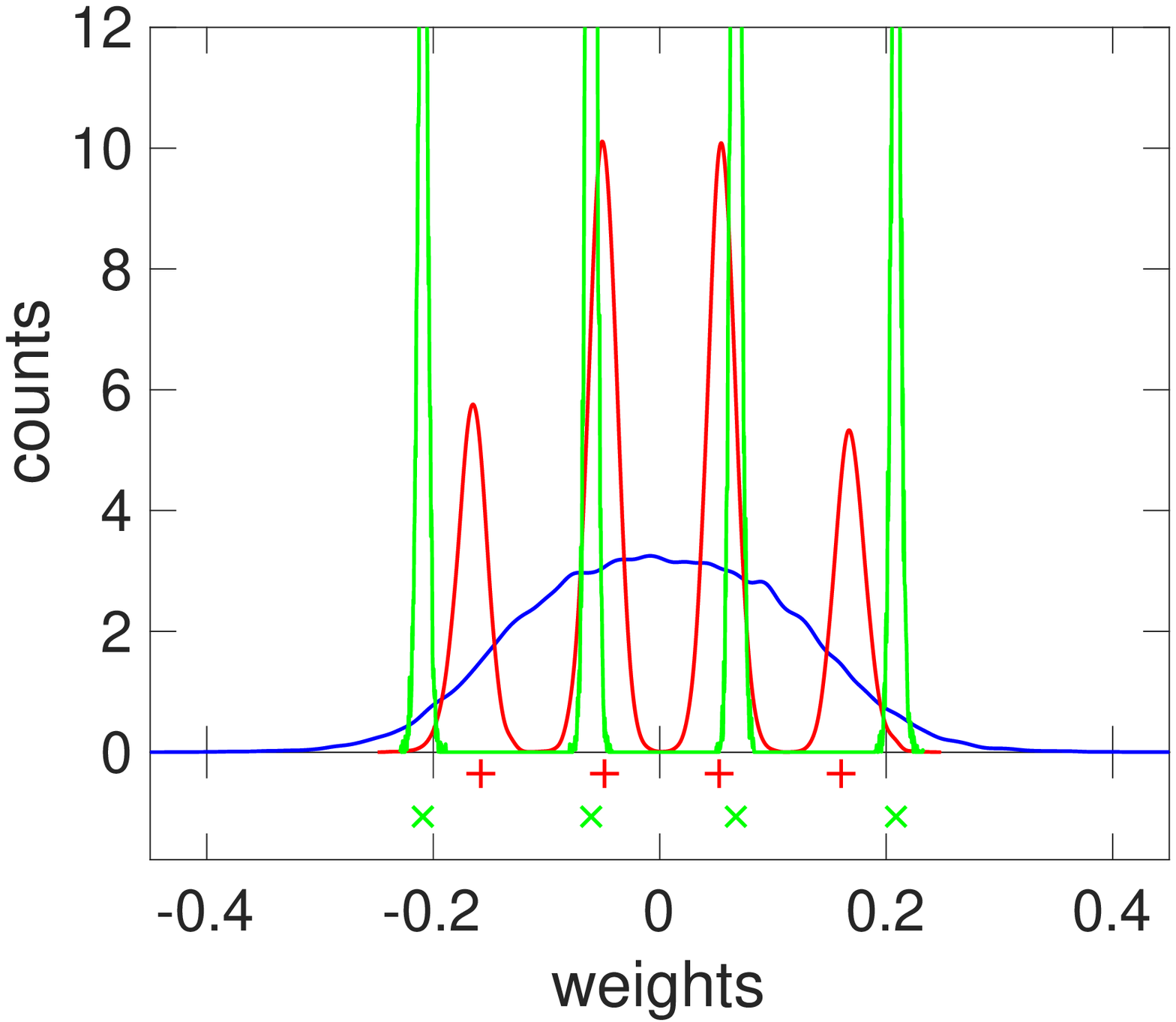} &
    \includegraphics*[width=0.49\linewidth]{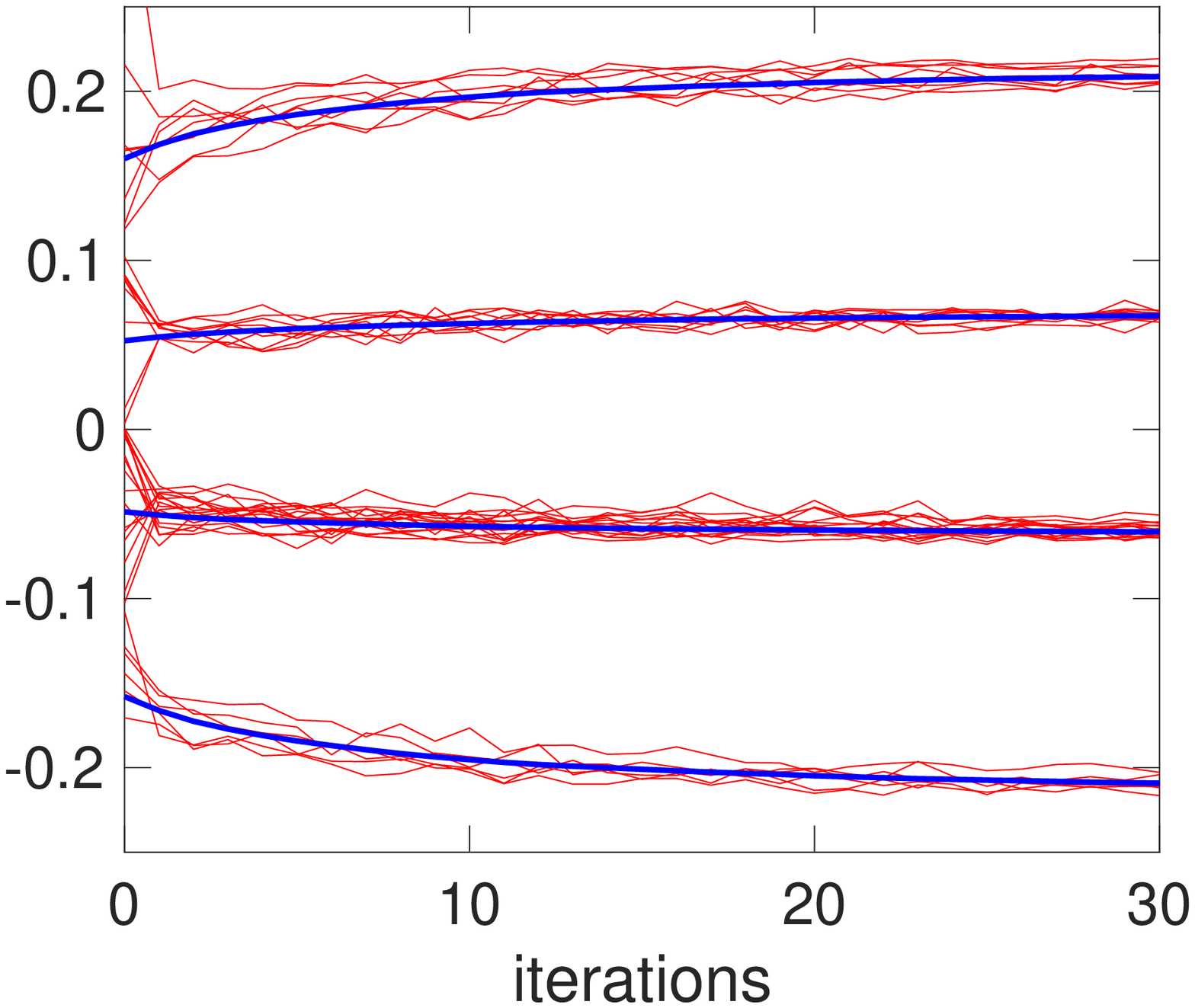} \\[-2ex]
    \psfrag{counts}[][]{density, layer 3}
    \psfrag{weights}[][B]{weights $w_i$}
    \includegraphics*[width=0.49\linewidth]{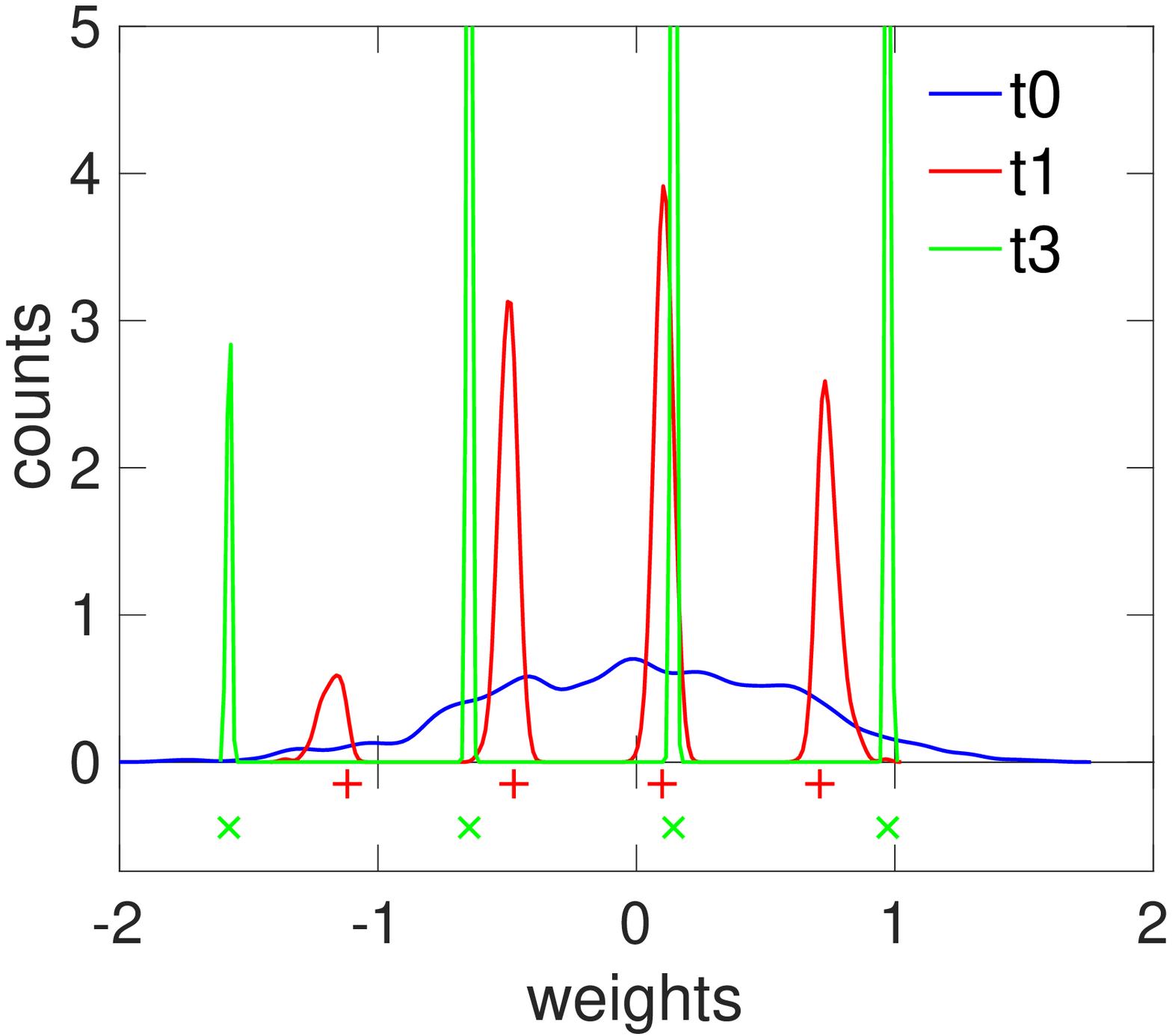} &
    \psfrag{iterations}[][B]{SGD iterations $\times$2k}
    \includegraphics*[width=0.49\linewidth]{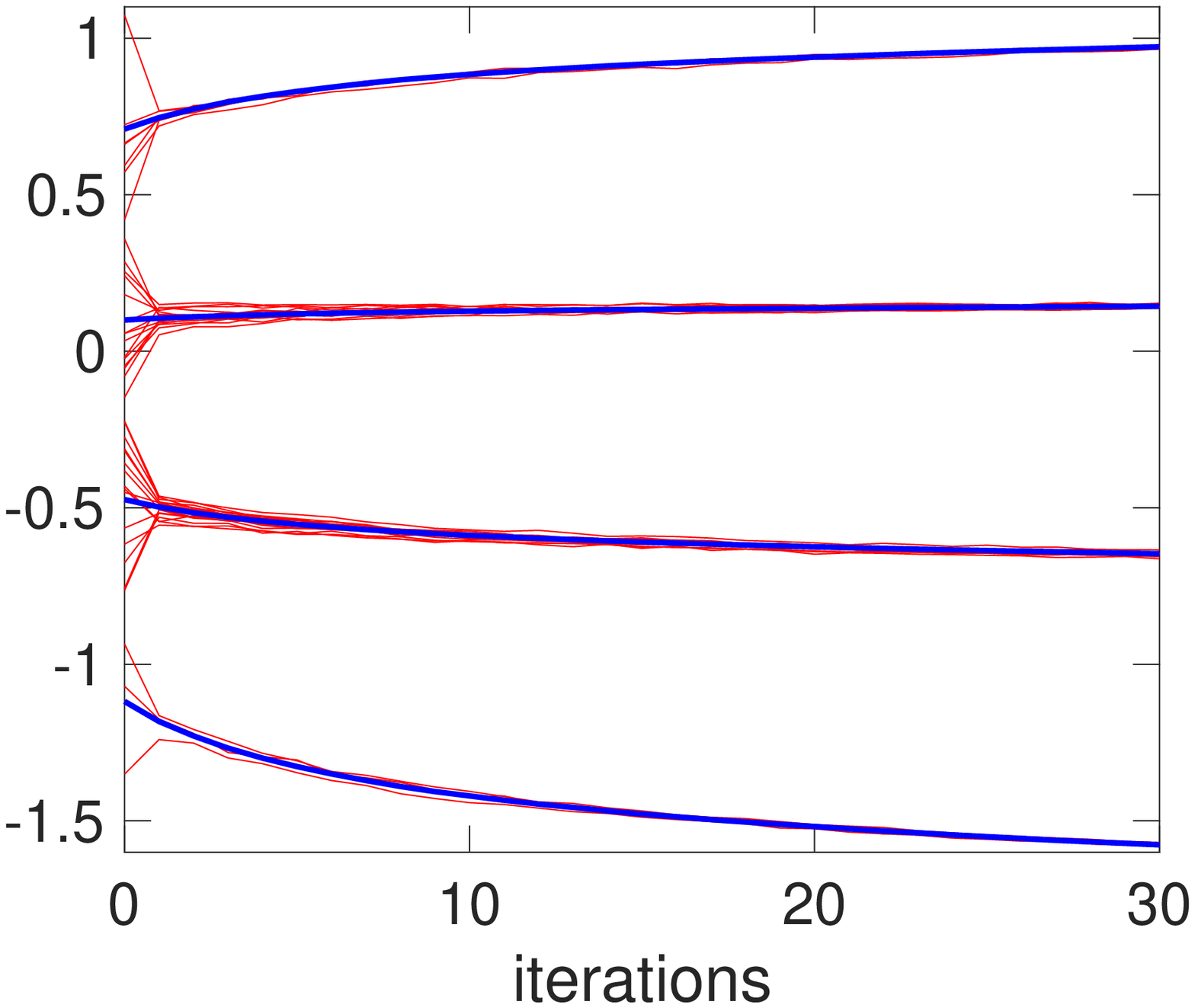}
  \end{tabular}
  \caption{Like fig.~\ref{f:weights-centroids-LC} but for iDC.}
  \vspace*{8ex} % to align vertically with the previous page figure
  \label{f:weights-centroids-iDC}
\end{figure}

\begin{figure}[t!]
  \centering
  \psfrag{codebook}[c][c]{codebook size $K$}
  \psfrag{idc}[l][l]{iDC}
  \psfrag{lc}[l][l]{LC}
  \psfrag{inf}[c][c]{$\infty$}
  \psfrag{lcmean}[l][l]{LC-mean}
  \psfrag{lcstd}[l][l]{LC-stdev}
  \psfrag{idcmean}[l][l]{iDC-mean}
  \psfrag{idcstd}[l][l]{iDC-stdev}
  \psfrag{ref}[l][l]{\small{reference}}
  \psfrag{centroids}[][]{centroids $c_k$}
  \psfrag{counts}[][]{{\tiny\caja[0.7]{c}{c}{reference \\ weight \\ density}}}
  \begin{tabular}{@{}r@{}r@{}r@{}}
    \multicolumn{1}{c@{}}{Layer 1} & \multicolumn{1}{c@{}}{Layer 2} & \multicolumn{1}{c@{}}{Layer 3} \\
    \includegraphics*[width=0.33\linewidth]{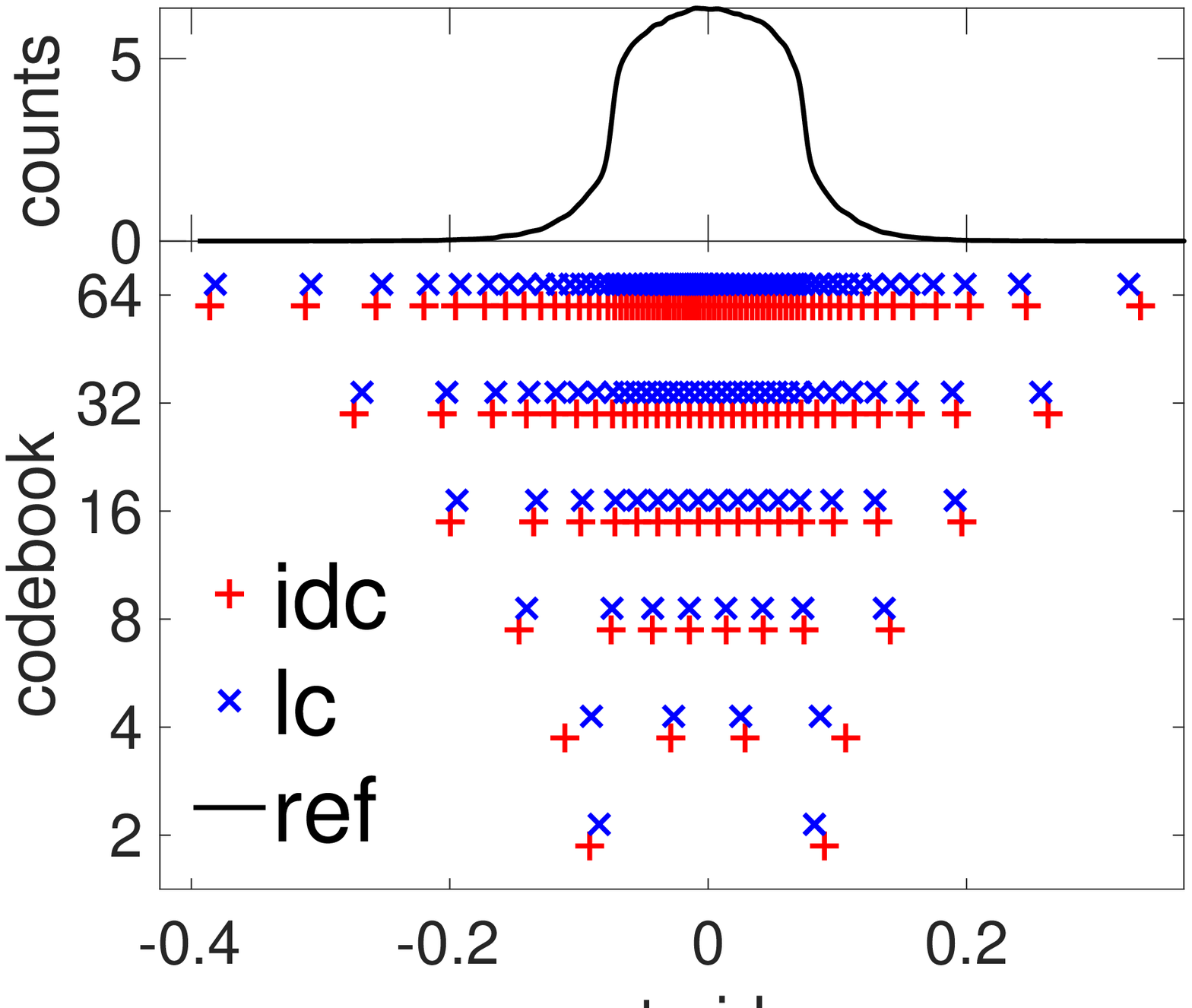} &
    \psfrag{codebook}{}
    \psfrag{counts}{}
    \includegraphics*[width=0.33\linewidth]{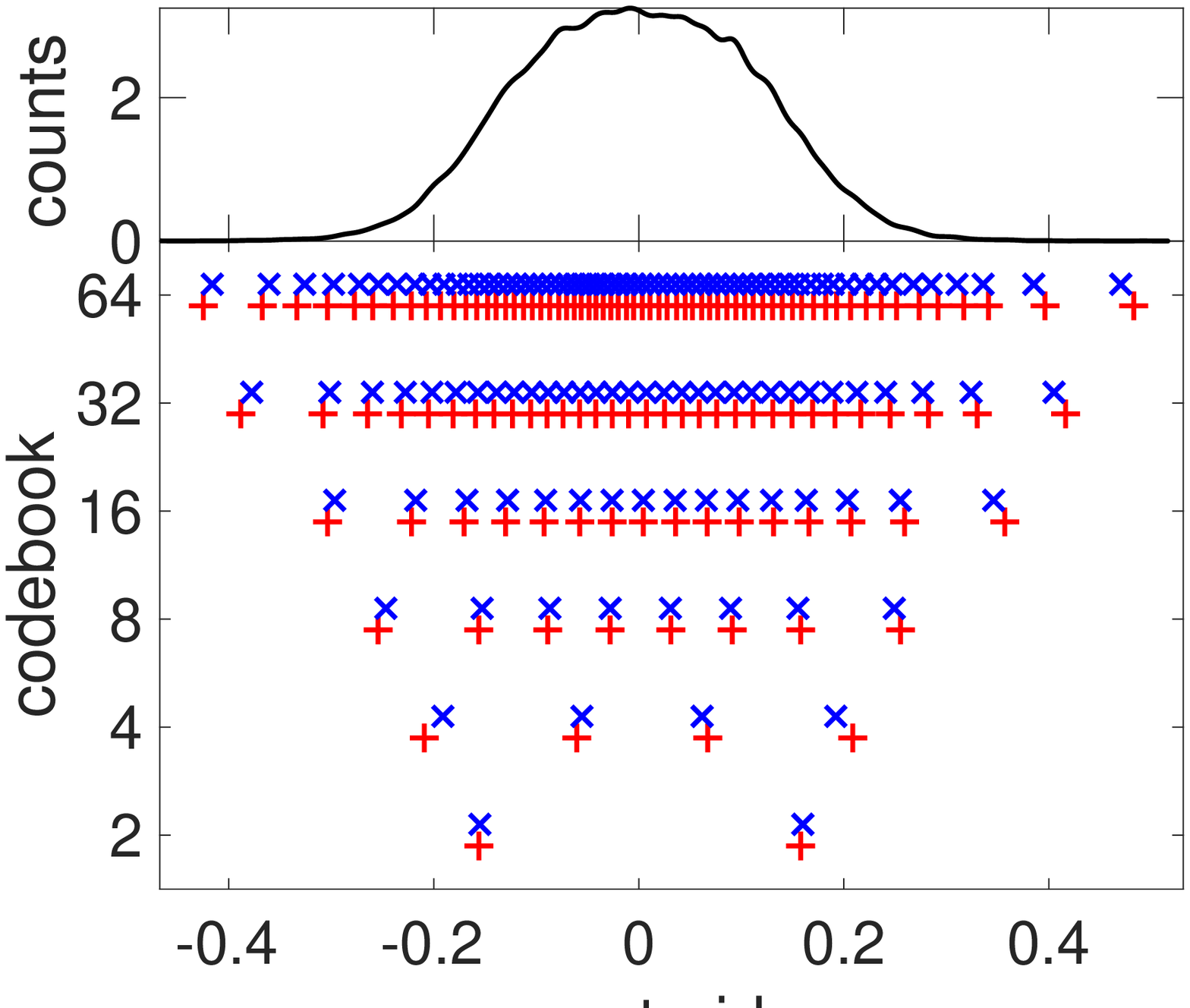} &
    \psfrag{codebook}{}
    \psfrag{counts}{}
    \includegraphics*[width=0.33\linewidth]{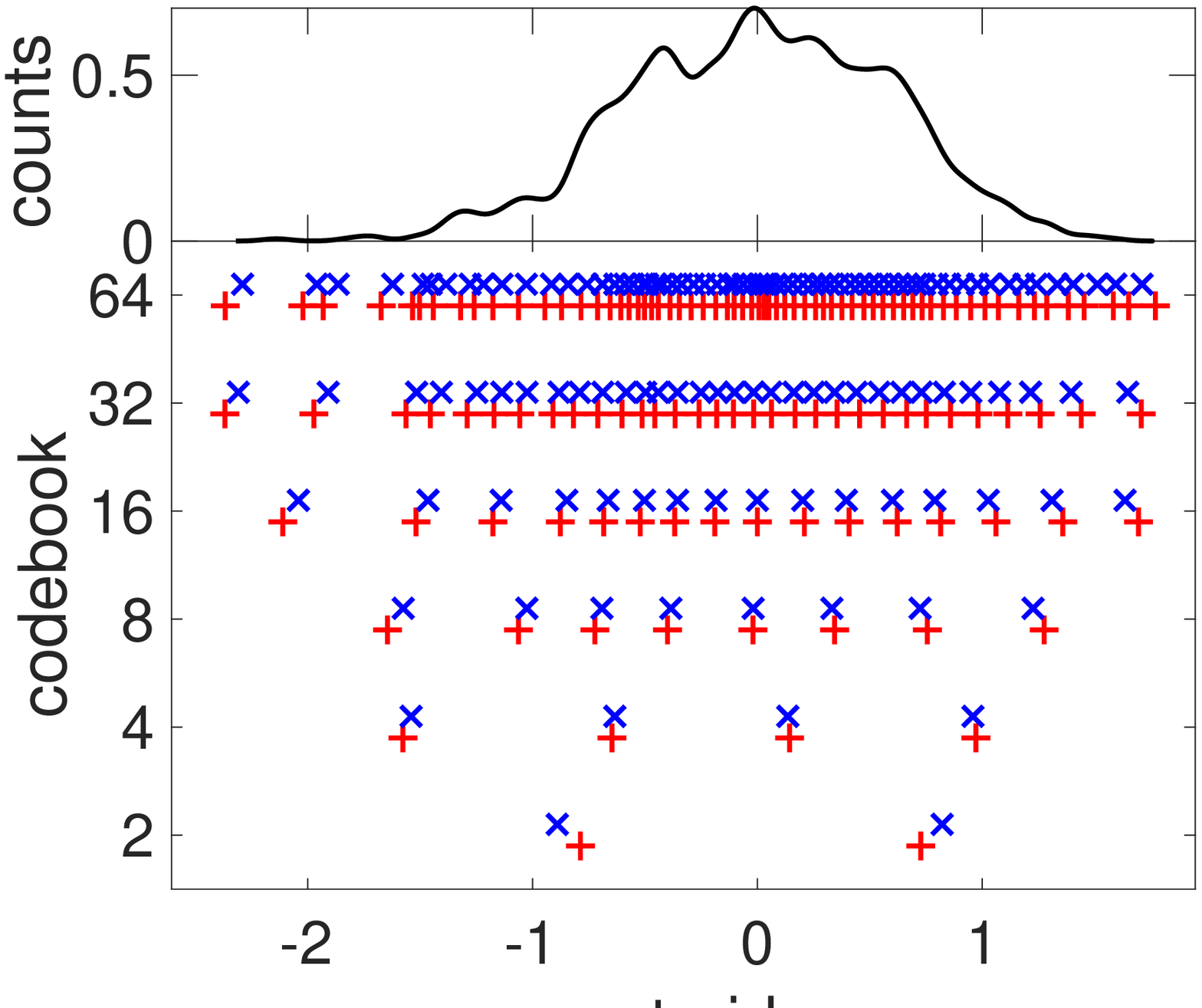} \\[2ex]
    \includegraphics*[width=0.33\linewidth]{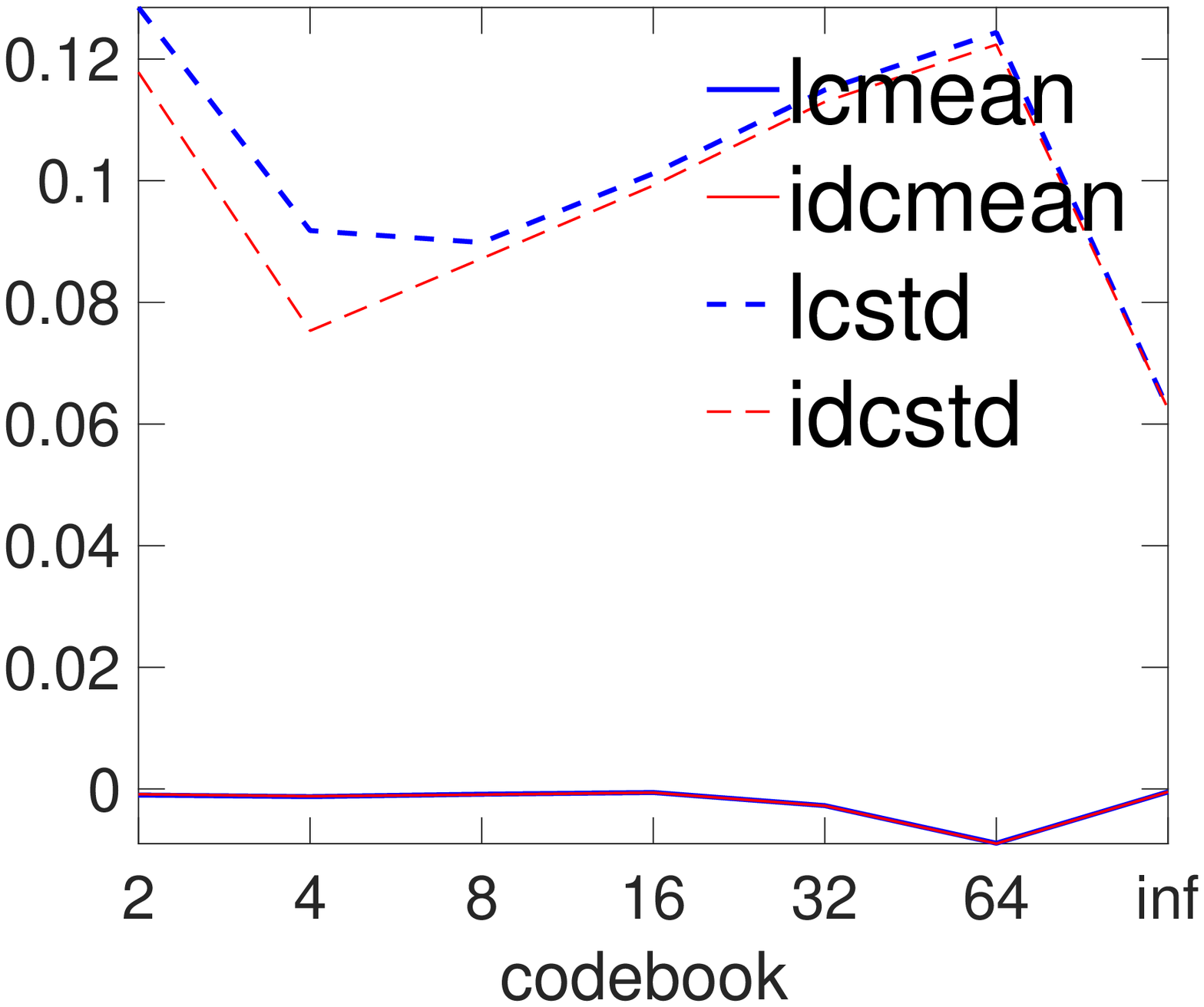} &
    \includegraphics*[width=0.33\linewidth]{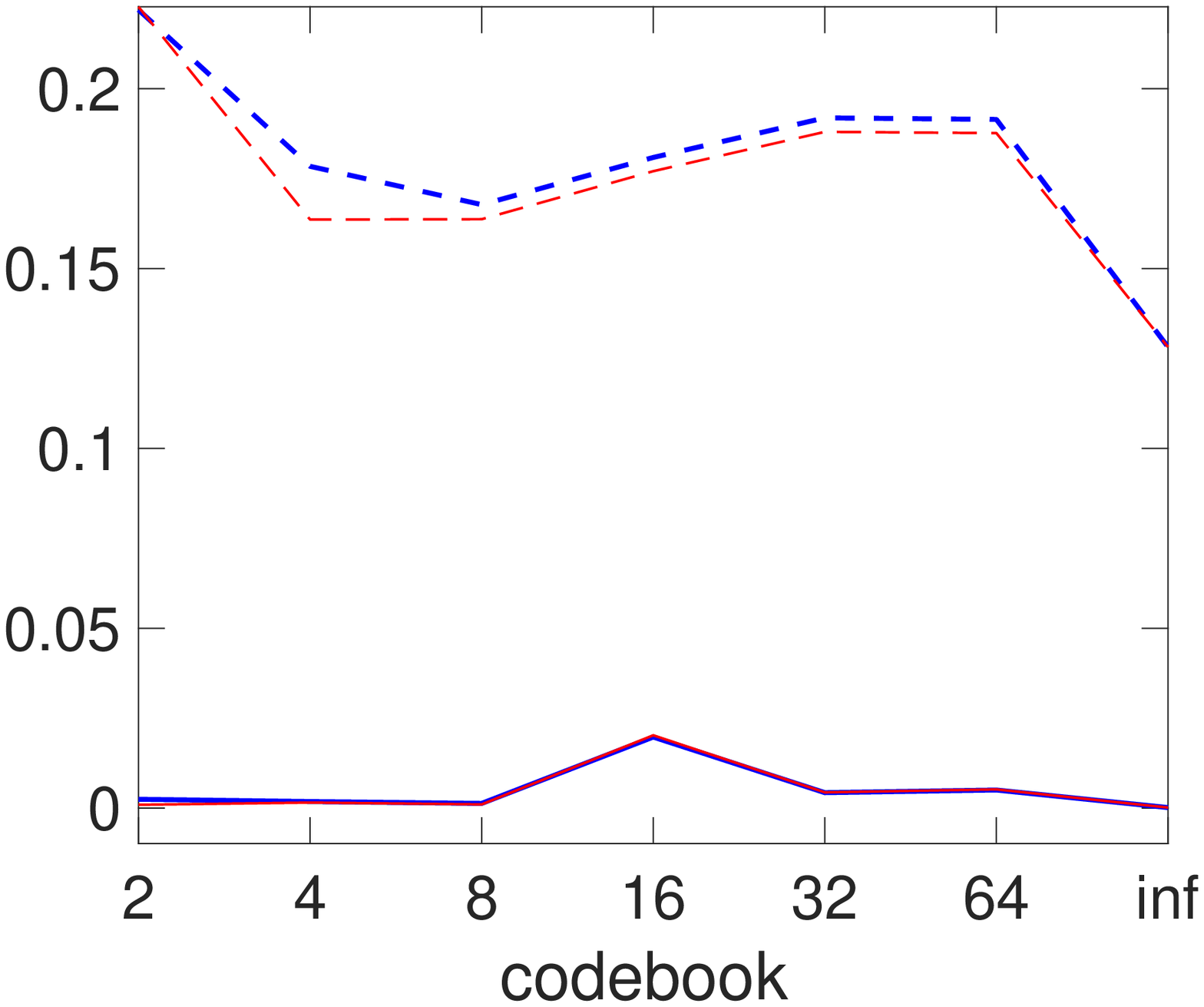} &
    \includegraphics*[width=0.32\linewidth]{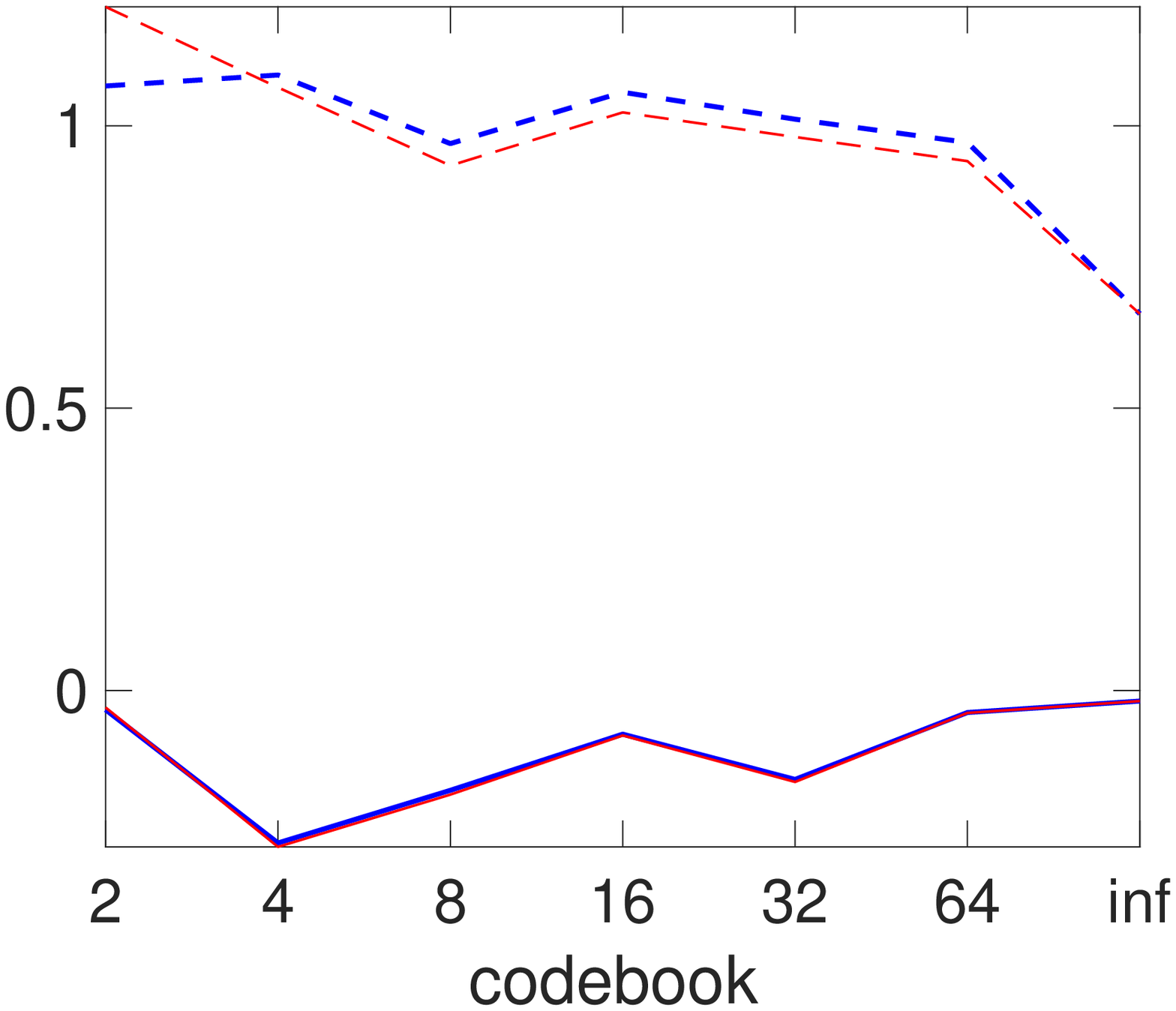}
  \end{tabular}
  \caption{Distribution of the centroids learnt by the LC and iDC algorithms, for layers 1--3 of LeNet300, for codebook sizes $K = 2$ to $64$. \emph{Top row}: actual centroid locations $c_k$, $k = 1,\dots,K$. The distribution of the weights of the reference net is shown at the top as a kernel density estimate. \emph{Bottom row}: mean and standard deviation of the centroid set $\calC = \{c_1,\dots,c_K\}$ (``$\infty$'' corresponds to no quantization, i.e., the mean and standard deviation of the reference net).}
  \label{f:centroids-K}
\end{figure}

Figs.~\ref{f:L1-neurons}--\ref{f:L2_L3-neurons} show the reference and final weights for LeNet300 compressed by the LC algorithm using a codebook of size $K = 2$ (a separate codebook for each of the 3 layers), which gives binary weights. For many of the weights, the sign of the quantized weight in LC equals the sign of the corresponding reference net weight. However, other weights change side during the optimization, specifically 5.04\%, 3.22\% and 1\% of the weights for layers 1, 2 and 3, respectively.

\begin{figure}[p]
  \centering
  \includegraphics*[width=0.83\linewidth]{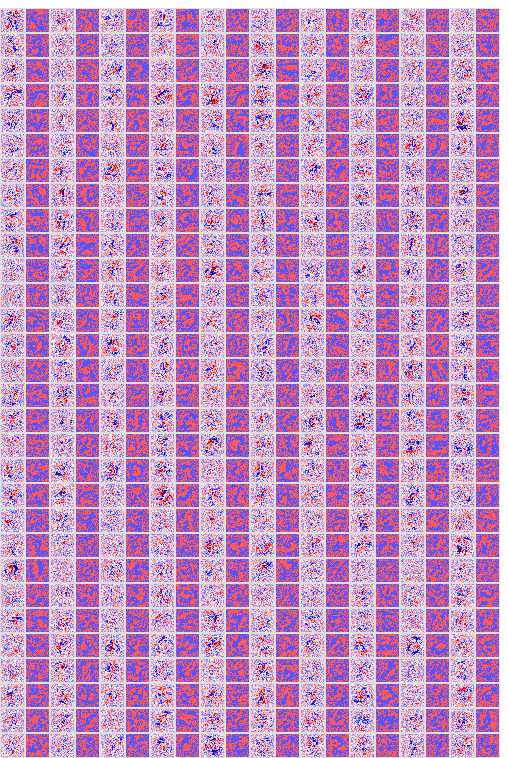} 
  \caption{Weight vector of the 300 neurons of layer 1 for LeNet300 for the reference net (left image of each horizontal pair, or odd-numbered columns) and for the net compressed with the LC algorithm using $K=2$ (right image of each pair, or even-numbered columns). We show each weight vector as a $28 \times 28$ image. All images have been globally normalized to the interval $[-3.5\sigma_1,3.5\sigma_1]$, where $\sigma_1$ is the standard deviation of the layer--1 reference weights (weights outside this interval are mapped to the respective end of the interval).}
  \label{f:L1-neurons}
\end{figure}

\begin{figure}[p]
  \centering
  \psfrag{codebook}[c][c]{$K$}
  \psfrag{layer2}[l][l]{Layer 2}
  \psfrag{layer3}[l][l]{Layer 3}
  \psfrag{idc}[l][l]{iDC}
  \psfrag{lc}[l][l]{LC}
  \psfrag{iterations}[][]{SGD iterations $\times$2k}
  \begin{tabular}{@{}c@{\hspace{1ex}}c@{\hspace{0.03\linewidth}}c@{\hspace{1ex}}c@{}}
    \multicolumn{2}{@{}c@{}}{\makebox[20ex]{\dotfill}Layer 2\makebox[20ex]{\dotfill}} & \multicolumn{2}{@{}c@{}}{\makebox[8ex]{\dotfill}Layer 3\makebox[8ex]{\dotfill}} \\
    reference & LC algorithm & reference & \makebox[0pt][c]{LC algorithm} \\
    \includegraphics*[width=1.10\linewidth,angle=90]{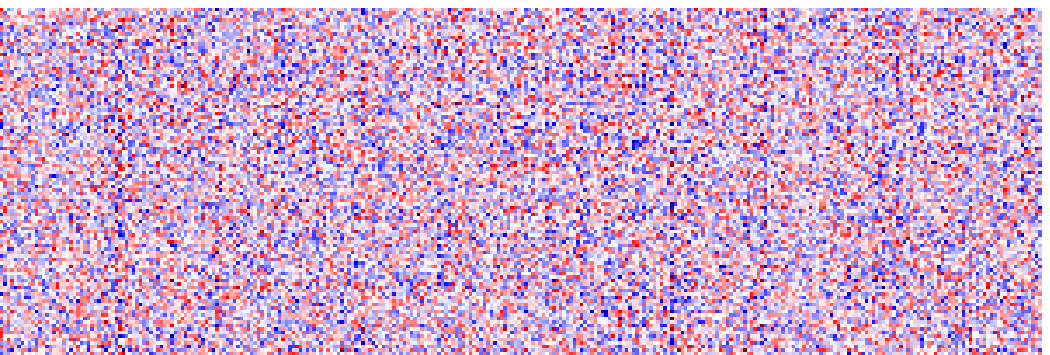} &
    \includegraphics*[width=1.10\linewidth,angle=90]{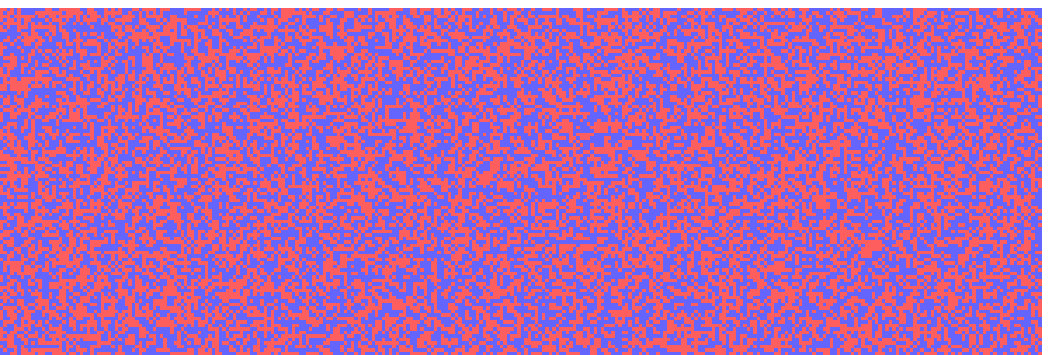} &
    \includegraphics*[width=1.10\linewidth,angle=90]{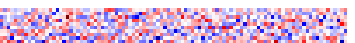} &
    \includegraphics*[width=1.10\linewidth,angle=90]{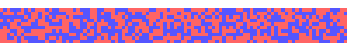} 
  \end{tabular}
  \caption{Weights of layers 2 and 3 of LeNet300 for the reference net and for the net compressed with the LC algorithm using $K=2$. We show layer 2 as a matrix of $300 \times 100$ and layer 3 as a matrix of $100 \times 10$. The normalization is as in fig.~\ref{f:L1-neurons}, i.e., to the interval $[-3.5\sigma_i,3.5\sigma_i]$, where $\sigma_i$ is the standard deviation of the $i$th layer reference weights, for $i = 2$ or $3$.}
  \label{f:L2_L3-neurons}
\end{figure}

\subsection{Quantizing a large deep net for classification on CIFAR10}

We randomly split the CIFAR10 training dataset (60k color images of $32 \times 32 \times 3$, 10 object classes) into training (90\%) and test (10\%) sets. We normalize the pixel colors to [0,1] and then subtract the mean. We train a 12-layer convolutional neural network inspired by the VGG net \citep{SimonyZisser15a} and described by \citet{Courbar_15a}, with the structure shown in table~\ref{t:CIFAR}. All convolutions are symmetrically padded with zero (padding size 1). The network has 14 million parameters ($P_1 =$ 14\,022\,016 weights and $P_0 =$ 3\,850 biases). Because of time considerations (each experiment takes 18 hours) we only report performance of LC with respect to the reference net. The reference net achieves 13.15\% error on the test set and a training loss of $1.1359 \cdot 10^{-7}$. Compressing with a $K=2$ codebook (compression ratio $\rho \approx \times$31.73), LC achieves a \emph{lower} test error of 13.03\% and a training loss of $1.6242 \cdot 10^{-5}$. 

\begin{table}[b!]
  \centering
  \begin{tabular}{@{}ll@{}}
    \toprule
    Layer & Connectivity \\
    \midrule
    Input & $3\times 32 \times 32$ image \\
    1 & \caja{t}{l}{convolutional, 128 (3$\times$3) filters (stride=1), \\ zero padding with size 1, total 131\,072 neurons, followed by ReLU} \\
    2 & \caja{t}{l}{convolutional, 128 (3$\times$3) filters (stride=1), \\ zero padding with size 1, total 131\,072 neurons, followed by ReLU} \\
    3 & \caja{t}{l}{max pool, $2 \times 2$ window (stride=2), \\ total 32\,768 neurons} \\
    4 & \caja{t}{l}{convolutional, 256 (3$\times$3) filters (stride=1), \\ zero padding with size 1, total 65\,536 neurons, followed by ReLU} \\
    5 & \caja{t}{l}{convolutional, 256 (3$\times$3) filters (stride=1), \\ zero padding with size 1, total 65\,536 neurons, followed by ReLU} \\
    6 & \caja{t}{l}{max pool, $2 \times 2$ window (stride=2), \\ total 16\,384 neurons} \\
    7 & \caja{t}{l}{convolutional, 512 (3$\times$3) filters (stride=1), \\ zero padding with size 1, total 32\,768 neurons, followed by ReLU} \\
    8 & \caja{t}{l}{convolutional, 512 (3$\times$3) filters (stride=1), \\ zero padding with size 1, total 32\,768 neurons, followed by ReLU} \\
    9 & \caja{t}{l}{max pool, $2 \times 2$ window (stride=2), \\ total 8\,192 neurons} \\
    10 & \caja{t}{l}{fully connected, 1024 neurons and dropout \\ with $p=0.5$, followed by ReLU} \\
    11 & \caja{t}{l}{fully connected, 1024 neurons and dropout \\ with $p=0.5$, followed by ReLU} \\
    \caja{t}{l}{12 \\ (output)} & \caja{t}{l}{fully connected, 10 neurons and dropout \\ with $p=0.5$, followed by softmax} \\
    \midrule
    \multicolumn{2}{c}{$P_1 =$ 14\,022\,016 weights, $P_0 =$ 3\,850 biases} \\
    \bottomrule
  \end{tabular}
  \caption{Structure of the 12-layer deep net trained on the CIFAR dataset.}
  \label{t:CIFAR}
\end{table}

\clearpage

\section{Conclusion}
\label{s:concl}

Neural net quantization involves minimizing the loss over weights taking discrete values, which makes the objective function nondifferentiable. We have reformulated this as optimizing the loss subject to quantization constraints, which is a mixed discrete-continuous problem, and given an iterative ``learning-compression'' (LC) algorithm to solve it. This alternates two steps: a learning step that optimizes the usual loss with a quadratic regularization term, which can be solved by SGD; and a compression step, independent of the loss and training set, which quantizes the current real-valued weights. The compression step takes the form of $k$-means if the codebook is adaptive, or of an optimal assignment and rescaling if the codebook is (partially) fixed, as for binarization. The algorithm is guaranteed to converge to a local optimum of the quantization problem, which is NP-complete. Experimentally, this LC algorithm beats previous approaches based on quantizing the reference net or on incorporating rounding into backpropagation. It often reaches the maximum possible compression (1 bit/weight) without significant loss degradation.

\subsubsection*{Acknowledgements}

Work supported by NSF award IIS--1423515, by a UC Merced Faculty Research Grant and by a Titan X Pascal GPU donated by the NVIDIA Corporation.

\appendix

\section{Theorems and proofs}

We prove several results concerning quantization with a fixed codebook (section~\ref{s:quant-fixed}) and give an alternative formulation for the case of binarization.

\subsection{Optimal quantization using a fixed codebook with or without scale}

We prove the optimal quantization results for binarization, ternarization and powers-of-two of section~\ref{s:quant-fixed}. The formulas for binarization and ternarization without scale follow from eq.~\eqref{e:quant-fixed-mapping-scalar}. The formulas for the powers-of-two and binarization and ternarization with scale are given in the theorems below. Define the sign function $\sgn{}$ as in eq.~\eqref{e:sgn}, and the floor function for $t \in \bbR$ as $\floor{t} = i$ if $i \le t < i+1$ and $i$ is integer.

\begin{thm}[powers of two]
  \label{th:pow2}
  Let $w \in \bbR$ and $C \ge 0$ integer. The solution $\theta^*$ of the problem
  \begin{equation}
    \label{e:pow2}
    \min_{\theta}{ E(\theta) = (w - \theta)^2 } \qquad \text{s.t.} \qquad \theta \in \{0,\pm 1,\pm 2^{-1},\dots,\pm 2^{-C}\}
  \end{equation}
  is $\theta^* = \alpha \, \sgn{w}$ where
  \begin{equation}
    \label{e:pow2-sol}
    \alpha =
    \begin{cases}
      0, & f > C+1 \\
      1, & f \le 0 \\
      2^{-C}, & f \in (C,C+1] \\
      2^{-\floor{f + \log_2{\frac{3}{2}}}}, & \text{otherwise}
    \end{cases}
  \end{equation}
  and $f = -\log_2{\abs{w}}$.
\end{thm}
\begin{proof}
  The sign of $\theta^*$ is obviously equal to the sign of $w$, so consider $w > 0$ and call $f = -\log_2{w}$. The solution can be written as a partition of $\bbR^+$ in four intervals $[0,2^{-C-1})$, $[2^{-C-1},2^{-C})$, $[2^{-C},1)$ and $[1,\infty)$ for $w$, or equivalently $(C+1,\infty)$, $(C,C+1]$, $(0,C]$ and $(-\infty,0]$ for $f$. These intervals are optimally assigned to centroids $0$, $2^{-C}$, $2^{-\floor{f + \log_2{\frac{3}{2}}}}$ and $1$, respectively. The solutions for the first, second and fourth intervals are obvious. The solution for the third interval $w \in [2^{-C},1) \Leftrightarrow f \in (0,C]$ is as follows. The interval of $\bbR^+$ that is assigned to centroid $2^{-i}$ for $i \in \{0,1,\dots,C\}$ is $w \in (3 \cdot 2^{-i-2},3 \cdot 2^{-i-1}]$, given by the midpoints between centroids (and breaking ties as shown), or equivalently $f + \log_2{\frac{3}{2}} \in [i,i+1)$, and so $i = \floor{f + \log_2{\frac{3}{2}}}$.
\end{proof}

\begin{thm}[binarization with scale]
  \label{th:bin-scale}
  Let $w_1,\dots,w_P \in \bbR$. The solution $(a^*,\btheta^*)$ of the problem
  \begin{equation}
    \label{e:bin-scale}
    \min_{a,\btheta}{ E(a,\btheta) = \sum^P_{i=1}{ (w_i - a \, \theta_i)^2 } } \qquad \text{s.t.} \qquad a \in \bbR,\quad \theta_1,\dots,\theta_P \in \{-1,+1\}
  \end{equation}
  is
  \begin{equation}
    \label{e:bin-scale-sol}
    a^* = \frac{1}{P} \sum^P_{i=1}{ \abs{w_i} } \qquad \theta^*_i = \sgn{w_i} =
    \begin{cases}
      -1, & \text{if } w_i < 0 \\
      +1, & \text{if } w_i \ge 0
    \end{cases}
    \quad \text{for } i = 1,\dots,P.
  \end{equation}
\end{thm}
\begin{proof}
  For any $a \in \bbR$, the solution $\btheta^*(a)$ for \btheta\ results from minimizing $E(a,\btheta)$ over \btheta\ (breaking the ties as shown in~\eqref{e:bin-scale-sol}). Substituting this into the objective function:
  \begin{multline*}
    E(a,\btheta^*(a)) = \sum^P_{i=1}{ (w_i - a \, \sgn{w_i})^2 } = \sum^P_{i=1}{ \left( w^2_i + a^2 (\sgn{w_i})^2 - 2 a \, w_i \sgn{w_i} \right) } \\
    = \bigg( \sum^P_{i=1}{ w^2_i } \bigg) + P a^2 - 2 a \, \sum^P_{i=1}{\abs{w_i}}.
  \end{multline*}
  The result for $a^*$ follows from differentiating wrt $a$ and equating to zero.
\end{proof}

\begin{thm}[ternarization with scale]
  \label{th:ter-scale}
  Let $w_1,\dots,w_P \in \bbR$ and assume w.l.o.g.\ that $\abs{w_1} \ge \abs{w_2} \ge \dots \ge \abs{w_P}$. The solution $(a^*,\btheta^*)$ of the problem
  \begin{equation}
    \label{e:ter-scale}
    \min_{a,\btheta}{ E(a,\btheta) = \sum^P_{i=1}{ (w_i - a \, \theta_i)^2 } } \qquad \text{s.t.} \qquad a \in \bbR,\quad \theta_1,\dots,\theta_P \in \{-1,0,+1\}
  \end{equation}
  is
  \begin{equation}
    \label{e:ter-scale-sol}
    j^* = \argmax_{1 \le j \le P}{ \frac{1}{\sqrt{j}} \sum^j_{i=1}{ \abs{w_i} } } \qquad a^* = \frac{1}{j^*} \sum^{j^*}_{i=1}{ \abs{w_i} } \qquad \theta^*_i =
    \begin{cases}
      0, & \text{if } \abs{w_i} < a^*/2 \\
      \sgn{w_i}, & \text{if } \abs{w_i} \ge a^*/2
    \end{cases}
    \quad \text{for } i = 1,\dots,P.
  \end{equation}
\end{thm}
\begin{proof}
  For any $a \in \bbR$, the solution $\btheta^*(a)$ for \btheta\ results from minimizing $E(a,\btheta)$ over \btheta\ (breaking the ties as shown in~\eqref{e:ter-scale-sol}). Substituting this into the objective function:
  \begin{equation*}
    E(a,\btheta^*(a)) = \sum_{i \in \overline{\calS}}{ w^2_i } + \sum_{i \in \calS}{ (w_i - a \, \sgn{w_i})^2 } = \bigg( \sum^P_{i=1}{ w^2_i } \bigg) + \abs{\calS} a^2 - 2 a \, \sum_{i \in \calS}{ \abs{w_i} }
  \end{equation*}
  where $\calS = \{i \in \{1,\dots,P\}\mathpunct{:}\ \abs{w_i} \ge a\}$, $\overline{\calS} = \{1,\dots,P\} \setminus \calS$ and $\abs{\calS}$ is the cardinality of \calS. Differentiating wrt $a$ keeping \calS\ fixed and equating to zero yields $a = \frac{1}{\abs{\calS}} \sum_{i \in \calS}{ \abs{w_i} }$. It only remains to find the set $\calS^*$ that is consistent with the previous two conditions on \btheta\ and $a$. Since the $w_i$ are sorted in decreasing magnitude, there are $P$ possible sets that \calS\ can be and they are of the form $\calS_i = \{w_1,\dots,w_i\} = \{w_j\mathpunct{:}\ \abs{w_j} \ge w_i\}$ for $i = 1,\dots,P$, and $\calS^* = \calS_{j^*}$ is such that the objective function is maximal. Hence, calling $a_j = \frac{1}{\abs{\calS_j}} \sum_{i \in \calS_j}{ \abs{w_i} }$ and noting that
  \begin{equation*}
    E(a_j,\btheta^*(a_j)) = \bigg( \sum^P_{i=1}{ w^2_i } \bigg) - \abs{\calS_j} a^2_j = \bigg( \sum^P_{i=1}{ w^2_i } \bigg) - j \, a^2_j
  \end{equation*}
  we have
  \begin{equation*}
    j^* = \argmin_{1 \le j \le P}{ E(a_j,\btheta^*(a_j)) } = \argmax_{1 \le j \le P}{ j \, a^2_j } = \argmax_{1 \le j \le P}{ \sqrt{j} \, a_j } = \argmax_{1 \le j \le P}{ \frac{1}{\sqrt{j}} \sum^j_{i=1}{ \abs{w_i} } }.
  \end{equation*}
  Finally, let us prove that the set $\calS_{j^*}$ is consistent with $a_{j^*}$ and $\btheta(a_{j^*})$, i.e., that $\calS_{j^*} = \{i \in \{1,\dots,P\}\mathpunct{:}\ \abs{w_i} \ge \frac{1}{2} a_{j^*}\}$. Since the $w_i$ are sorted in decreasing magnitude, it suffices to prove that $\abs{w_{j^*}} > \frac{1}{2} a_{j^*} > \abs{w_{j^*+1}}$. Since $j^* = \argmax_{1 \le j \le P}{ \sqrt{j} \, a_j }$, we have (in the rest of the proof we write $j$ instead of $j^*$ to avoid clutter):
  \begin{equation*}
    \sqrt{j} \, a_j \ge \sqrt{j+1} \, a_{j+1} = \sqrt{j+1} \, \bigg( \frac{1}{j+1} \sum^{j+1}_{i=1}{ \abs{w_i} } \bigg) = \frac{1}{\sqrt{j+1}} \, ( \abs{w_{j+1}} + j \, a_j ) \Leftrightarrow \abs{w_{j+1}} \le a_j \left( \sqrt{j(j+1)} - j \right).
  \end{equation*}
  Now $\sqrt{j(j+1)} - j < \frac{1}{2}$ (since $j(j+1) < (j + \frac{1}{2})^2 = j^2 + \frac{1}{4} + j$ $\forall j$), hence $\abs{w_{j+1}} < \frac{1}{2} a_j$. Likewise:
  \begin{equation*}
    \sqrt{j} \, a_j \ge \sqrt{j-1} \, a_{j-1} = \frac{1}{\sqrt{j+1}} \, ( j \, a_j - \abs{w_j} ) \Leftrightarrow \abs{w_j} \ge a_j \left( j - \sqrt{j(j-1)} \right).
  \end{equation*}
  Now $j - \sqrt{j(j-1)} > \frac{1}{2}$ (since $j(j-1) < (j - \frac{1}{2})^2 = j^2 + \frac{1}{4} - j$ $\forall j$), hence $\abs{w_j} > \frac{1}{2} a_j$.
\end{proof}

\subsection{An equivalent formulation for binarization}

We show that, in the binarization case (with or without scale), our constrained optimization formulation ``$\min_{\w,\bTheta}{ L(\w) } \text{ s.t.\ } \w = \bDelta(\bTheta)$'' of eq.~\eqref{e:compression-problem} can be written equivalently without using assignment variables \Z, as follows:
\begin{align*}
  \text{Binarization: } & \min_{\w,\b}{ L(\w) } \quad \text{s.t.} \quad \w = \b,\ \b \in \{-1,+1\}^P \\
  \text{Binarization with scale: } & \min_{\w,\b,a}{ L(\w) } \quad \text{s.t.} \quad \w = a \, \b,\ \b \in \{-1,+1\}^P,\ a > 0.
\end{align*}
We can write an augmented-Lagrangian function as (for the quadratic-penalty function, set $\blambda = \0$):
\begin{align*}
  \calL_A(\w,\b,\blambda;\mu) &= L(\w) - \blambda^T (\w - \b) + \frac{\mu}{2} \norm{\w - \b}^2 \quad \text{s.t.} \quad \b \in \{-1,+1\}^P \\
  \calL_A(\w,\b,a,\blambda;\mu) &= L(\w) - \blambda^T (\w - a \, \b) + \frac{\mu}{2} \norm{\w - a \, \b}^2 \quad \text{s.t.} \quad \b \in \{-1,+1\}^P,\ a > 0.
\end{align*}
And applying alternating optimization gives the steps:
\begin{itemize}
\item L step: $\min_{\w}{ L(\w) + \frac{\mu}{2} \norm{\w - \b - \smash{\frac{1}{\mu}} \blambda}^2 }$ or $\min_{\w}{ L(\w) + \frac{\mu}{2} \norm{\w - a \, \b - \smash{\frac{1}{\mu}} \blambda}^2 }$, respectively.
\item C step: $\b = \sgn{\w}$ or $\b = a \, \sgn{\w}$ and $a = \frac{1}{P} \sum^P_{i=1}{\abs{w_i}}$, respectively, an elementwise binarization and global rescaling.
\end{itemize}
This is identical to our LC algorithm.

% \bibliographystyle{abbrvnat}
% \bibliography{macp,macp-xref}

\end{document}